\documentclass{article}

\usepackage[preprint]{neurips_2025}

\usepackage[utf8]{inputenc}
\usepackage[T1]{fontenc}

\usepackage{microtype}
\usepackage{graphicx}
\usepackage{subcaption}
\usepackage{booktabs}
\usepackage{hyperref}
\usepackage{url}
\usepackage{algorithm}
\usepackage{algorithmic}

\usepackage{amsmath}
\usepackage{amssymb}
\usepackage{mathtools}
\usepackage{amsthm}

\usepackage{amsfonts}
\usepackage{bm}
\usepackage{xfrac}
\usepackage{dsfont}
\usepackage{bbm}
\usepackage{xspace}
\usepackage{thmtools}
\usepackage{thm-restate}

\usepackage{xcolor}
\usepackage{nicefrac}
\usepackage{multirow}
\usepackage{multicol}
\usepackage{colortbl}
\usepackage{tablefootnote}
\usepackage[framemethod=TikZ]{mdframed}
\usepackage[inline]{enumitem}
\usepackage{notation}
\usepackage{etoolbox}
\usepackage{wrapfig} 
\usepackage{booktabs} 
\usepackage{siunitx}  

\newif\ifshowtodos
\showtodostrue
\ifshowtodos
  \usepackage[textsize=tiny]{todonotes}
\else
  \usepackage[disable]{todonotes}
\fi

\mdfdefinestyle{propFrame}{
  linecolor=white,
  outerlinewidth=1pt,
  roundcorner=6pt,
  innertopmargin=5.5pt,
  innerbottommargin=2.5pt,
  innerrightmargin=3.5pt,
  innerleftmargin=3.5pt,
  backgroundcolor=black!3!white
}

\theoremstyle{plain}
\newtheorem{lemma}{Lemma}
\newtheorem{theorem}{Theorem}
\newtheorem*{theorem*}{Theorem}
\newtheorem{corollary}{Corollary}
\newtheorem*{corollary*}{Corollary}

\theoremstyle{definition}

\newtheorem{proposition}{Proposition}

\newtheorem{assumption}{Assumption}
\newtheorem{remark}{Remark}

\title{Proximal Point Nash Learning from Human Feedback}

\author{%
Daniil Tiapkin$^{1,2}$ \quad Daniele Calandriello$^3$ \quad Denis Belomestny$^{4,5}$ \quad \'Eric Moulines$^{6,7}$ \\
\textbf{Alexey Naumov}$^5$ \quad \textbf{Kashif Rasul}$^8$ \quad \textbf{Michal Valko}$^{9}$ \quad \textbf{Pierre M\'enard}$^{10}$ \\
$^1$CMAP, CNRS, \'Ecole Polytechnique, IPP \quad
$^2$LMO, Université Paris-Saclay \quad
$^3$Google DeepMind \\
$^4$Duisburg-Essen University \quad
$^5$HSE University \quad
$^6$Mohamed Bin Zayed University of AI \\
$^7$LRE EPITA \quad
$^8$Hugging Face \quad
$^9$Isara Labs \quad
$^{10}$ENS Lyon \\
\texttt{daniil.tiapkin@polytechnique.edu} \quad
\texttt{dcalandriello@google.com} \\
\texttt{denis.belomestny@uni-due.de} \quad \texttt{eric.moulines@mbzuai.ac.ae} \\
\texttt{anaumov@hse.ru} \quad
\texttt{kashif.rasul@gmail.com} \\
\texttt{michal@isara.io} \quad
\texttt{pierre.menard@ens-lyon.fr}
}

\begin{document}

\maketitle

\begin{abstract}
Traditional Reinforcement Learning from Human Feedback (RLHF) often relies on reward models, frequently assuming preference structures like the Bradley--Terry model, which may not accurately capture the complexities of real human preferences (e.g., intransitivity).
Nash Learning from Human Feedback (NLHF) offers a more direct alternative by framing the problem as finding a Nash equilibrium of a game defined by these preferences. While many works study the Nash learning problem directly in the policy space, we instead consider it under a more realistic policy parametrization setting. We first analyze a simple self-play policy gradient method, which is equivalent to Online IPO. We establish high-probability last-iterate convergence guarantees for this method, but our analysis also reveals a possible stability limitation of the underlying dynamics. Motivated by this, we embed the self-play updates into a proximal point framework, yielding a stabilized algorithm. For this combined method, we prove high-probability last-iterate convergence and discuss its more practical version, which we call Nash Prox. Finally, we apply this method to post-training of large language models and validate its empirical performance.

\end{abstract}

\section{Introduction}\label{sec:intro}

Aligning powerful pre-trained Large Language Models (LLMs) with complex and often subjective human preferences and values is a central challenge for safe and beneficial AI. Reinforcement Learning from Human Feedback (RLHF)~\citep{NIPS2017_d5e2c0ad} addresses this by learning from human preference signals rather than sparse or hand-engineered reward functions. RLHF has been successfully used to fine-tune LLMs for tasks such as summarization~\citep{stiennon2020learning}, dialogue, and question answering~\citep{ziegler2019finetuning,ouyang2022training,bai2022training}.

A common approach within RLHF, rooted in contextual dueling bandits~\citep{yue2012karmed,zoghi2014relative,bengs2021preference}, is to posit an underlying reward model. The most prevalent choice is the Bradley--Terry (BT) model~\citep{zermelo1929berechnung,bradley1952rank}, which assigns each action a scalar reward and models pairwise preferences as a function of reward differences. Under BT, the goal reduces to learning the reward and choosing an action that is preferred, on average, to alternatives. In Social Choice Theory, this corresponds to a Condorcet winner.

However, reward-model approaches can be misspecified~\citep{dudik2015contextual,munos2023nash}. In particular, BT imposes transitivity: if $a$ is preferred to $b$ and $b$ to $c$, then $a$ must be preferred to $c$. Real human judgments often violate this property, exhibiting intransitivity~\citep{gardner1970paradox,tversky1969intransitivity,klimenko2015intransitivity}. Moreover, even when individual preferences are transitive, group aggregation can create cycles~\citep{may1954intransitivity,kreweras1965aggregation}. Such non-transitive preferences can preclude a Condorcet winner and, more broadly, a consistent scalar reward function matching all comparisons.

\textbf{Nash Learning from Human Feedback (NLHF).}
To avoid assuming a consistent reward or a Condorcet winner, \citet{dudik2015contextual} proposed a preference-based approach for dueling bandits, later termed Nash Learning from Human Feedback (NLHF) by \citet{munos2023nash}. NLHF models pairwise preferences via a symmetric two-player game in which each player proposes an action. The objective is to find a symmetric Nash equilibrium (NE,~\citealt{neumann1928zur,nash1950equilibrium}), called a von Neumann winner (VNW) in the dueling bandits literature~\citep{dudik2015contextual}. Unlike a Condorcet winner, a VNW is a mixed policy and remains well-defined under intransitive preferences.

\textbf{Regularization and the parameterized setting.}
In LLM post-training, we want to align to preferences while staying close to a fixed reference policy (for example, an instruction-following model). We therefore consider the NE of a \emph{regularized} preference game, adding a penalty proportional to the Kullback--Leibler (KL) divergence from the current policy to the reference policy. While many works analyze algorithms for regularized NLHF by exploiting convexity in the policy space~\citep{sokota2022unified,munos2023nash,cen2024fast}, less is known when the policy is \emph{parameterized} (as in neural networks) and trained via stochastic gradient methods.

\textbf{Our approach.}
We close this gap by analyzing the simplest self-play policy gradient (SPG) under parameterization. Under a leave-one-out advantage estimator, SPG iterations coincide with Online IPO~\citep{calandriello2024human}. We prove high-probability last-iterate convergence and identify a stability condition on regularization; beyond it, the dynamics may be unstable.

Motivated by this observation, we wrap self-play updates in an (inexact) proximal-point outer loop~\citep{martinet1970breve}. The resulting method uses two anchors: the reference policy (to limit drift from the base model) and the previous outer iterate (to stabilize dynamics). This yields high-probability last-iterate convergence without a stability restriction on the regularization parameter. We also propose a practical variant, \texttt{Nash Prox}, which approximates the proximal anchor via an exponential moving average (EMA), and show competitive performance on synthetic preference games and LLM post-training.

\textbf{Contributions.}
We summarize our contributions as follows.
\begin{itemize}[leftmargin=*, itemsep=0.5pt,topsep=0pt]
    \item We analyze SPG (equivalently, Online IPO under leave-one-out advantage estimator) for regularized NLHF under general policy parameterization, which includes softmax parametrization, proving high-probability last-iterate convergence under a stability condition on the regularized game.
    \item We embed SPG into an inexact proximal-point framework (PP-SPG), achieving high-probability last-iterate convergence in the parameterized setting without extra stability conditions.
    \item We develop a practical deep-learning implementation of PP-SPG called \texttt{Nash Prox}, based on an EMA target anchor, and show competitive results on synthetic games and LLM post-training.
\end{itemize}

\section{Related work}

\textbf{Nash Learning from Human Feedback.} The NLHF framework was introduced by \citet{munos2023nash}, building on the formulation of contextual dueling bandits as a symmetric two-player game by \citet{dudik2015contextual}. They proposed the $\NashMD$ algorithm and showed that it enjoys last-iterate convergence to the VNW of a regularized preference game at a polynomial rate.

Subsequent work studied how to approximate a VNW in the \emph{unregularized} preference game, e.g., via regret-minimization tools~\citep{swamy2024minimaximalist,wu2024self}, optimistic mirror descent variants~\citep{zhang2025improving,wu2025multi}, and meta-algorithms that ensure asymptotic convergence~\citep{liu2024comal}. A separate line of work improved guarantees in the regularized setting~\citep{zhang2024iterative,wang2024magnetic,tang2025rspo}, but these methods typically operate in the policy space via direct distributional updates. A notable exception is the work by \citet{zhou2025extragradient}, which provides a policy-gradient interpretation but requires samples from a uniform distribution over the action space, which is impractical in the LLM setting where actions are sentences.

Separately, \citet{calandriello2024human} showed that the online version of IPO~\citep{azar2024general} converges to the VNW, but without explicit rates. In this work, we give high-probability last-iterate convergence guarantees for Online IPO in the parameterized setting. The result follows from its equivalence to SPG under a leave-one-out advantage estimator and holds under a condition on the regularization coefficient. We also introduce a stabilized proximal-point variant that uses Online IPO as a building block.

Most of the above work focuses on single-step preference feedback. Addressing multi-turn decision making, \citet{shani2025multi} studied preference feedback in settings requiring planning, proposes a self-play mirror descent based algorithm, and proves convergence to a Nash equilibrium. Another line of work extends the NLHF to a Stackelberg principal-agent formulation, breaking the symmetry between max and min-players \citep{choi2025selfimproving,pasztor2025stackelberg}.

\textbf{Policy gradient methods.}
A large body of work studies policy gradient (PG) methods for unregularized reinforcement learning with general parameterized policy classes.
A widely used set of sufficient conditions combines regularity of the score function with Fisher non-degeneracy and (approximate) compatibility~\citep{sutton1999policy}, enabling global convergence and finite-sample guarantees beyond the tabular setting
\citep{papini2018stochastic,huang2020momentum,liu2020improved,agarwal2021theory,yuan2022general,ding2022global,fatkhullin2023stochastic,lu2024towards}.
These analyses typically consider a \emph{fixed} objective and quantify progress using smoothness and gradient-dominance conditions on the induced objective in parameter space.

\textbf{Regularized policy gradients.}
Less is known about \emph{Kullback-Leibler (KL)- or entropy-regularized} objectives under parameterization, despite their prominence in RLHF and preference optimization.
For entropy-regularized MDPs, \citet{mei2020global} proved global convergence of softmax PG and showed gradient dominance is non-uniform without control of the minimum action probability.
Subsequent work analyzes regularized gradient flows and stability~\citep{leahy2022convergence}, and uses projection/truncation to maintain a minimum action probability~\citep{zhang2021sample,labbi2025global}; see also \citet{liu2024elementary} for convergence over a wider range of step sizes.
For two-player zero-sum Markov games, \citet{zeng2022regularized} studied regularized gradient descent--ascent with softmax policies under partially decoupled updates (alternating steps and unequal step sizes).
In contrast, our self-play policy gradient uses a single timescale with simultaneous updates and a shared learning rate, so the opponent evolves at the same rate each step, making the analysis more challenging.

\section{Setting}\label{sec:setting}

 We consider a contextual dueling bandit setting $(\cX, \cY, \cP)$, where $\cX$ is a context space, $\cY$ is a finite action space, and $\cP(y \succ y' \mid x)\in[0,1]$ is the probability that action $y \in \cY$ is preferred to $y' \in \cY$ given context $x \in \cX$, which satisfies the symmetry condition: $\cP(y \succ y' \mid x) = 1 -\cP(y' \succ y \mid x)$ for all $x, y, y'$. A policy $\policy \colon \cX \to \simplex_{\cY}$ maps contexts to probability distributions over actions, where $\simplex_{\cY}$ is the probability simplex over $\cY$. Let $\policies$ be the space of all such policies. For a context $x \in \cX$, we define the expected preference of a policy $\policy \in \policies$ over $\policy'$, as:
\[
    \cP(\policy \succ \policy' | x) \triangleq \E_{y \sim \policy(\cdot|x), y' \sim \policy'(\cdot|x)}\left[ \cP(y \succ y' | x) \right]\,.
\]
For a context distribution $\rho \in \simplex_{\cX}$ we define the expected preference and Kullback-Leibler (KL) divergence as $\cP(\policy \succ \policy') = \E_{x \sim \rho}[\cP(\policy \succ \policy'|x)]$ and $\KL_\rho(\policy \Vert \policy') = \E_{x \sim \rho}[\KL(\policy(x) \Vert \policy'(x)]$. 

\textbf{Regularized preference game.}
For a fixed reference policy $\piref \in \policies$ we define a \emph{$\beta$-regularized preference} as $\cP_\beta(\pi \succ \pi') \triangleq \cP(\pi \succ \pi') - \beta \KL(\pi \Vert \piref) + \beta \KL(\pi' \Vert \piref)$, where $\beta > 0$ is a regularization parameter. The $\beta$-regularized preference function induces a $\beta$-regularized preference game \citep{munos2023nash}, which is symmetric, so it admits a symmetric Nash Equilibrium (NE) $(\pistar_\beta, \pistar_\beta)$ \citep{neumann1928zur,nash1950equilibrium}. We call a policy $\pistar_\beta$ a $\beta$-regularized von Neumann winner (VNW, \citealt{dudik2015contextual}) and in particular it satisfies 
\[
    \textstyle \pistar_\beta \in \argmax_{\pi \in \policies} \min_{\pi' \in \policies} \cP_\beta( \pi \succ \pi')\,.
\]
The $\beta$-regularized suboptimality (also known as exploitability gap) of $\policy$ against $\policy'$ is $\subopt_{\beta}(\policy, \policy') \triangleq \frac{1}{2} - \cP_{\beta}(\policy \succ \policy')$.
The worst-case $\beta$-regularized suboptimality of $\policy$ is:
\begin{equation}\label{eq:def_reg_subopt}
    \textstyle \subopt_{\beta}(\policy) \triangleq \frac{1}{2} - \min_{\policy' \in \policies} \cP_\beta(\policy \succ \policy')\,.
\end{equation}
A policy $\policy$ is an $\epsilon$-VNW in the $\beta$-regularized game if $\subopt_\beta(\policy)\le \epsilon$. Our goal is to learn an $\varepsilon$-VNW given (i) sampling access to the context distribution $\rho$, (ii) the ability to sample from context-conditioned policies, and (iii) a stochastic estimate of $\cP(y \succ y' \mid x)$ (e.g., from pairwise comparisons). We measure efficiency via the iteration complexity $N_{\mathrm{iter}}(\varepsilon)$, i.e., the number of algorithmic updates, and the sample complexity $N_{\mathrm{sample}}(\varepsilon)$, i.e., the number of queries to a comparison oracle.

\textbf{Value and best-response.} Due to the symmetry of the game, it is enough to consider only the point of view of the min-player, the value of its policy  $\pi \in \policies$ against a competitor policy $\mu \in \policies$ is
\begin{equation}\label{eq:value_definition}
    V_{\beta}(\pi; \mu) \triangleq \cP(\mu \succ \pi) + \beta \KL_\rho(\pi \Vert \piref)\,.
\end{equation}
Thus, we can define a best-response policy $\nu^\star_\beta(\mu) \in \argmin_{\pi \in \policies} V_\beta(\pi; \mu)$ and its value as $V^\star_\beta(\mu)$. We can rewrite the worst-case suboptimality in terms of value as $\subopt_\beta(\pi) = V_\beta(\pi;\pi) - V^\star_\beta(\pi)$.

\textbf{Additional notations.} For a vector $x \in \R^d$ we define a span seminorm of $x$ as $\norm{x}_{\spann} = \inf_{c \in \R} \norm{x + c\bOne}_{\infty}$, where $\bOne = (1,\ldots,1)^\top$ is a vector of all ones. For $M > 0$ define $\clip_{[-M,M]}(x) = \max\{ -M, \min\{x, M\}\}$. For a distribution $\rho \in \simplex_{\cX}$ and two functions $f \colon \cX \to \R^d$, define $\norm{f}_{1,\rho}^2 \triangleq \E_{x \sim \rho}[\norm{f(x)}_1^2]$ and $\norm{f}_{\spann,\rho}^2 \triangleq \E_{x \sim \rho}[\norm{f(x)}^2_{\spann}]$\,.

\section{Self-Play Policy Gradients}\label{sec:spg_main}

In this section, we consider a simple algorithm for finding a VNW of a $\beta$-regularized preference game with a general parameterized policy class and a finite action space. We call this algorithm the Self-Play Policy Gradient (SPG) method.

For a parameter $\theta \in \Theta = \R^d$, we define $\theta \mapsto \pi_\theta \in \policies$ as a differentiable policy parameterization. Define a parametrized value as $J_\beta(\theta; \mu) \triangleq V_\beta(\pi_\theta;\mu)$ where $\mu$ is a competitor policy. Given this definition and an initial parameter $\theta_0 \in \Theta$, we define the iterates of SPG as follows
\begin{equation}\label{eq:spg_update_main}
  \theta_{t+1}=\cT(\theta_t-\gamma_t g_t),\qquad \pi_{t+1}=\pi_{\theta_{t+1}}\,,
\end{equation}
where $g_t$ is a stochastic gradient estimator of $\nabla J_\beta(\theta_t; \pi_t)$, and $\cT$ is an optional projection-like policy improvement operator (e.g., clipping). As a main example, we consider the following mini-batch pairwise REINFORCE gradient estimator \citep{williams1992simple}, defined as $g_t = (1/B_t) \sum_{j=1}^{B_t} G_j(\theta_t)$ for
\begin{equation}\label{eq:pairwise_reinforce_main}
    G_j(\theta) \triangleq (\nabla_\theta \log \pi_\theta(y_j | x_j) - \nabla_\theta \log \pi_\theta(y_j' | x_j)) \cdot \clip_{[-M,M]}(A_j)\,,
\end{equation}
where $M$ is a clipping threshold and $A_j$ is an advantage estimate defined as
\begin{equation}
     \textstyle A_j \triangleq  \tfrac{1}{2} - p_j + \beta \left(\log\frac{\pi_\theta(y|x)}{\piref(y|x)} - \log\frac{\pi_\theta(y'|x)}{\piref(y'|x)}\right) \,,
\end{equation}
and where $x_j \sim \rho$, $y_j,y'_j \sim \pi_t(x_j)$ and $p_j$ is a sample from a Bernoulli distribution of parameter $\cP(y_j \succ y'_j | x_j)$. This gradient estimator can be considered as a leave-one-out (LOO) advantage estimator \citep{kool2019buy,ahmadian2024back} with a group of size 2, which was also applied to estimate the KL-divergence.

\textbf{Connection with Online IPO.} Recall the Online IPO loss
\begin{equation}\label{eq:online_ipo_loss}
    \textstyle \cL_{\IPO}(\theta) \triangleq \E_{x \sim \rho, y,y' \sim \mathtt{sg}(\pi_\theta(\cdot|x)),p \sim \Ber(\cP(y \succ y' | x))}\left[ \left( \log\bigg(\frac{\pi_\theta(y|x)}{\pi_\theta(y'|x)} \frac{\piref(y'|x)}{\piref(y|x)}\bigg) - \frac{p}{2\beta} \right)^2 \right]\,,
\end{equation}
where $\mathtt{sg}$ as a stop-gradient operation.
It is known \citep[Proposition 4.2]{calandriello2024human} that expected gradients of Online IPO loss are the same of the self-play updates. However, it is easy to check the sample-based version of~\eqref{eq:online_ipo_loss} have exactly the same (stochastic) gradients as a pairwise REINFORCE estimator~\eqref{eq:pairwise_reinforce_main} used in our analysis, up to a constant scaling and clipping. In particular, it allows us to extend all our theoretical results to a parametrized version of online IPO, and also directly interpret Online IPO as a LOO-variance-reduced self-play policy gradient update with a group size of 2.

\subsection{Theoretical Guarantees}

We emphasize that the primary challenge in the parametrized setting lies in the intrinsic non-convexity of the value $J_\beta(\theta;\mu)$ with respect to its first argument. Consequently, in the absence of additional structural assumptions, global convergence of the method cannot be guaranteed.

\begin{assumption}[Parametrization regularity, informal] A pair $(\theta \mapsto \pi_\theta, \cT)$ satisfies the following properties:
\textbf{(A1)} Lipschitz parametrisation
$\|\pi_\theta-\pi_{\theta'}\|_{1,\rho}\le G\|\theta-\theta'\|_2$;
\textbf{(A2)} $L$-smoothness of $\theta\mapsto J_\beta(\theta;\mu)$;
\textbf{(A3)} (approx.) Polyak–Łojasiewicz (PL) inequality:
$\|\nabla J_\beta(\theta;\pi_\theta)\|_2^2+\varepsilon_{\rm PL}\ge 2\pl\subopt_\beta(\pi_\theta)$
for $\theta$ in the range of $\cT$;
\textbf{(A4)} $\cT$ non-increase suboptimality;
\textbf{(A5)} $g_t$ has a bias $\le \varepsilon_{\grad}$ and is subgaussian with a variance-proxy $\propto M^2/B_t$, where $M$ is a clipping threshold and $B_t$ is a batch size.
\end{assumption}

These conditions hold for context-free softmax policies $\pi_\theta(y)\propto \exp(\theta_y)$ with a suitable choice of $\cT$, and more generally for Fisher-nondegenerate compatible parameterizations~\citep{yuan2022general}, where $\epspl$ reflects function-approximation error. Appendix~\ref{app:self_play_pg} provides the full assumptions and verification. The constants may depend on $\beta$ and the reference policy $\piref$; in particular, the clipping level $M$ depends on the bias $\varepsilon_{\grad}$ and regularity of $\piref$.

\begin{theorem}[Convergence guarantees of SPG, informal] Assume \textbf{(A1)}--\textbf{(A5)} and $\beta\pl \geq G^2$. Let $\kappa\triangleq L/\pl$. For the iterates
of self-play policy gradients (SPG,~\ref{eq:spg_update_main}):

\vspace{-4pt}
\textbf{Deterministic:} with exact gradients and $\gamma_t \equiv 1/(2L)$,
\[
  \subopt_\beta(\pi_t) \leq (1-\gamma \pl/2)^t\,\subopt_\beta(\pi_0)+\cO(\epspl/\pl).
\]

\vspace{-9pt}
\textbf{Stochastic:} with $\gamma_t=\Theta(1/(\pl t))$ and a growing batch size $B_t=\Theta(t/\pl)$, with probability at least $1-\delta$,
\[
  \textstyle \subopt_\beta(\pi_t) = \tcO\left( \kappa^2/t^2 + \kappa M^2/t^2 + G^2M^2/(\beta \pl \cdot t) + (\varepsilon_{\rm PL}+\varepsilon_{\grad}^2)/\pl\right)\,.
\]
\end{theorem}

\vspace{-9pt}
We refer to Propositions~\ref{prop:spg_selfplay_deterministic} and \ref{prop:spg_selfplay_stochastic} in Appendix~\ref{app:self_play_pg} for complete statements and proofs. In the bound above, for the stochastic case, the first term corresponds to deterministic convergence, and the next two terms correspond to different noise terms: the first comes from standard PL-SGD arguments~\citep{madden2024high}, whereas the second comes from a \emph{moving best-response}. The last term corresponds to an accumulated bias from an inexact PL inequality or biased gradient updates. Before discussing the assumptions, we demonstrate the source of the $1/(\beta m \cdot t)$-term (moving best-response) in the stochastic bound.

\begin{lemma}[Descent Lemma I]
Assume \textbf{(A1)},\textbf{(A2)}, and \textbf{(A4)}. Let $\tilde{L} \triangleq L + G^2/(4\beta)$. For the iterates of self-play policy gradients~\eqref{eq:spg_update_main},
for any sequence $(\gamma_t)_{t \geq 0}$ and for any $t \geq 0$, it holds that
\begin{align*}
   \subopt_\beta(\pi_{t+1})
  &\le
  \sqrt{
      \subopt_\beta(\pi_{t})
      -\gamma_t  \big\langle \nabla J_\beta(\theta_{t}; \pi_t), g_t \big\rangle
      +  
         \tfrac{\gamma_t^2 \tilde{L}}{2} \Vert  g_t \Vert^2_{2}
    } + \sqrt{\tfrac{\gamma_t^2 G^2\norm{g_t}_2^2}{8\beta}}\,.
\end{align*}
\end{lemma}

\vspace{-10pt}
We refer to Lemma~\ref{lem:descent_lemma_I} in Appendix~\ref{app:convergence_sppg} for a complete proof. In the following, we provide a proof sketch to demonstrate where the influence of moving baselines appears. In fact, the additional term $\sqrt{\gamma_t^2 \norm{g_t}_2^2/\beta}$ is exactly the reason to introduce assumptions on $\beta$ and a growing batch size.

\vspace{-9pt}
\begin{proof}[Sketch of the proof.]
For simplicity of exposition, assume $\cT \equiv \mathrm{Id}$. Then the smoothness condition \textbf{(A2)} implies for $f_{t}(\theta) \equiv J_\beta(\theta; \pi_t) \equiv V_\beta(\pi_\theta; \pi_t)$ and $\theta_{t+1} = \theta_t - \gamma_t g_t$
\[
    f_t(\theta_{t+1}) \leq f_t(\theta_t) - \gamma_t \langle \nabla f_t(\theta_t), g_t \rangle + \tfrac{\gamma^2_t L}{2} \norm{g_t}_2^2\,.
\]
Define $f^\star_t = V^\star_\beta(\pi_t)$, and using a relation $S_{t}^2 = f_t(\theta_t) - f^\star_t$ and simple algebraic manipulations
\begin{align*}
    S_{t+1}^2 &\leq S_t^2  - \gamma_t \langle \nabla f_t(\theta_t), g_t \rangle + \tfrac{\gamma_t^2 L}{2} \norm{g_t}_2^2  + \underbrace{[f_{t+1}(\theta_{t+1}) - f_t(\theta_{t+1})] + [f^\star_t- f^\star_{t+1}]}_{:=\mathcal{R}_t}\,.
\end{align*}
For the term $\cR_t$, we first notice that in the difference $ f_{t+1}(\theta_{t+1}) - f_t(\theta_{t+1})$ the KL-terms cancels out, thus we have $f_{t+1}(\theta_{t+1}) - f_t(\theta_{t+1}) = \cP(\pi_{t+1} \succ \pi_{t+1}) - \cP(\pi_{t}\succ \pi_{t+1})$, which we denote as $\cP(\pi_{t+1} - \pi_t \succ \pi_{t+1})$ using the linearity of expectation. For the final term, we use a connection between $f^\star_t \equiv V^\star_\beta(\pi_{t})$ and a convex conjugate of entropy, thus $f^\star_{t+1} - f^\star_t \geq \cP(\pi_{t+1} - \pi_t \succ \nu_{t+1})$, where $\nu_{t+1}$ is a best response against $\pi_{t+1}$. As a result, using bilinearity and bi-Lipschitzness of preferences
\begin{align*}
    \cR_t \leq \cP(\pi_{t+1} -\pi_t \succ \pi_{t+1} - \nu_{t+1}) \leq \tfrac{1}{2}\norm{\pi_{t+1} - \pi_t}_{1,\rho} \norm{\nu^\star_{t+1} - \pi_{t+1}}_{1,\rho}\,.
\end{align*}
Finally, using Assumption~\textbf{(A1)}, Pinsker's inequality, and the connection $\beta \KL_\rho(\pi_{t+1} \Vert \nu^\star_{t+1}) = S_{t+1}^2$, we have $\cR_t \leq (\gamma_t/2)\sqrt{ G^2 \norm{g_t}_2^2 \cdot S_{t+1}^2 / \beta}$. Plugging this inequality into the inequality connecting $S_{t+1}^2$ and $S_t^2$ and solving the resulting quadratic inequality, we conclude the statement.
\end{proof}

\vspace{-4pt}

\textbf{Why growing batches?}
Unlike standard analysis of SGD under the PL-assumption on a fixed objective, the exploitability
$\subopt_\beta(\pi_t)=V_\beta(\pi_t;\pi_t)-V_\beta^\star(\pi_t)$ uses a \emph{moving}
best-response baseline. Controlling its drift yields an intrinsic noise term of order
$\gamma_t\|\xi_t\|_2^2$, which motivates the need of a batch size $B_t$ increasing with $t$. 
We also note that an implicit assumption about growing batch sizes appears in the existing analysis of the extragradient method \citep{zhou2025extragradient}, where the noise term is scaled by $\sigma^2/\beta^2$, with $\sigma^2$ denoting the gradient variance. As remarked by \citet{madden2024high}, similar issues arise in the analysis of projected gradient descent in the non-convex setting, and are typically resolved by a similar growing mini-batch condition~\citep{ghadimi2016mini} or by variance-reduction techniques~\citep{reddi2016proximal}.

\textbf{Why $\beta m \geq G^2$? } This condition ensures the best-response map is well-contractive and one-step improvement is achievable; similar ``sufficient regularization'' assumptions also appear in Mirror-Prox-style analyses, see~\citet{nemirovski2004prox} for a discussion. We note that in the convex case \citep{sokota2022unified}, it was shown that a similar algorithm does not require such assumption; however, in a parametrized non-convex setting, it is unclear whether this is an artifact of the analysis or a necessity for the method's stability. We note that empirically the method may not achieves monotonic improvements in suboptimality over the course of training in the case of small $\beta$, making the last-iterate convergence non-achievable; see Appendix~\ref{app:rps_experiment} for numerical experiments. We leave the complete identification of the gap between these two settings as a direction for future work.

For the tabular softmax and Fisher-compatible parameterizations, the following corollary summarizes the resulting iteration and sample complexity. The rates depend on the regularity of the reference policy under $\rho$: it controls the tails of the stochastic gradient estimator and hence the clipping level $M$. This yields two regimes: the context-free and the contextual settings.

\begin{corollary}[SPG iteration and sample complexity, informal]\label{cor:spg_complexity_main}
Fix confidence level $\delta \in (0,1)$ and target accuracy $\varepsilon \in (0,1)$. Assume
$\E_{x\sim \rho}\!\left[\log^2\!\big(1/\piref_{\min}(x)\big)\right] < \infty$,
where $\piref_{\min}(x) \triangleq \min_{y\in \cY} \piref(y\mid x)$.
Then SPG returns an $\cO(\varepsilon+\epspl)$-VNW for the $\beta$-regularized game in the following settings:
\begin{itemize}[leftmargin=1.25em,nosep]
\item[(i)]
\emph{Context-free, tabular softmax}~\citep{mei2020global}. With clipping
$M=\Theta(\log(1/\varepsilon))$ and $\beta \gtrsim \beta_{\min}$ (where $\beta_{\min}$ depends on the minimum reference mass),
\[
  N_{\mathrm{iter}}(\varepsilon)=T=\tcO\!\left(\varepsilon^{-1}\right),
  \qquad
  N_{\mathrm{sample}}(\varepsilon)=\tcO\!\left(\varepsilon^{-2}\right).
\]

\item[(ii)]
\emph{Contextual, Fisher non-degenerate compatible parameterization}~\citep{yuan2022general}. With clipping
$M=\Theta(1/\sqrt{\varepsilon})$ and $\beta \gtrsim \beta_{\min}$ (where $\beta_{\min}$ depends on the Fisher conditioning),
\[
  N_{\mathrm{iter}}(\varepsilon)=T=\tcO\!\left(\varepsilon^{-2}\right),
  \qquad
  N_{\mathrm{sample}}(\varepsilon)=\tcO\!\left(\varepsilon^{-4}\right).
\]
\end{itemize}
\end{corollary}

We refer to Corollary~\ref{cor:spg_complexity_softmax} and ~\ref{cor:spg_complexity_fisher} in Appendix~\ref{app:self_play_pg} for proofs and complete statements. We also notice that under a stronger condition $\E_{x\sim \rho}[1/\piref_{\min}(x)]<\infty$, one can use a milder clipping level and recover the same improved rates in both the contextual and context-free cases. For general context spaces, however, $\E[1/\piref_{\min}(x)]<\infty$ can be substantially more restrictive than the log-moment condition.

\section{Proximal Point Method with Self-Play Policy Gradients}\label{sec:algo}

In this section, we propose a way to overcome a theoretical limitation of the self-play policy-gradient (SPG) method by embedding it in the Proximal Point (PP) method.

\subsection{(Approximate) Proximal Point Method}

A well-known method for computing a NE is the Proximal Point (PP) method \citep{martinet1970breve,rockafellar1976monotone}. At each iteration, the PP method computes the NE of an auxiliary game that is additionally regularized toward the previous iterate. To formally define the iterates $(\pi_k)_{k \geq 0}$, we initialize $\pi_0\equiv \piref$, and for a PP step size $\eta > 0$ we define $\pi_{k+1}$ as a solution to the following game:
\begin{equation}\label{eq:pp_subgame_main}
\textstyle \max_{\pi\in\policies}\min_{\mu\in\policies}
  \Bigl\{
    \cP_\beta(\pi\succ\mu)
    - (\beta/\eta)\KL_\rho(\pi\Vert\pi_k)
    +(\beta/\eta)\KL_\rho(\mu\Vert\pi_k)
  \Bigr\}\,,
\end{equation}
which is equivalent to the original game, up to additional regularization toward the previous iterate $\pi_k$. Alternatively, we can view $\pi_{k+1}$ as a fixed point of the best-response operator:
\begin{equation}\label{eq:proximal_point_iterations}
    \textstyle \pi_{k+1} \in \argmin_{\pi \in \policies} V_k(\pi; \pi_{k+1})\,, \qquad V_k(\pi;\mu) \triangleq V_{\beta}(\pi;\mu) + (\beta/\eta) \KL_\rho(\pi \Vert \pi_k).
\end{equation}

This approach often yields accelerated convergence rates for solving games, but it is primarily a \emph{conceptual} method rather than a \emph{practical} one, since exactly computing the NE of the regularized game can be just as challenging as solving the original game. The key observation is that the PP subproblems in \eqref{eq:proximal_point_iterations} induce stronger effective regularization than the original game (when $\eta$ is small), which allows us to address the theoretical limitation of self-play policy gradients. This motivation is similar to the motivation behind the Mirror-Prox algorithm \citep{nemirovski2004prox}, which efficiently approximates a PP iteration by performing two best-response steps for a well-chosen value of $\eta$.

\textbf{PP residual and convergence guarantees.}
As discussed earlier, computing \eqref{eq:proximal_point_iterations} exactly is challenging, especially in the function-approximation setting. The exact iterates can equivalently be characterized by $\nabla_\pi V_k(\pi_{k+1},\pi_{k+1}) \propto \bOne$, i.e., the gradient vector becomes constant. This motivates the following definition: we call $\pi_{k+1}$ an $\varepsilon_{\pp}$-approximate PP update if
\begin{equation}\label{eq:pp_residual_main}
  \|\nabla_\pi V_k(\pi_{k+1},\pi_{k+1})\|_{\spann,\rho}^2 \le \varepsilon_{\pp}\,,
\end{equation}
where $\|\cdot\|_{\spann,\rho}^2$ is an expected (over contexts) span seminorm of a vector; by definition, it equals zero on constant vectors. Under this definition, it is possible to derive the following guarantees; for the full statement and proof, see Appendix~\ref{app:analysis_pp}.

\begin{proposition}[Approximate PP convergence]\label{prop:approx_pp_contraction_main}
Assume that \eqref{eq:pp_residual_main} holds for all $k\ge 0$. Then the
iterates of approximate PP converge toward $\pistar_\beta$ up to a residual floor:
\begin{equation}\label{eq:pp_kl_contraction_main}
  \textstyle \KL_\rho(\pistar_\beta\Vert\pi_k)
  \le
  (1+\eta/2)^{-k}\,\KL_\rho(\pistar_\beta\Vert\pi_0)
  + 2\varepsilon_{\pp} / \beta^2\,,
\end{equation}
and similarly $\subopt_\beta(\pi_k)$ and $\|\log\pi_k-\log\pistar_\beta\|_{\spann,\rho}$ decays at the rate $(1+\eta/2)^{-k}$ with an
additive $\cO(\varepsilon_{\pp})$ term (see
Proposition~\ref{prop:approx_pp_convergence_context} in the Appendix~\ref{app:analysis_pp} for explicit constants).
\end{proposition}

\textbf{Choice of $\eta$.}
Larger $\eta$ improves the outer contraction factor in
\eqref{eq:pp_kl_contraction_main}, but it also weakens the proximal regularizer
$\beta/\eta$ and can make the inner games harder to solve. This is the same
trade-off that motivates the choice $\eta=\cO(\beta)$ in Mirror-Prox-style algorithms for this problem
\citep{cen2024fast}, which resolves the inner problem using two iterations of the best-response operator.

\subsection{PP--SPG Method}\label{sec:pp_spg_main}

We now describe how we implement the inexact PP step
\eqref{eq:pp_subgame_main} using SPG, and summarize the resulting convergence
guarantees.

\textbf{Outer--inner structure.}
We set $\pi_0=\piref$ and run $K$ outer PP steps. At outer iteration
$k$, we fix $\pi_k$ and (approximately) solve the proximal subgame
\eqref{eq:pp_subgame_main} by running SPG for $T_k$ inner steps. Denote

\vspace{-12pt}
\[
  \textstyle \beta_{\target}\triangleq \beta/\eta,
  \qquad
  \lambda\triangleq\beta+\beta_{\target}=\beta\bigl(1+\frac{1}{\eta}\bigr),
  \qquad
  \tilde{\pi}_k \propto [\piref]^{\beta/\lambda} \times [\pi_k]^{\beta_{\target}/\lambda}.
\]

\vspace{-6pt}
The inner problem is exactly a $\lambda$-regularized preference game anchored at $\tilde{\pi}_k$, i.e., it is
of the same form as the game analyzed for SPG in Section~\ref{sec:spg_main}.

\textbf{Inner objective for SPG.}
Let $\pi_\theta$ be a parametrized policy. The $\min$-player
objective against policy $\pi$ at the outer step $k$ is
\begin{equation}\label{eq:pp_spg_inner_obj_main}
  J_k(\theta;\pi)
  \triangleq
  \cP(\pi\succ\pi_\theta)
  +\beta\,\KL_\rho(\pi_\theta\Vert\piref)
  +\beta_{\target}\,\KL_\rho(\pi_\theta\Vert\pi_k).
\end{equation}
Self-play corresponds to using $\pi=\pi_\theta$ in $J_k$ (as in Section~\ref{sec:spg_main}), and updating parameters with stochastic gradients and an appropriately chosen improvement
operator $\cT_k \colon \Theta \to \Theta$ to ensure regularity:
\begin{equation}\label{eq:pp_spg_update_main}
  \theta_{k,t+1}=\cT_k\left(\theta_{k,t}-\gamma_{k,t}\,g_{k,t}\right),
  \qquad
  \pi_{k,t+1}=\pi_{\theta_{k,t+1}}.
\end{equation}
At the end of the inner loop, we set $\pi_{k+1}\triangleq \pi_{k,T_k}$. A
concrete example of $g_{k,t}$ is the clipped pairwise REINFORCE estimator as in Section~\ref{sec:spg_main}, see Algorithm~\ref{alg:pp_spg} in Appendix~\ref{app:pp-spg} for a formal definition and the pseudo-code.

\textbf{Connection to COMAL \citep{liu2024comal}.}
Our method is closely related in spirit to COMAL, which is motivated as a practical implementation of the conceptual prox method for \emph{unregularized} preference games. COMAL proceeds by repeatedly solving a KL-regularized subgame centered at a reference policy, and then updating the reference to the newly computed iterate. In contrast, our goal is to compute the Nash equilibrium of the \emph{fixed} $\beta$-regularized game (anchored at $\pi_{\mathrm{ref}}$), and we apply an (inexact) proximal-point outer loop \emph{on top of} this baseline regularization. Concretely, each PP outer step adds an \emph{additional} KL term toward the previous iterate, which strengthens the effective regularization of the inner problem and yields an improved contraction guarantee for the outer loop. This additional proximal regularizer allows us to overcome the limitation of plain self-play policy gradients in our setting and derive end-to-end convergence rates.

\subsection{Theoretical Guarantees}

Approximate PP requires a bound on the residual
\eqref{eq:pp_residual_main}, which is a gradient in \emph{policy space}.
Our inner solver operates in \emph{parameter space}. We bridge this gap via a
gradient-compatibility condition (Assumption~\ref{ass:gradient_compatibility} in Appendix~\ref{app:analysis_pp_spg})
showing that, up to constants,
\[
  \textstyle \|\nabla_\pi V_k(\pi_\theta,\pi_\theta)\|_{\spann,\rho}^2
  \ \lesssim\
  \|\nabla_\theta J_k(\theta;\pi_\theta)\|_2^2\,.
\]
Notably, this assumption holds under both tabular softmax and Fisher-compatible parameterizations. Combined with the smoothness of $J_k$, this yields the key implication
(Lemma~\ref{lem:inner_residual_vs_subopt}):
if the inner loop returns $\pi_{k+1}$ with
$\subopt^{\tilde{\pi}_k}_\lambda(\pi_{k+1})\le \varepsilon_{\rin}$, then
\begin{equation}\label{eq:inner_to_pp_main}
  \textstyle  \|\nabla_\pi V_k(\pi_{k+1},\pi_{k+1})\|_{\spann,\rho}^2
  \le
  \varepsilon_{\pp}
  \qquad\text{with}\qquad
  \varepsilon_{\pp}
  =
  \cO(\varepsilon_{\rin})\ \ (\text{explicit in Appendix~\ref{app:pp-spg}}).
\end{equation}
Therefore, controlling the \emph{inner} SPG suboptimality
suffices to control the \emph{outer} PP residual, which allows us to provide end-to-end guarantees for the combined method. The main caveat in the analysis is that, because our anchor sequence $(\tpi_k)_{k\geq 0}$ depends on the previous iterate $\pi_k$, we cannot uniformly control its behavior over all contexts, but we still can control the second moment thanks to established convergence in $\norm{\log \pi_k - \log \pistar_\beta}_{\spann,\rho}$.

\begin{theorem}[PP--SPG sample complexity, informal]
\label{thm:pp_spg_sample_complexity}
Fix arbitrary $\beta\le 1$, confidence $\delta\in(0,1)$, and target accuracy
$\varepsilon\in(0,1)$. Run PP--SPG for $K=\tcO(1)$ outer steps,
and set the inner target accuracy $\varepsilon_{\rin}=\Theta(\varepsilon)$.
Then, with probability at least $1-\delta$, the output $\pi_K$ satisfies
$\subopt_\beta(\pi_K)\le \cO(\varepsilon + \epspl)$. Moreover, the combined method achieves the same order of total iteration and sample complexity as in Corollary~\ref{cor:spg_complexity_main}.
\end{theorem}

See Corollaries~\ref{cor:final_rates_softmax} and~\ref{cor:final_rates_fisher} in Appendix~\ref{app:analysis_pp_spg} for full statements and proofs. The contextual rate improves to $\tcO(\varepsilon^{-2})$ under stronger assumptions on $\cT$ and $\piref$. With PP alone, we control only a second log-moment, so gradients may be heavy-tailed. SVRG-style variance reduction could further reduce sampling complexity~\citep{reddi2016proximal}; we leave this direction to future work.

\subsection{Nash Prox: Practical Deep Learning Implementation} \label{sec:implementation}
Next, we provide a practical implementation, which we call $\algo$. We consider the following policies: an online policy $\pi_t$ with parameters $\theta_t$, a \emph{target policy} $\pi^\mathrm{target}_t$ with parameters $\theta^{\mathrm{target}}_t$, and a fixed reference policy $\piref$. Next, we define the following IPO-style loss function (cf.~\eqref{eq:online_ipo_loss})
\begin{equation}\label{eq:prox_ipo_loss}
    \textstyle \cL_{\mathtt{Prox}}(\theta; \pi^\mathrm{target},\piref) \triangleq \E_{x \sim \rho,\, y,y' \sim \mathtt{sg}(\pi_\theta(\cdot|x))}\left[ \left( \ell_\theta(x,y) - \ell_\theta(x,y')- \tfrac{p-1/2}{\beta + \beta_{\target}} \right)^2 \right]\,,
\end{equation}
where $\ell_\theta(x,y) \triangleq (\beta/\lambda) \cdot \log (\pi_\theta(y|x)/\piref(y|x)) + (\beta_\target/\lambda) \cdot \log (\pi_\theta(y|x)/\pi^{\target}(y|x))$ for $\lambda \triangleq \beta + \beta_{\target}$, and $p$ is an estimator of $\cP(y \succ y' \mid x)$.
As explained before, the stochastic gradient of this loss function corresponds to our method, but it is easier to implement in practice. To implement the algorithm exactly as we discussed in Section~\ref{sec:pp_spg_main}, we need to update the target policy with the weights of the online policy every $T_k$ steps, where $T_k$ defines the update schedule. However, following a very common strategy in deep reinforcement learning \citep{mnih2015humanlevel,lillicrap2016continuous}, we employ a more practical and elegant approach and update the target with an exponential moving average instead, resulting in the following updates:
\[
    \theta_{t+1} = \theta_t - \alpha_t \cdot \nabla_{\theta} \cL_{\mathtt{Prox}}(\theta_t; \pi^\mathrm{target}_t,\piref)\,, \qquad  \theta^{\mathrm{target}}_{t+1} = (1-\kappa) \cdot \theta^{\mathrm{target}}_t + \kappa \cdot \theta_{t+1}\,,
\]
where $\alpha_t$ is a learning rate and the parameter $\kappa\in[0,1]$ implicitly controls the number of steps for one proximal update. Thus, we approximate the solution of one proximal subproblem with $T_k \approx 1/\kappa$ gradient steps. Also, we found it empirically beneficial to anneal the value of $\kappa$ as $\kappa_t = 1/(c \cdot t + 1)$ for some value of $c$, where $K \approx 1/c$ corresponds to the approximate number of outer PP iterates.

\section{Experiments}

We use the simple contextual dueling bandit problem as an initial experiment to study the deep learning implementation of $\algo$. We refer to Appendix~\ref{app:rps_experiment} for an additional small-scale experiment on the Rock-Paper-Scissors game.

\textbf{Game definition.}
Let us fix a number of actions $Y \geq 2$ and a positive integer $r \geq 1$. We consider a dueling bandit game with a context space $\cX = \R^{r \times r}$ and an action space $\cY = \{1, \ldots, Y\}$. Preference probabilities are defined as follows
\[
    \cP(y \succ y' | x) \triangleq \sigma( A_{y,y'} - A_{y',y} )\,,
\]
for $A \triangleq U \Theta_x V^\top$, where $U \in \R^{Y \times r}$ and $V \in \R^{Y \times r}$ are fixed matrices, $\Theta_x \in \R^{r \times r}$ is a corresponding context matrix, and $\sigma(\cdot)$ is a sigmoid function. This type of dueling bandit instance is a generalization of a low-rank linear bandit problem. Notice that for any $r \geq 2$ this problem does not admit a Bradley-Terry model. The distribution over contexts $\rho$ is assumed to be a standard Gaussian random matrix (i.e., elements of $\Theta_x$ are i.i.d. with distribution $\cN(0,1)$). We aim to find a policy $\pi \colon \cX \to \simplex_{\cY}$ that approximates a $\beta$-regularized VNW. We refer to Appendix~\ref{app:experiments} for more details on the setup.

\begin{wrapfigure}{r}{0.55\textwidth}
  \centering
  \vspace{-24pt} 
  \includegraphics[width=\linewidth]{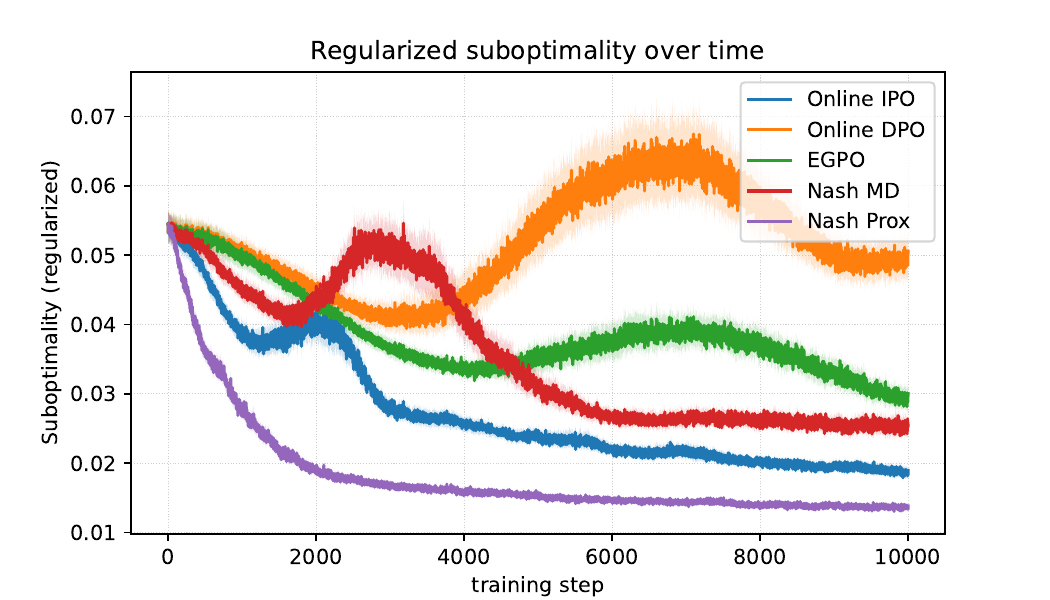}
  \caption{Comparison of $\algo$ (in \textcolor{violet}{violet}) with baseline methods. Suboptimality is averaged over 25 random seeds; shaded regions indicate standard error.}
  \label{fig:baselines_matrix_game}
  \vspace{-9pt} 
\end{wrapfigure}

\textbf{Results.}
We compare our method with the following baselines: Online IPO~\citep{calandriello2024human}, Online DPO~\citep{guo2024direct},  EGPO~\citep{zhou2025extragradient}, Nash MD~\citep{munos2023nash}, and $\algo$. The results are presented in Figure~\ref{fig:baselines_matrix_game}. We can observe that $\algo$ with an adaptive $\kappa_t=1/(0.3 \cdot t+1)$ and $\beta_{\target} = 10 \times \beta$ outperforms all the baselines. We also observe that the two-step stabilization procedure of EGPO is insufficient to stabilize training in the function approximation setting, while Online DPO diverges because it is not designed to solve the underlying preference game in the first place. Additionally, we observe non-monotonous behavior of suboptimality for all the baselines, including online IPO. \emph{This empirical fact is a direct demonstration of a necessity of the stability condition on $\beta$ for last-iterate convergence guarantees, which is typically relies on contraction principle.} In Appendix~\ref{app:experiments}, we provide an additional ablation study on the choice of $\kappa$ on the optimization procedure as well as a comparison between hard and soft updates.

\subsection{LLM Alignment}\label{subsec:llm_experiment}

\textbf{Experiment setup.} For our LLM-based experiments, we use the Gemma-3-4B \citep{team2025gemma} pretrained model checkpoints and train on the RLHFlow \citep{dong2024rlhf} datasets for all the analysis. In particular, we first perform SFT on the RLHFlow SFT dataset \citep{rlhflow_sft} and then all our NLHF experiments use LoRA adapters~\citep{hu2022lora} of rank 16 with $\alpha=32$ on the resulting checkpoint, using a subset of RLHFlow Prompt collection \citep{rlhflow_prompt}. The pairwise judge model is Gemma-2-2B trained via the Robust Reward Models \citep{liu2025rrm} method. All the experiments were performed using the TRL library \citep{vonwerra2022trl}.

\textbf{Baselines.} We compare the practical version of $\algo$ against Online DPO \citep{guo2024direct} and Online IPO \citep{calandriello2024human}. We also considered $\NashMD$ \citep{munos2023nash} and EGPO \citep{zhou2025extragradient}, but did not include them in our experiments because their current implementations are prohibitively slow in our setting (we estimate on the order of $10^3$ GPU-hours per configuration). For $\NashMD$, the main bottleneck is the need to sample from geometric mixtures, while for EGPO the reference implementation incurs substantial overhead due to repeated optimizer off-loading and reloading. Additional details on baselines and hyperparameters are provided in Appendix~\ref{app:experiments}.

\textbf{Results.} We report in Table~\ref{tab:winrates_sft_final_scriptstyle_ci_v2} the pairwise win-rates between the different methods on a held-out test set of prompts for the separate Gemma3-27B-IT judge model. We observe that $\algo$ outperforms all the baselines. From the implementation side, the main difference between $\algo$ and Online IPO is the use of additional regularization with respect to a target model, and our results show that this regularization can be valuable.

\begin{table*}[tbp]
\centering
\caption{Pairwise win rates (mean $\pm$ 99.9\%-confidence intervals). Statistically significant wins are in \textbf{bold}. Confidence intervals are in a smaller font size. Row/column for $\protect\algo$ is highlighted.}
\begin{tabular}{lcccc}
\hline
Win rate & SFT & Online DPO & Online IPO  &\gc{$\algo$} \\
\hline
SFT
& $-$
& $0.0674 {\scriptstyle \pm 0.011}$
& $0.0636 {\scriptstyle \pm 0.011}$
& \gc{$0.0567 {\scriptstyle \pm 0.010}$} \\

Online DPO
& $\boldsymbol{0.9326} {\scriptstyle \pm 0.011}$
& $-$
& $\boldsymbol{0.5283} {\scriptstyle \pm 0.017}$
& \gc{$0.4656 {\scriptstyle \pm 0.017}$} \\

Online IPO
& $\boldsymbol{0.9364} {\scriptstyle \pm 0.011}$
& $0.4717 {\scriptstyle \pm 0.017}$
& $-$
& \gc{$0.4373 {\scriptstyle \pm 0.017}$} \\

\gc{$\algo$}
& \gc{$\boldsymbol{0.9433} {\scriptstyle \pm 0.010}$}
& \gc{$\boldsymbol{0.5344} {\scriptstyle \pm 0.017}$}
& \gc{$\boldsymbol{0.5627} {\scriptstyle \pm 0.017}$}
& \gc{$-$} \\
\hline
\end{tabular}%
\label{tab:winrates_sft_final_scriptstyle_ci_v2}

\vspace{-0.3cm}
\end{table*}

\section{Conclusion}\label{sec:conclusion}

In this work, we study the problem of computing Nash equilibria in KL-regularized preference games that arise in NLHF, where the regularizer enforces proximity to a reference policy. We first analyze a simple self-play policy gradient method under general policy parameterizations and establish last-iterate, high-probability convergence guarantees. The analysis also highlights a structural difficulty of self-play, namely that the best-response baseline changes across iterations and may force a stringent effective-regularization condition. To remove this restriction, we embed self-play within a proximal point scheme, whose proximal subgames induce stronger regularization and yield last-iterate convergence without requiring a lower bound on the original regularization strength. Finally, we instantiate this framework with practical stochastic policy-gradient updates and EMA-based stabilization mechanisms, resulting in a concrete algorithm, $\algo$, that matches state-of-the-art empirical performance while retaining guarantees in the parameterized setting.

\section*{Impact Statement}
This paper aims to advance methods for aligning learning systems with human feedback without assuming transitive preferences.
Potential societal impacts include improved safety and robustness of interactive AI systems, and risks include mis-specifying human preferences or reinforcing annotator biases.

\section*{Acknowledgements}
The work of D.Tiapkin has been supported by the Paris Île-de-France Région in the framework of DIM AI4IDF.
The work of E. Moulines has been funded by the European Union (ERC-2022-SYG-OCEAN-101071601). Views and opinions expressed are however those of the author(s) only and do not necessarily reflect those of the European Union or the European Research Council Executive Agency. Neither the European Union nor the granting authority can be held responsible for them.
This project was provided with computing and storage resources by GENCI at IDRIS thanks to the grant 2025-AD011016276 on the supercomputer Jean Zay's A100 and H100 partitions.

\bibliography{ref}
\bibliographystyle{icml2026}

\newpage
\appendix
\section{Notation}\label{app:notation}

\begin{table}[!ht]
\centering
\small
\setlength{\tabcolsep}{6pt}
\renewcommand{\arraystretch}{1.15}
\begin{tabular}{p{0.22\linewidth} p{0.73\linewidth}}
\hline
\textbf{Symbol} & \textbf{Meaning} \\
\hline
$\cX,\cY$ & Context space and finite action set; $|\cY|=Y$. \\
$\rho$ & Distribution over contexts $x\in\cX$. \\
$\pi,\mu$ & Policies $\cX\to \simplex_{\cY}$; $\pi(y|x)$ is action-probability. \\
$\cP(y\succ y' \mid x)$ & Preference probability; satisfies $\cP(y\succ y'|x)+\cP(y'\succ y|x)=1$. \\
$\bP_x$ & Preference matrix at context $x$: $(\bP_x)_{y,y'}=\cP(y\succ y'|x)$. \\
$\cP(\pi\succ\mu\mid x)$ & Contextual preference: $\pi(x)^\top \bP_x \mu(x)$. \\
$\cP(\pi\succ\mu)$ & Overall preference: $\E_{x\sim\rho}[\pi(x)^\top \bP_x \mu(x)]$. \\
$\langle f,g\rangle_\rho$ & Inner product: $\E_{x\sim\rho}[\langle f(x),g(x)\rangle]$. \\
$\|\cdot\|_{1,\rho},\ \|\cdot\|_{\infty,\rho}$ & $\sqrt{\E_{x\sim\rho}[\|\cdot\|_1^2]}$ and $\sqrt{\E_{x\sim\rho}[\|\cdot\|_\infty^2]}$. \\
$\|\cdot\|_{\spann,\rho}$ & Span seminorm in context: $\sqrt{\E_{x\sim\rho}[\|\cdot\|_{\spann}^2]}$. \\
$\KL_\rho(\pi\Vert\mu)$ & Contextual KL: $\E_{x\sim\rho}[\KL(\pi(x)\Vert\mu(x))]$. \\
$\tpi$ & Anchor/reference policy (generic). \\
$\cP^{\tpi}_\lambda(\pi\succ\mu)$ & $\lambda$-regularized payoff:
$\cP(\pi\succ\mu)-\lambda\KL_\rho(\pi\Vert\tpi)+\lambda\KL_\rho(\mu\Vert\tpi)$. \\
$\subopt^{\tpi}_\lambda(\pi,\mu)$ & Regularized suboptimality: $\tfrac12-\cP^{\tpi}_\lambda(\pi\succ\mu)$. \\
$V^{\tpi}_\lambda(\pi,\mu)$ & Regularized value: $\cP(\mu\succ\pi)+\lambda\KL_\rho(\pi\Vert\tpi)$. \\
$\nu^{\tpi,\star}_\lambda(\mu)$ & Best response: $\argmin_{\pi} V^{\tpi}_\lambda(\pi,\mu)$. \\
$V^{\tpi,\star}_\lambda(\mu)$ & Best-response value: $\min_\pi V^{\tpi}_\lambda(\pi,\mu)$. \\
$\beta$ & Regularization strength of the ``outer/original'' preference game. \\
$\eta$ & Proximal-point (PP) step size; $\beta_{\target}=\beta/\eta$ is the proximal penalty. \\
$\pi_t$ & Outer PP iterate (policy at PP step $t$). \\
$V_t(\pi,\mu)$ & PP inner value: $V_\beta(\pi,\mu)+(\beta/\eta)\KL_\rho(\pi\Vert\pi_t)$. \\
$\zeta_{t+1}$ & PP residual: $\nabla_\pi V_t(\pi_{t+1},\pi_{t+1})$. \\
$\varepsilon$ & PP inner-solve accuracy target: $\|\zeta_{t+1}\|_{\spann,\rho}^2\le \varepsilon$. \\
$\theta,\pi_\theta$ & Parametric policy: parameter $\theta\in\R^d$ with induced policy $\pi_\theta$. \\
$J^{\tpi}(\theta;\pi)$ & Best-response objective: $\cP(\pi\succ \pi_\theta)+\lambda\KL_\rho(\pi_\theta\Vert\tpi)$. \\
$g_t,\xi_t,b_t,B_t$ & Stochastic gradient, noise, bias, and mini-batch size in SPG. \\
$\pl_{\tpi,\lambda},\epspl$ & (Approx.) PL constant and additive PL slack in Assumption~\ref{ass:pl}. \\
\hline\\
\end{tabular}
\caption{Notation used in Appendix~\ref{app:analysis_pp}--\ref{app:pp-spg}.}
\label{tab:appendix_notation}
\end{table}

The space of discrete probability measure of dimension $d$ is denoted by $\simplex_d = \{ p \in \R^d : p_i \geq 0, \sum_{i=1}^d p_i = 1\}$. In the same manner, we associate for any set $\cY$ of size $Y = |\cY|$ $\simplex_{\cY}$ with $\simplex_{Y}$. Given two discrete probability measure $p,q \in \simplex_d$, let us define the KL-divergence between them as $\KL(p \Vert q) \triangleq \sum_{i=1}^d p_i \log(p_i/q_i)$. For two vectors $x,y \in \R^d$, we define an inner product as $\langle x, y \rangle \triangleq \sum_{i=1} x_i y_i$.
\section{Analysis of Proximal Point Method}\label{app:analysis_pp}

In this section, we analyze the exact and approximate proximal point methods.

\subsection{Preference Properties}

We work in the contextual dueling bandit setting $(\cX,\cY,\cP,\rho)$ described in Section~\ref{sec:setting}. The action space $\cY$ is finite with $|\cY| = Y \ge 2$, and in each round, a context $x \in \cX$ is drawn from $\rho$. We assume that the context space is finite or countable for simplicity. We define a policy space $\policies$ as a functional space with a domain $\cX$ and a codomain $\simplex_{\cY}$.

For every context $x\in\cX$ and actions $y,y'\in\cY$, we are given preference probabilities $\cP(y \succ y' \mid x) \in [0,1]$ that satisfy the symmetry assumption
\[
    \cP(y \succ y' \mid x) + \cP(y' \succ y \mid x) = 1
    \qquad \forall x \in \cX,\, y,y' \in \cY.
\]

For each context $x$, we define the preference matrix $\bP_x \in [0,1]^{|\cY|\times|\cY|}$ with entries $(\bP_x)_{y,y'} \triangleq \cP(y \succ y' \mid x)$. For policies $\pi,\mu\colon \cX \to \simplex_{\cY}$, we define the contextual preference
\[
    \cP(\pi \succ \mu \mid x)
        \triangleq \E_{y \sim \pi(x),\,y' \sim \mu(x)}
        \big[ \cP(y \succ y' \mid x) \big]
        = \pi(x)^\top \bP_x \mu(x)\,,
\]
and the overall preference
\[
    \cP(\pi \succ \mu)
        \triangleq \E_{x\sim\rho}\big[\cP(\pi \succ \mu \mid x)\big]
        = \E_{x\sim\rho}\big[\pi(x)^\top \bP_x \mu(x)\big]\,.
\]

More generally, for measurable functions $v,u\colon \cX \to \R^{|\cY|}$ (not necessarily probability distributions), we extend the notation as
\begin{equation}\label{eq:context_bilinear_form}
    \cP(v \succ u \mid x) \triangleq v(x)^\top \bP_x u(x),
    \qquad
    \cP(v \succ u) \triangleq \E_{x\sim\rho}\big[ v(x)^\top \bP_x u(x) \big].
\end{equation}

On such functions, we use the following inner product and norms:
\[
    \langle f,g \rangle_\rho \triangleq \E_{x\sim\rho}[\langle f(x), g(x) \rangle]\,,
    \
    \|f\|_{1,\rho} \triangleq \sqrt{\E_{x\sim\rho}[\|f(x)\|_1^2]}\,,\ 
    \|f\|_{\infty,\rho} \triangleq \sqrt{\E_{x\sim\rho}[\|f(x)\|_\infty^2]}\,.
\]
For a function $g\colon \cX \to \R^{|\cY|}$, we extend the span seminorm by
\[
    \|g\|_{\spann,\rho}^2 \triangleq \E_{x\sim\rho}\big[ \|g(x)\|_{\spann}^2 \big],
\]
where $\|\cdot\|_{\spann}$ on $\R^{|\cY|}$ is the span seminorm introduced in Section~\ref{sec:setting}. 

For any functional $\cF$ defined on a policy space $\policies$, we define its derivative at $\pi$ as the function $\nabla \cF(\pi)\colon \cX \to \R^{|\cY|}$ satisfying the directional derivative identity
\[
    \frac{\rmd}{\rmd t} \cF(\pi + t h)\Big|_{t=0}
    = \langle \nabla \cF(\pi), h \rangle_\rho
    \qquad\forall h\colon \cX \to \R^{|\cY|}, \langle h(x), \bOne \rangle = 0 \text{ for all } x \in \cX\,.
\]
In particular, it can be characterized as $(\nabla \cF(\pi))(x,y) = \frac{\partial}{\partial \pi(y|x)} \cF(\pi)$ for $\rho$-almost all $x$ and all $y \in \cY$.

Finally, for policies $\pi,\mu$ we set
\[
    \KL_{\rho}(\pi \Vert \mu) \triangleq
      \E_{x\sim\rho}[\KL(\pi(x)\Vert \mu(x))],
\]
and we recall Pinsker's inequality applied pointwise in $x$:
\[
    \|\pi(x)-\mu(x)\|_1^2 \le 2\,\KL(\pi(x)\Vert \mu(x)),
\]
which implies
\begin{equation}\label{eq:pinsker_context}
    \|\pi - \mu\|_{1,\rho}^2 \le 2\,\KL_\rho(\pi\Vert\mu).
\end{equation}
We assume throughout that all policies have full support on $\cY$ for $\rho$-almost all $x$, so that the expressions involving $\log\pi(y|x)$ are well-defined.

\begin{lemma}[Lipschitz property of preferences]\label{lem:lipschitz_preference_context}
    Let $u,v\colon \cX \to \R^{|\cY|}$ be measurable functions such that
    $\langle v(x),\bOne\rangle = 0$ for all $x\in\cX$. Then
    \[
        |\cP(u \succ v)| \le \frac{1}{2}\,\|u\|_{1,\rho}\,\|v\|_{1,\rho},\qquad
        |\cP(v \succ u)| \le \frac{1}{2}\,\|u\|_{1,\rho}\,\|v\|_{1,\rho}.
    \]
\end{lemma}
\begin{proof}
    Fix any $x \in \cX$. Using $\sum_{y'} v(x,y') = 0$, for any $y\in\cY$ we can write
    \[
        (\bP_x v(x))_y
        = \sum_{y'} \cP(y \succ y' \mid x)v(x,y')
        = \sum_{y'} (\cP(y \succ y' \mid x) - 1/2)v(x,y').
    \]
    Hence
    \[
        |(\bP_x v(x))_y|
        \le \max_{y'} |\cP(y \succ y' \mid x)-1/2|
            \sum_{y'} |v(x,y')|
        \le \frac{1}{2}\|v(x)\|_1,
    \]
    which implies $\|\bP_x v(x)\|_\infty \le \frac12\|v(x)\|_1$. Therefore,
    \[
        |\cP(u \succ v \mid x)|
        = |u(x)^\top \bP_x v(x)|
        \le \|u(x)\|_1\|\bP_x v(x)\|_\infty
        \le \frac{1}{2}\|u(x)\|_1\|v(x)\|_1.
    \]
    Taking expectation over $x\sim\rho$ and applying Cauchy--Schwarz,
    \[
        |\cP(u \succ v)|
        \le \frac12\,\E_{x\sim\rho}[\|u(x)\|_1\|v(x)\|_1]
        \le \frac12\,\|u\|_{1,\rho}\,\|v\|_{1,\rho}.
    \]
    The bound for $|\cP(v \succ u)|$ is obtained in the same way.
\end{proof}

\subsection{Regularized Preference and Value}

Fix a reference policy $\tpi \in \policies$ and $\lambda > 0$. We consider the contextual $\lambda$-regularized preference game
\begin{equation}\label{eq:reg_preference_game_context}
    \max_{\pi \in \policies} \min_{\mu \in \policies}
    \left\{
        \cP^{\tpi}_{\lambda}(\pi \succ \mu)
        \triangleq
        \cP(\pi \succ \mu)
        - \lambda \KL_{\rho}(\pi \Vert \tpi)
        + \lambda \KL_{\rho}(\mu \Vert \tpi)
    \right\}.
\end{equation}
The (contextual) regularized suboptimality is defined as
\[
    \subopt^{\tpi}_{\lambda}(\pi,\mu)
        \triangleq \frac{1}{2} - \cP^{\tpi}_{\lambda}(\pi \succ \mu), \qquad
    \subopt^{\tpi}_{\lambda}(\pi)
        \triangleq \max_{\mu\in\policies} \subopt^{\tpi}_{\lambda}(\pi,\mu).
\]

We introduce the regularized value function
\begin{equation}\label{eq:regularized_value_context}
    V_{\lambda}^{\tpi}(\pi,\mu)
        \triangleq \cP(\mu \succ \pi) + \lambda \KL_{\rho}(\pi \Vert \tpi)\,,
\end{equation}
and denote
\[
    \nu^{\tpi, \star}_{\lambda}(\mu)
        \in \argmin_{\pi \in \policies} V_{\lambda}^{\tpi}(\pi,\mu), \qquad
    V^{\tpi, \star}_{\lambda}(\mu)
        \triangleq \min_{\pi \in \policies} V_{\lambda}^{\tpi}(\pi,\mu)\,.
\]
Note that $\nu^{\tpi, \star}_{\lambda}(\mu)$ might be non-unique on $\rho$-zero measure sets, but $V^{\tpi, \star}_{\lambda}(\mu)$ is always well-defined.

\begin{lemma}[Exploitability gap via value function]\label{lem:exploitability_via_value_context}
    For any policy $\pi \in \policies$, it holds
    \[
        \subopt^{\tpi}_{\lambda}(\pi) = V_{\lambda}^{\tpi}(\pi,\pi) - V^{\tpi, \star}_{\lambda}(\pi)\,.
    \]
\end{lemma}
\begin{proof}
    By definition of $\subopt^{\tpi}_{\lambda}(\pi,\mu)$ and using $\cP(\mu\succ\pi)=1-\cP(\pi\succ\mu)$,
    \[
        \subopt^{\tpi}_{\lambda}(\pi,\mu)
        = \frac{1}{2} - \cP(\pi\succ \mu)
          + \lambda \KL_{\rho}(\pi\Vert\tpi)
          - \lambda \KL_{\rho}(\mu\Vert\tpi) 
        = V^{\tpi}_{\lambda}(\pi,\mu) - \tfrac{1}{2}
          - \lambda \KL_{\rho}(\mu\Vert\tpi)\,.
    \]
    Maximizing over $\mu$ yields
    \[
        \subopt^{\tpi}_{\lambda}(\pi)
        = \max_{\mu} \subopt^{\tpi}_{\lambda}(\pi,\mu)
        = \max_{\mu}
           \bigl(
              V_{\lambda}^{\tpi}(\pi,\mu)
              - \tfrac{1}{2}
              - \lambda \KL_{\rho}(\mu\Vert\tpi)
           \bigr).
    \]
    In particular, evaluating at $\mu=\pi$ gives
    \[
        V_{\lambda}^{\tpi}(\pi,\pi)
        = \frac{1}{2} + \lambda \KL_{\rho}(\pi\Vert\tpi).
    \]
    On the other hand,
    \[
        V^{\tpi, \star}_{\lambda}(\pi)
        = \min_{\nu \in \policies}
           \bigl(
              \cP(\pi\succ\nu) + \lambda\KL_{\rho}(\nu\Vert\tpi)
           \bigr),
    \]
    so
    \[
        V^{\tpi}_{\lambda}(\pi,\pi) - V^{\tpi, \star}_{\lambda}(\pi)
        = \max_{\nu \in \policies}
           \bigl(
              \tfrac{1}{2} + \lambda \KL_{\rho}(\pi\Vert\tpi)
              - \cP(\pi\succ\nu)
              - \lambda\KL_{\rho}(\nu\Vert\tpi)
           \bigr),
    \]
    where the right-hand side exactly matches the definition of $\subopt^{\tpi}_{\lambda}(\pi)$.
\end{proof}
\begin{lemma}[Strong convexity in $\pi$ with respect to $\KL_\rho$]
\label{lem:value_strong_convexity_context}
Fix $\mu,\tpi \in \policies$. Then, for any $\pi,\pi' \in \policies$,
\begin{equation}\label{eq:V_strong_convexity_identity}
    V_{\lambda}^{\tpi}(\pi,\mu) - V_{\lambda}^{\tpi}(\pi',\mu)
    = \langle \nabla_{\pi} V_{\lambda}^{\tpi}(\pi',\mu), \pi - \pi' \rangle_{\rho}
      + \lambda\KL_{\rho}(\pi \Vert \pi')\,,
\end{equation}
where $\nabla_\pi V^{\tpi}_\lambda(\pi',\mu)$, defined as a Fréchet derivative of $V^{\tpi}_\lambda(\pi',\mu)$ with respect to the first argument, is given element-wise by
\[
    \big(\nabla_\pi V^{\tpi}_\lambda(\pi',\mu)\big)(x,y)
    \triangleq \big(\bP_x^\top \mu(x)\big)_y
      + \lambda\Big(1 + \log\frac{\pi'(y|x)}{\tpi(y|x)}\Big),
\]      
In particular, if $\pi' = \nu^{\tpi,\star}_{\lambda}(\mu) \in \argmin_{\pi} V_{\lambda}^{\tpi}(\pi,\mu)$, then
\[
    V_{\lambda}^{\tpi}(\pi,\mu) - V^{\tpi, \star}_{\lambda}(\mu)
    \geq \lambda\,\KL_{\rho}(\pi \Vert \nu^{\tpi,\star}_{\lambda}(\mu)).
\]
\end{lemma}

\begin{proof}
Write
\[
    V^{\tpi}_\lambda(\pi,\mu)
    = \E_{x\sim\rho}\Big[
         \mu(x)^\top \bP_x \pi(x)
         + \lambda f_x(\pi(x))
      \Big],
    \qquad
    f_x(p) \triangleq \sum_{y\in\cY} p_y \log\frac{p_y}{\tpi(y|x)}.
\]
Fix $\mu,\tpi$ and $\pi'$. Let $h\colon\cX\to\R^{|\cY|}$ be any direction and consider
\[
    \Phi(t) \triangleq V^{\tpi}_\lambda(\pi' + t h,\mu).
\]
For each $x$ and $y$, the derivative of $p\mapsto p\log\frac{p}{\tau}$ at $p=p'$ is
\[
    \frac{\rmd}{\rmd p}\Big(p\log\frac{p}{\tpi(y|x)}\Big)\Big|_{p=p'}
    = \log\frac{p'}{\tpi(y|x)} + 1.
\]
Thus, pointwise in $x$,
\begin{align*}
    \frac{\rmd}{\rmd t}\Big[\mu(x)^\top \bP_x (\pi'(x)+t h(x))\Big]_{t=0}
    &= \mu(x)^\top \bP_x h(x), \\
    \frac{\rmd}{\rmd t}\Big[\lambda f_x(\pi'(x)+t h(x))\Big]_{t=0}
    &= \lambda \sum_{y} h(x,y)\Big( \log\frac{\pi'(y|x)}{\tpi(y|x)} + 1 \Big).
\end{align*}
Therefore
\[
    \Phi'(0)
    = \E_{x\sim\rho}\Big[
         \sum_{y} \Big(
            (\bP_x^\top\mu(x))_y + \lambda\big(1+\log\tfrac{\pi'(y|x)}{\tpi(y|x)}\big)
         \Big) h(x,y)
      \Big]
    = \langle \nabla_\pi V^{\tpi}_\lambda(\pi',\mu), h\rangle_\rho,
\]
with
\[
    \big(\nabla_\pi V^{\tpi}_\lambda(\pi',\mu)\big)(x,y)
    = (\bP_x^\top\mu(x))_y + \lambda\Big(1+\log\tfrac{\pi'(y|x)}{\tpi(y|x)}\Big).
\]

Now consider the difference
\[
    V^{\tpi}_\lambda(\pi,\mu) - V^{\tpi}_\lambda(\pi',\mu)
    = \E_{x\sim\rho}\Big[
         \mu(x)^\top \bP_x\big(\pi(x)-\pi'(x)\big)
         + \lambda\big(f_x(\pi(x)) - f_x(\pi'(x))\big)
      \Big].
\]
Similarly,
\[
    \langle \nabla_\pi V^{\tpi}_\lambda(\pi',\mu), \pi-\pi'\rangle_\rho
    = \E_{x\sim\rho}\Big[
         \mu(x)^\top \bP_x\big(\pi(x)-\pi'(x)\big)
         + \lambda \sum_y \big(1+\log\tfrac{\pi'(y|x)}{\tpi(y|x)}\big)
                         \big(\pi(y|x)-\pi'(y|x)\big)
      \Big].
\]
Subtracting, we isolate the Bregman term:
\begin{align*}
    &V^{\tpi}_\lambda(\pi,\mu) - V^{\tpi}_\lambda(\pi',\mu)
      - \langle \nabla_\pi V^{\tpi}_\lambda(\pi',\mu), \pi-\pi'\rangle_\rho \\
    &\quad = \lambda\,\E_{x\sim\rho}\Big[
          f_x(\pi(x)) - f_x(\pi'(x))
          - \sum_y \big(1+\log\tfrac{\pi'(y|x)}{\tpi(y|x)}\big)
                    \big(\pi(y|x)-\pi'(y|x)\big)
       \Big].
\end{align*}
Fix $x$ and abbreviate $p_y = \pi(y|x)$, $p'_y = \pi'(y|x)$, $\tau_y = \tpi(y|x)$. Then
\[
    f_x(p) = \sum_y p_y \log\frac{p_y}{\tau_y},
    \qquad
    \frac{\partial f_x}{\partial p_y}(p') = \log\frac{p'_y}{\tau_y} + 1.
\]
The inner expression is the Bregman divergence of $f_x$ at $p\Vert p'$:
\[
    f_x(p) - f_x(p') - \sum_y \frac{\partial f_x}{\partial p_y}(p')(p_y-p'_y).
\]
A standard calculation gives
\[
    f_x(p) - f_x(p') - \sum_y \frac{\partial f_x}{\partial p_y}(p')(p_y-p'_y)
    = \sum_y p_y \log\frac{p_y}{p'_y}
    = \KL(p\Vert p').
\]
Thus
\[
    V^{\tpi}_\lambda(\pi,\mu) - V^{\tpi}_\lambda(\pi',\mu)
      - \langle \nabla_\pi V^{\tpi}_\lambda(\pi',\mu), \pi-\pi'\rangle_\rho
    = \lambda\,\E_{x\sim\rho}[\KL(\pi(x)\Vert\pi'(x))]
    = \lambda\KL_\rho(\pi\Vert\pi'),
\]
which gives \eqref{eq:V_strong_convexity_identity}.

For the final claim, let $\pi' = \nu^{\tpi,\star}_\lambda(\mu)$ be a minimizer of $V^{\tpi}_\lambda(\cdot,\mu)$. First-order optimality implies
\[
    \langle \nabla_\pi V^{\tpi}_\lambda(\pi',\mu), \pi-\pi'\rangle_\rho \ge 0
    \qquad\forall \pi\in\policies.
\]
Plugging $\pi'$ into \eqref{eq:V_strong_convexity_identity} and using $V^{\tpi}_\lambda(\pi',\mu)=V^{\tpi,\star}_\lambda(\mu)$ yields
\[
    V^{\tpi}_\lambda(\pi,\mu) - V^{\tpi,\star}_\lambda(\mu)
    \ge \lambda\,\KL_\rho(\pi\Vert\pi')
    = \lambda\KL_\rho(\pi\Vert\nu^{\tpi,\star}_\lambda(\mu))\,,
\]
as claimed.
\end{proof}

\begin{lemma}[Smoothness of the best-response value]\label{lem:br_value_smoothness_context}
Fix a reference policy $\tpi \in \policies$ and $\lambda > 0$. Then, for any policies $\mu,\mu' \in \policies$,
\[
    \cP(\mu - \mu' \succ \nu^{\tpi,\star}_{\lambda}(\mu)) \leq V^{\tpi,\star}_{\lambda}(\mu) - V^{\tpi,\star}_{\lambda}(\mu') 
    \leq \cP(\mu - \mu' \succ \nu^{\tpi,\star}_{\lambda}(\mu'))\,.
\]
\end{lemma}
\begin{proof}
Note that the lower bound follows automatically from the upper bound by multiplying both sides by $-1$ and renaming $\mu$ and $\mu'$:
\[
V^{\tpi,\star}_{\lambda}(\mu) - V^{\tpi,\star}_{\lambda}(\mu') 
    \leq \cP(\mu - \mu' \succ \nu^{\tpi,\star}_{\lambda}(\mu')) \iff V^{\tpi,\star}_{\lambda}(\mu') - V^{\tpi,\star}_{\lambda}(\mu) 
    \geq \cP(\mu' - \mu \succ \nu^{\tpi,\star}_{\lambda}(\mu'))
\]
Thus, it is enough to prove only an upper bound.
Fix an arbitrary context $x \in \supp(\rho)$ and consider the regularizer $q \in \simplex_{\cY}$ as
\[
    \phi_x(q) \triangleq \lambda \KL(q \Vert \tpi(x))\,,
\]
and we define its convex conjugate $\phi_x^*$ on $g \in \R^{\cY}$ as
\[    \phi_x^*(g)
    \triangleq \sup_{q \in \simplex_{\cY}} \big\{ \langle g, q \rangle - \phi_x(q) \big\}
    = \sup_{q \in \simplex_{\cY}} \big\{ \langle g, q \rangle - \lambda \KL(q \Vert \tpi(x)) \big\}.
\]
We note that the convex conjugate is always a convex function. As a result, we have
\begin{equation}\label{eq:logsumexp_smoothness}
    \forall g,h \in \R^Y: \phi_x^*(g) - \phi_x^*(h) - \langle \nabla \phi_x^*(h), g - h \rangle
    \geq 0\,,
\end{equation}
where $\nabla \phi_x^*(g)$ is the unique maximizer in the definition of $\phi_x^*(g)$. 

Next, we relate $V^{\tpi,\star}_{\lambda}(\mu)$ to $\phi_x^*$. For any $\mu,\pi$,
\[
    V^{\tpi}_\lambda(\pi,\mu)
    = \cP(\mu\succ\pi) + \lambda\KL_\rho(\pi\Vert\tpi)
    = \E_{x \sim \rho}\left[ \langle \bP^\top_x \mu(x), \pi(x)\rangle + \lambda\KL(\pi(x)\Vert\tpi(x))\right]\,,
\]
Thus, using the fact that optimization is performed over all functions from $\cX$ to $\simplex_{\cY}$,
\begin{align*}
    V_\lambda^{\tpi,\star}(\mu)
    &= \min_{\pi\in\policies} V^{\tpi}_\lambda(\pi,\mu)
    = - \sup_{\pi\in\policies} \E_{x \sim \rho}\left[ \langle -\bP^\top_x \mu(x), \pi(x)\rangle - \lambda\KL(\pi(x)\Vert\tpi(x))\right]
    \\
    &= - \E_{x \sim \rho}\left[ \sup_{q \in \simplex_\cY} \langle -\bP^\top_x \mu(x), q\rangle - \lambda\KL(q\Vert\tpi(x))\right] = - \E_{x \sim \rho}\left[ \phi_x^*(-\bP^\top_x \mu(x)) \right]\,.
\end{align*}
Also, for any $x \in \supp(\rho)$, the unique minimizer of $V^{\tpi}_\lambda(\cdot,\mu)$ at context $x$ is given by
\[
    [\nu^{\tpi,\star}_\lambda(\mu)](x)
    = \argmin_{q \in \simplex_\cY} \langle \bP^\top_x \mu(x), q \rangle + \lambda\KL(q\Vert\tpi(x))
    = \nabla \phi_x^*(-\bP^\top_x \mu(x))\,.
\]
Using these relations, we can now prove the smoothness property of $V^{\tpi,\star}_\lambda$. Applying \eqref{eq:logsumexp_smoothness} pointwise at each $x \in \supp(\rho)$ with $g = -\bP^\top_x \mu(x)$, $h = -\bP^\top_x \mu'(x)$,
we get
\begin{align*}
    \phi_x^*(-\bP^\top_x \mu(x)) - \phi_x^*(-\bP^\top_x \mu'(x)) - \langle \nabla \phi_x^*(-\bP^\top_x \mu'(x)), -\bP^\top_x (\mu(x) - \mu'(x)) \rangle \geq 0\,.
\end{align*}
Taking expectation over $x \sim \rho$ and multiplying by $-1$, we obtain
\begin{align*}
    \E_{x \sim \rho}\left[ -\phi_x^*(-\bP^\top_x \mu(x)) \right] - \E_{x \sim \rho}\left[ -\phi_x^*(-\bP^\top_x \mu'(x)) \right] - \E_{x \sim \rho}\left[ \langle \nabla \phi_x^*(-\bP^\top_x \mu'(x)), \bP^\top_x (\mu(x) - \mu'(x)) \rangle \right] \leq 0\,.
\end{align*}
Rewriting using the relations to $V^{\tpi,\star}_\lambda$ and $\nu^{\tpi,\star}_\lambda$, we get
\begin{align*}
    V^{\tpi,\star}_\lambda(\mu) - V^{\tpi,\star}_\lambda(\mu') - \cP(\mu - \mu' \succ \nu^{\tpi,\star}_\lambda(\mu')) \leq 0 \,.
\end{align*}
\end{proof}

\begin{lemma}[Suboptimality decomposition]\label{lem:suboptimality_decomposition_context}
    Let $\pi^{\tpi,\star}_{\lambda}$ be a solution to \eqref{eq:reg_preference_game_context}. Then, for any policies $\pi, \mu \in \policies$, it holds
    \[
        \subopt^{\tpi}_\lambda(\pi, \mu) \leq  \subopt^{\tpi}_\lambda(\pi, \pi^{\tpi,\star}_{\lambda}) + \frac{1}{2}\,\|\pi - \pi^{\tpi,\star}_{\lambda}\|_{1,\rho} \cdot \|\mu - \pi^{\tpi,\star}_{\lambda}\|_{1,\rho}\,.
    \]
\end{lemma}
\begin{proof}
    By definition,
    \begin{align*}
        \subopt^{\tpi}_\lambda(\pi, \mu)
        &= \frac{1}{2} - \cP(\pi \succ \mu) + \lambda \KL_{\rho}(\pi \Vert \tpi) - \lambda \KL_{\rho}(\mu \Vert \tpi) \\
        &= \underbrace{\frac{1}{2} - \cP(\pi \succ \pi^{\tpi,\star}_{\lambda}) + \lambda \KL_{\rho}(\pi \Vert \tpi) - \lambda \KL_{\rho}(\pi^{\tpi,\star}_{\lambda} \Vert \tpi)}_{\subopt^{\tpi}_\lambda(\pi, \pi^{\tpi,\star}_{\lambda})} \\
        &\quad + \underbrace{\frac{1}{2} - \cP(\pi^{\tpi,\star}_{\lambda} \succ \mu) + \lambda \KL_{\rho}(\pi^{\tpi,\star}_{\lambda} \Vert \tpi) - \lambda \KL_{\rho}(\mu \Vert \tpi)}_{\subopt^{\tpi}_\lambda(\pi^{\tpi,\star}_{\lambda}, \mu)} \\
        &\quad + \cP(\pi - \pi^{\tpi,\star}_{\lambda} \succ \pi^{\tpi,\star}_{\lambda} - \mu)\,,
    \end{align*}
    where, by bilinearity~\eqref{eq:context_bilinear_form} and symmetry, we used
    \[
        \cP(\pi - \pi^{\tpi,\star}_{\lambda} \succ \pi^{\tpi,\star}_{\lambda} - \mu)
        = \cP(\pi \succ \pi^{\tpi,\star}_{\lambda}) - \cP(\pi \succ \mu) - \cP(\pi^{\tpi,\star}_{\lambda}\succ \pi^{\tpi,\star}_{\lambda}) + \cP(\pi^{\tpi,\star}_{\lambda} \succ \mu),
    \]
    and $\cP(\pi^{\tpi,\star}_{\lambda} \succ \pi^{\tpi,\star}_{\lambda}) = 1/2$.

    Since $\pi^{\tpi,\star}_{\lambda}$ is a Nash equilibrium of the regularized game, we have
    \[
        \cP^{\tpi}_{\lambda}(\pi^{\tpi,\star}_{\lambda} \succ \mu)
        \ge \frac{1}{2}
    \]
    for all $\mu$, hence $\subopt^{\tpi}_\lambda(\pi^{\tpi,\star}_{\lambda}, \mu) \le 0$. It remains to bound the last term. For each $x$, both $\pi^{\tpi,\star}_{\lambda}(x)$ and $\mu(x)$ are distributions, so $\sum_y (\pi^{\tpi,\star}_{\lambda}(y|x)-\mu(y|x))=0$. Therefore, by Lemma~\ref{lem:lipschitz_preference_context} applied to $u(x)=\pi(x)-\pi^{\tpi,\star}_{\lambda}(x)$, $v(x)=\pi^{\tpi,\star}_{\lambda}(x)-\mu(x)$,
    \[
        \bigl|\cP(\pi - \pi^{\tpi,\star}_{\lambda} \succ \pi^{\tpi,\star}_{\lambda} - \mu)\bigr|
        \leq \frac{1}{2} \|\pi - \pi^{\tpi,\star}_{\lambda}\|_{1,\rho} \cdot \|\pi^{\tpi,\star}_{\lambda} - \mu\|_{1,\rho}\,.
    \]
    Combining the above inequalities yields
    \[
        \subopt^{\tpi}_\lambda(\pi, \mu)
        \leq  \subopt^{\tpi}_\lambda(\pi, \pi^{\tpi,\star}_{\lambda})
            + \frac{1}{2} \|\pi - \pi^{\tpi,\star}_{\lambda}\|_{1,\rho} \cdot \|\mu - \pi^{\tpi,\star}_{\lambda}\|_{1,\rho}\,.
    \]
\end{proof}

\subsection{Proximal Point Method}

We now specialize to the $\beta$-regularized contextual game with respect to $\piref\in\policies$:
\[
    \max_{\pi \in \policies} \min_{\mu \in \policies} \left\{ \cP_{\beta}(\pi \succ \mu) \triangleq  \cP(\pi \succ \mu) - \beta \KL_{\rho}(\pi \Vert \piref) + \beta \KL_{\rho}(\mu \Vert \piref)\right\}\,.
\]
We define
\[
    V_{\beta}(\pi, \mu) \triangleq \cP(\mu \succ \pi) + \beta \KL_{\rho}(\pi \Vert \piref)\,, \qquad V^\star_\beta(\mu) \triangleq \min_{\pi \in \policies}V_{\beta}(\pi, \mu)\,,
\]
and let $\pistar_\beta$ denote the $\beta$-regularized von Neumann winner in this contextual game. The suboptimality is
\[
    \subopt_{\beta}(\pi,\mu) \triangleq \tfrac12 - \cP_{\beta}(\pi\succ\mu),
    \qquad
    \subopt_{\beta}(\pi) \triangleq \max_{\mu} \subopt_{\beta}(\pi,\mu).
\]

\paragraph{Method description.} The iterates of the approximate proximal point (PP) method can be written as follows:
\begin{equation}\label{eq:approximate_pp_iterates_context}
    \pi_{t+1} \approx \argmax_{\pi \in \policies} \min_{\mu \in \policies}\left\{ \cP_{\beta}(\pi \succ \mu) - \frac{\beta}{\eta} \KL_{\rho}(\pi \Vert \pi_t) + \frac{\beta}{\eta} \KL_{\rho}(\mu \Vert \pi_t)\right\}\,,
\end{equation}
where $\eta > 0$ is a PP learning rate. To define the success criteria for approximation, we introduce the regularized value
\[
    V_t(\pi, \mu) \triangleq V_{\beta}(\pi,\mu) + \frac{\beta}{\eta} \KL_{\rho}(\pi \Vert \pi_t)\,, \qquad V^\star_t(\mu) \triangleq \min_{\pi \in \policies} V_t(\pi, \mu)\,.
\]
For the exact solution to \eqref{eq:approximate_pp_iterates_context}, the updated policy $\pi_{t+1}$ is its own best response, i.e., $\nabla_{\pi} V_t(\pi_{t+1}, \pi_{t+1})(x) \propto \bOne$ for all $x \in \cX$. In the contextual setting, we measure the approximation error by
\[
    \big\|\nabla_{\pi} V_t(\pi_{t+1}, \pi_{t+1})\big\|_{\spann,\rho}^2
    \triangleq
    \E_{x\sim\rho}\Big[\big\|\nabla_{\pi} V_t(\pi_{t+1}, \pi_{t+1})(x)\big\|_{\spann}^2\Big].
\]
We say that $\pi_{t+1}$ is an $\varepsilon$-approximation of \eqref{eq:approximate_pp_iterates_context} if
\[
    \big\|\nabla_{\pi} V_t(\pi_{t+1}, \pi_{t+1})\big\|_{\spann,\rho}^2 \leq \varepsilon.
\]
For brevity, we denote $\zeta_{t+1} \triangleq \nabla_{\pi} V_t(\pi_{t+1}, \pi_{t+1})$.

\paragraph{Equivalence with a mixed-reference value.}
For each $x\in\cX$, consider the regularization term
\[
    \beta \KL(\pi(x)\Vert\piref(x)) + \frac{\beta}{\eta} \KL(\pi(x)\Vert\pi_t(x)).
\]
A direct calculation shows that there exists a policy $\tilde{\pi}_t$ such that
\[
    \tilde{\pi}_t(y \mid x) \propto [\piref(y \mid x)]^{\eta/(1+\eta)}[\pi_t(y \mid x)]^{1/(1+\eta)},
\]
and
\[
    \beta \KL(\pi(x)\Vert\piref(x)) + \frac{\beta}{\eta} \KL(\pi(x)\Vert\pi_t(x))
    = \beta\Big(1+\frac1\eta\Big)\KL(\pi(x)\Vert\tilde{\pi}_t(x)) + C_t(x),
\]
where $C_t(x)$ is independent of $\pi(x)$. Averaging over $x\sim\rho$, we obtain
\[
    V_t(\pi,\mu) = \cP(\mu\succ\pi) + \beta\Big(1+\frac1\eta\Big)\KL_{\rho}(\pi\Vert\tilde{\pi}_t) + \text{const}(\piref,\pi_t).
\]
Hence, by Lemma~\ref{lem:value_strong_convexity_context} with $\lambda = \beta(1+1/\eta)$ and reference $\tilde{\pi}_t$, we have
\begin{equation}\label{eq:Vt_strong_convexity}
    V_t(\pi, \mu) - V_t(\pi', \mu)
    = \langle \nabla_\pi V_t(\pi', \mu), \pi - \pi' \rangle_\rho
      + \beta\Big(1+\frac1\eta\Big)\KL_{\rho}(\pi\Vert\pi')
\end{equation}
for all $\pi,\pi',\mu\in\policies$.

\begin{proposition}[Convergence of approximate PP method]\label{prop:approx_pp_convergence_context}
    Assume that each iterate $\pi_{t+1}$ is an $\varepsilon$-approximation in the sense that $\big\|\nabla_{\pi} V_t(\pi_{t+1}, \pi_{t+1})\big\|_{\spann,\rho}^2 \leq \varepsilon$ for all $t \ge 0$. Then, for any $t \in \N$, it holds
    \[
        \KL_{\rho}(\pistar_\beta \Vert \pi_{t}) \leq (1+\eta/2)^{-t} \cdot \KL_{\rho}(\pistar_\beta \Vert \pi_0) + \frac{2 \varepsilon}{\beta^2}\,,
    \]
    and
    \[
        \subopt_{\beta}(\pi_{t}) \leq (1+\eta/2)^{-t} \cdot \left(\frac{1}{\beta} + \frac{2\beta}{\eta} + 3\right)  \cdot \KL_{\rho}(\pistar_\beta \Vert \pi_0)  + \left(2 + \frac{2}{\beta^3} + \frac{4}{\eta \beta} + \frac{6}{\beta^2}\right)\varepsilon\,,
    \]
    and, moreover,
    \[
        \big\|\log \pi_{t} - \log \pistar_\beta\big\|_{\spann,\rho}^2 \leq (1+\eta)^{-t} \big\|\log \pi_{0} - \log \pistar_\beta\big\|_{\spann,\rho}^2 + \frac{2}{\beta^2} \KL_{\rho}(\pistar_\beta \Vert \pi_0) \cdot (1+\eta/2)^{-t}
        + \frac{2(1+\beta^2) \cdot \varepsilon}{\beta^4}\,.
    \]
\end{proposition}
\begin{remark}[On the choice of $\eta$]
    The present bound suggests that taking $\eta \to +\infty$ is the best solution since it will drive convergence in just one iterate. However, the actual tradeoff is implicitly hidden in the value of $\varepsilon$: the speed of convergence of the internal subproblems depends on $\eta$ and a smaller $\eta$ makes these subproblems better-conditioned thanks to stronger regularization. In particular, the convergence of Mirror Prox-style methods requires $\eta = \cO(\beta)$ if transferred to our setting \citep{cen2024fast}.
\end{remark}
\begin{proof}
    We split the proof into three parts: convergence in $\KL_{\rho}$, convergence in suboptimality, and convergence in the span seminorm of log-probabilities.

    \paragraph{Convergence in $\KL_{\rho}$.}
    Define $A_t \triangleq \sqrt{\KL_{\rho}(\pistar_\beta \Vert \pi_t)}$. Applying~\eqref{eq:Vt_strong_convexity} with $\pi = \pistar_\beta$, $\pi' = \pi_{t+1}$, $\mu = \pi_{t+1}$, we obtain
    \begin{align*}
        \beta\Big(1+\frac1\eta\Big)\KL_{\rho}(\pistar_\beta\Vert\pi_{t+1})
        &= V_t(\pistar_\beta,\pi_{t+1}) - V_t(\pi_{t+1},\pi_{t+1})
           + \langle \nabla_\pi V_t(\pi_{t+1},\pi_{t+1}), \pi_{t+1}-\pistar_\beta\rangle_\rho \\
        &= V_{\beta}(\pistar_\beta,\pi_{t+1}) + \frac{\beta}{\eta}\KL_{\rho}(\pistar_\beta\Vert\pi_t)  - V_{\beta}(\pi_{t+1},\pi_{t+1}) - \frac{\beta}{\eta}\KL_{\rho}(\pi_{t+1}\Vert\pi_t) \\
        &\quad + \langle \zeta_{t+1}, \pi_{t+1}-\pistar_\beta\rangle_\rho.
    \end{align*}
    Using the definition of $V_{\beta}$ and $\cP_{\beta}$, we have
    \[
        V_{\beta}(\pistar_\beta,\pi_{t+1}) - V_{\beta}(\pi_{t+1},\pi_{t+1})
        = \cP_{\beta}(\pi_{t+1} \succ \pistar_\beta) - \frac12
        = -\subopt_{\beta}(\pi_{t+1},\pistar_\beta) \le 0.
    \]
    Dropping the non-negative term $\frac{\beta}{\eta}\KL_{\rho}(\pi_{t+1}\Vert\pi_t)$, we get
    \begin{equation}\label{eq:KL_recursion_raw}
        \beta\Big(1+\frac1\eta\Big)A_{t+1}^2
        \le \frac{\beta}{\eta}A_t^2
           + \langle \zeta_{t+1}, \pi_{t+1}-\pistar_\beta\rangle_\rho -\subopt_{\beta}(\pi_{t+1},\pistar_\beta)\,,
    \end{equation}
    where we can drop the last term too since it is non-positive.
    For each $x$, since $\sum_y (\pi_{t+1}(y|x)-\pistar_\beta(y|x))=0$, we have for any scalar function $c(x)$
    \[
        \langle \zeta_{t+1}(x), \pi_{t+1}(x)-\pistar_\beta(x)\rangle
        = \langle \zeta_{t+1}(x)-c(x) \cdot \bOne, \pi_{t+1}(x)-\pistar_\beta(x)\rangle.
    \]
    Thus
    \[
        \langle \zeta_{t+1}, \pi_{t+1}-\pistar_\beta\rangle_\rho
        \le \E_{x\sim\rho}\bigl[\|\zeta_{t+1}(x)-c(x) \cdot \bOne\|_\infty \|\pi_{t+1}(x)-\pistar_\beta(x)\|_1\bigr].
    \]
    Using Cauchy--Schwarz,
    \begin{align*}
        \langle \zeta_{t+1}, \pi_{t+1}-\pistar_\beta\rangle_\rho &\leq \E_{x\sim\rho}\bigl[\|\zeta_{t+1}(x)-c(x) \cdot \bOne\|_\infty \|\pi_{t+1}(x)-\pistar_\beta(x)\|_1\bigr] \\
        &\leq \sqrt{\E_{x \sim \rho}\left[ \|\zeta_{t+1}(x)-c(x) \cdot \bOne\|_\infty^2 \right] \cdot \E_{x \sim \rho}\left[ \|\pi_{t+1}(x)-\pistar_\beta(x)\|_1^2\right]}\,.
    \end{align*}
    Finally, minimizing over $c(x)$, we have
    \[
        \langle \zeta_{t+1}, \pi_{t+1}-\pistar_\beta\rangle_\rho \leq \norm{\zeta_{t+1}}_{\spann,\rho} \cdot \norm{\pi_{t+1} - \pistar_\beta}_{1,\rho}\,. 
    \]
    By Pinsker's inequality~\eqref{eq:pinsker_context},
    \[
        \|\pi_{t+1}-\pistar_\beta\|_{1,\rho}
        \le \sqrt{2\KL_{\rho}(\pistar_\beta\Vert\pi_{t+1})}
        = \sqrt{2}\,A_{t+1},
    \]
    hence
    \[
        \langle \zeta_{t+1}, \pi_{t+1}-\pistar_\beta\rangle_\rho
        \le \sqrt{2}\,\|\zeta_{t+1}\|_{\spann,\rho} A_{t+1}.
    \]
    Substituting into~\eqref{eq:KL_recursion_raw}, we obtain
    \[
        \beta\Big(1+\frac1\eta\Big)A_{t+1}^2
        \le \frac{\beta}{\eta}A_t^2
           + \sqrt{2}\,\|\zeta_{t+1}\|_{\spann,\rho} A_{t+1}.
    \]
    Dividing by $\beta(1+1/\eta)$, we have
    \[
        A_{t+1}^2 \leq \frac{1}{1+\eta} A_t^2 + 2 \cdot \frac{\eta \sqrt{2 \norm{\zeta_{t+1}}_{\spann, \rho}^2}}{2\beta \cdot (1+\eta)}  \cdot A_{t+1}\,.
    \]
    Solving this quadratic inequality in $A_{t+1}$ 
    \[
        A_{t+1} \leq \frac{\eta \sqrt{2 \norm{\zeta_{t+1}}_{\spann, \rho}^2}}{2\beta \cdot (1+\eta)} +\sqrt{\frac{1}{1+\eta} A_t^2 + \frac{\eta^2 \norm{\zeta_{t+1}}_{\spann, \rho}^2}{2\beta^2 \cdot (1+\eta)^2}} \,.
    \]
    Taking the square of both sides and using an inequality $(a+b)^2 \leq (1+\alpha)a^2 + (1+1/\alpha)b^2$ for any $\alpha > 0$
    \[
        A_{t+1}^2 \leq \frac{1+\alpha}{1+\eta} A_t^2 + (2 + \alpha + 1/\alpha)\frac{\eta^2 \norm{\zeta_{t+1}}_{\spann, \rho}^2}{2\beta^2 \cdot (1+\eta)^2}\,.
    \]
    Taking $\alpha$ as a solution to $(1+\alpha)/(1+\eta) = 1/(1+\eta/2)$, we have $2+\alpha+1/\alpha=(1+\eta)^2 /(\eta/2 \cdot (1+\eta/2))$, thus we achieve
    \[
        A_{t+1}^2 \leq \frac{1}{1+\eta/2}A_t^2 +  \frac{\eta \norm{\zeta_{t+1}}_{\spann, \rho}^2}{\beta^2 \cdot (1+\eta/2)} \leq \frac{1}{1+\eta/2}A_t^2 +  \frac{\eta \cdot \varepsilon}{\beta^2 \cdot (1+\eta/2)}\,,
    \]
    where the condition on a gradient error $\norm{\zeta_{t+1}}_{\spann, \rho}^2 \leq \varepsilon$ is applied. Unrolling this inequality, we achieve for any $T \in \N$
    \[
        A_T^2 \leq \frac{A_0^2}{(1+\eta/2)^{T}} + \frac{\eta \cdot \varepsilon}{\beta^2 (1+\eta/2)} \cdot \frac{1}{1-1/(1+\eta/2)} \leq \frac{A_0^2}{(1+\eta/2)^{T}} + \frac{2\varepsilon}{\beta^2}\,.
    \]
    \paragraph{Convergence in suboptimality.}
    Rearranging~\eqref{eq:KL_recursion_raw}, we have
    \[
        \subopt_\beta(\pi_{t+1}, \pistar_{\beta}) \leq (\beta/\eta) A_t^2 + \sqrt{2} \|\zeta_{t+1}\|_{\spann,\rho} \cdot A_{t+1}\,.
    \]
    Let $\mu^\star_{t+1} \in \argmax_{\mu} \subopt_{\beta}(\pi_{t+1}, \mu)$ be a worst-case opponent so that $\subopt_{\beta}(\pi_{t+1}) = \subopt_{\beta}(\pi_{t+1}, \mu^\star_{t+1})$. Applying Lemma~\ref{lem:suboptimality_decomposition_context} with $\pi^{\piref,\star}_{\lambda} = \pistar_\beta$ and $\lambda=\beta, \tpi=\piref$ and triangle inequality, we get
    \begin{align*}
        \subopt_\beta(\pi_{t+1}) &\leq (\beta/\eta) A_t^2 +  \sqrt{2} \|\zeta_{t+1}\|_{\spann,\rho} \cdot A_{t+1} + \frac{1}{2} \|\pi_{t+1} - \pistar_{\beta}\|_{1,\rho} \cdot \|\mu^\star_{t+1} - \pistar_\beta\|_{1,\rho} \\
        &\leq (\beta/\eta) A_t^2 +  \sqrt{2} \|\zeta_{t+1}\|_{\spann,\rho} \cdot A_{t+1} + \frac{1}{2} \|\pi_{t+1} - \pistar_{\beta}\|_{1,\rho} \cdot \|\mu^\star_{t+1} - \pi_{t+1}\|_{1,\rho} \\
        &\qquad\qquad+ \frac{1}{2} \norm{\pi_{t+1} - \pistar_\beta}_{1,\rho}^2\,.
    \end{align*}
    By Pinsker's inequality~\eqref{eq:pinsker_context},
    \[
        \|\pi_{t+1} - \pistar_{\beta}\|_{1,\rho}
        \leq \sqrt{2\KL_{\rho}(\pistar_\beta \Vert \pi_{t+1})} = \sqrt{2} \cdot A_{t+1}.
    \]
    Next, we show that
    \[
        \KL_{\rho}(\pi_{t+1} \Vert \mu^\star_{t+1}) \le \frac{1}{\beta} \subopt_\beta(\pi_{t+1}).
    \]
    For any $\pi,\mu$, a direct calculation gives
    \[
        V_{\beta}(\pi,\pi) - V_{\beta}(\mu,\pi)
        = \cP(\pi\succ\pi) + \beta\KL_{\rho}(\pi\Vert\piref)
          - \cP(\pi\succ\mu) - \beta\KL_{\rho}(\mu\Vert\piref)
        = \subopt_\beta(\pi,\mu).
    \]
    Therefore, for fixed $\pi$,
    \[
        \subopt_\beta(\pi)
        = \max_{\mu}\subopt_\beta(\pi,\mu)
        = V_{\beta}(\pi,\pi) - \min_{\mu}V_{\beta}(\mu,\pi).
    \]
    This implies that $\mu^\star_{t+1}$ is a minimizer of $V_{\beta}(\cdot,\pi_{t+1})$ in its first argument. Applying Lemma~\ref{lem:value_strong_convexity_context} with $\lambda=\beta$, $\mu=\pi_{t+1}$, $\pi=\pi_{t+1}$, and $\pi'=\mu^\star_{t+1}$, we obtain
    \[
        V_{\beta}(\pi_{t+1},\pi_{t+1}) - V_{\beta}(\mu^\star_{t+1},\pi_{t+1})
        \ge \beta \KL_{\rho}(\pi_{t+1}\Vert\mu^\star_{t+1}).
    \]
    The left-hand side is precisely $\subopt_\beta(\pi_{t+1})$, hence
    \[
        \beta \KL_{\rho}(\pi_{t+1}\Vert\mu^\star_{t+1})
        \le \subopt_\beta(\pi_{t+1}).
    \]
    Using Pinsker's inequality once more,
    \[
        \|\pi_{t+1} - \mu^\star_{t+1}\|_{1,\rho}
        \le \sqrt{2\KL_{\rho}(\pi_{t+1}\Vert\mu^\star_{t+1})}
        \le \sqrt{\frac{2}{\beta} \subopt_\beta(\pi_{t+1})}.
    \]
    Combining the two bounds, we get
    \[
        \frac{1}{2} \|\pi_{t+1} - \pistar_{\beta}\|_{1,\rho} \cdot \|\pi_{t+1} - \mu^\star_{t+1}\|_{1,\rho}
        \le \sqrt{1/\beta}\,A_{t+1}\sqrt{\subopt_\beta(\pi_{t+1})}.
    \]
    Denoting $B_{t+1} = \sqrt{\subopt_\beta(\pi_{t+1})}$, we have
    \[
        B_{t+1}^2 \leq A_{t+1}^2 + (\beta / \eta) A_t^2 +  \sqrt{2}\|\zeta_{t+1}\|_{\spann,\rho} A_{t+1} + 2\sqrt{1/(4\beta)} A_{t+1}B_{t+1}\,.
    \]
    Solving this quadratic, we have 
    \[
        B_{t+1} \leq \sqrt{1/(4\beta)} \cdot A_{t+1} + \sqrt{ (1 + 1/(4\beta)) \cdot A_{t+1}^2 + (\beta / \eta) A_{t}^2 + \sqrt{2 \cdot \norm{\zeta_{t+1}}_{\spann,\rho}^2} \cdot A_{t+1}}\,,
    \]
    and, after taking square and using an inequality $(a+b)^2 \leq 2a^2 + 2b^2$ and $2ab \leq a^2 + b^2$
    \[
        B_{t+1}^2 \leq (2+1/\beta) A_{t+1}^2 + 2(\beta / \eta) A_{t}^2 + 2 \sqrt{2 \norm{\zeta_{t+1}}_{\spann,\rho}^2} \cdot A_{t+1} \leq (1/\beta + 3)A_{t+1}^2 + 2(\beta / \eta) A_t^2 + 2 \norm{\zeta_{t+1}}_{\spann,\rho}^2\,.
    \]
    After plugging in the bound on $A_{t+1}^2$, $A_t^2$, and $\norm{\zeta_{t+1}}_{\spann,\rho}^2$ we conclude the statement.

    \paragraph{Convergence in log-probabilities.}
    Finally, we establish convergence in the span seminorm of log-probabilities. Define $C_t \triangleq \|\log\pi_{t} - \log\pistar_\beta\|_{\spann,\rho}$. A first-order optimality analysis of the regularized best response in $V_t$ shows that, for each context $x$, the best response to $\pi_{t+1}$ (denoted $\nu_{t+1}$) satisfies
    \[
        \beta (1+1/\eta) \log \nu_{t+1}(y|x) = \beta \log \piref(y|x) + (\beta / \eta) \log \pi_t(y|x) - \cP(\pi_{t+1}(x) \succ y \mid x) + c_{t+1}(x)\,,
    \]
    where $c_{t+1}(x)$ is a normalizing constant. Similarly, the regularized VNW $\pistar_\beta$ satisfies
    \[
        \beta (1+1/\eta) \log \pistar_\beta(y|x) = \beta \log \piref(y|x) + (\beta / \eta) \log \pistar_\beta(y|x)  - \cP(\pistar_\beta(x) \succ y \mid x) + c^\star(x)\,.
    \]
    Subtracting these two expressions and computing $\spann,\rho$-norm yields
    \[
        \beta(1+1/\eta) \|\log \nu_{t+1} - \log \pistar_\beta\|_{\spann,\rho}
        \leq (\beta / \eta) \|\log \pi_{t} - \log \pistar_\beta\|_{\spann,\rho}
        + \|\cP(\pistar_\beta - \pi_{t+1} \succ \cdot)\|_{\spann,\rho}.
    \]
    Using Lemma~\ref{lem:lipschitz_preference_context} and Pinsker's inequality, one can bound
    \[
        \|\cP(\pistar_\beta - \pi_{t+1} \succ \cdot)\|_{\spann,\rho}^2
        \leq \frac{1}{2} \KL_{\rho}(\pistar_\beta \Vert \pi_{t+1}) = \frac12 A_{t+1}^2.
    \]
    Moreover, comparing the gradient expression
    \[
        \zeta_{t+1}(x,y) = \nabla_{\pi} V_{t}(\pi_{t+1}, \pi_{t+1})(x,y) = \cP(\pi_{t+1}(x) \succ y \mid x) + \beta\log \frac{\pi_{t+1}(y|x)}{\piref(y |x)} + \beta/\eta \log \frac{\pi_{t+1}(y|x)}{\pi_t(y|x)} + d_{t+1}(x) \bOne\,,
    \]
    with the best-response equation for $\nu_{t+1}$ shows that 
    \[
        \|\zeta_{t+1}\|_{\spann,\rho} = \beta(1+1/\eta)\|\log \nu_{t+1} - \log \pi_{t+1}\|_{\spann,\rho}\,.
    \]
    Thus, we have
    \[
        \beta(1+1/\eta) \norm{\log \pi_{t+1} - \log \pistar_\beta}_{\spann,\rho} \leq (\beta / \eta) \norm{\log \pi_{t} - \log \pistar_\beta}_{\spann, \rho} + A_{t+1} \cdot \sqrt{1/2} + \norm{\zeta_{t+1}}_{\spann, \rho}\,.
    \]
    or, after rearranging and dividing by $\beta(1+1/\eta)$,
    \[
        C_{t+1} \leq \frac{1}{1+\eta} C_t + \frac{A_{t+1}}{\beta(1+1/\eta) \cdot \sqrt{2}} + \frac{\norm{\zeta_{t+1}}_{\spann, \rho}}{\beta(1+1/\eta)}\,.
    \]
    Taking the square and applying the inequality $(a+b)^2 \leq (1+\alpha)a^2 + (1+1/\alpha) b^2$ twice: first time with $\alpha = \eta$ and the second one with $\alpha=1$ implies
    \begin{align*}
        C_{t+1}^2 &\leq \frac{1}{1+\eta} C_t^2 + \frac{1}{\beta^2 \cdot (1+1/\eta)} \left( \frac{A_{t+1}}{\sqrt{2}} + \norm{\zeta_{t+1}}_{\spann, \rho} \right)^2  \leq \frac{1}{1+\eta} C_t^2 +  \frac{\eta \cdot A_{t+1}^2}{\beta^2(1+\eta)} + \frac{2\eta \cdot \norm{\zeta_{t+1}}_{\spann, \rho}^2}{\beta^2(1+\eta)}\,.
    \end{align*}
    
    Plugging in the bound on $A_{t+1}^2$ and $\norm{\zeta_{t+1}}_{\spann, \rho}^2$ implies
    \[
        C_{t+1}^2 \leq \frac{1}{1+\eta} C_t^2 + \frac{\eta \cdot \KL_{\rho}(\pistar_\beta \Vert \pi_0)}{\beta^2(1+\eta)} (1+\eta/2)^{-(t+1)} + \frac{2\eta (1 + \beta^2)\cdot \varepsilon}{\beta^4(1+\eta)} 
    \]
    
    Unrolling this recursion, we achieve for any $t \in \N$
    \begin{align*}
        C_t^2 \leq (1+\eta)^{-t} C_0^2 + \frac{2\KL_{\rho}(\pistar_\beta \Vert \pi_0)}{\beta^2} \cdot (1+\eta/2)^{-t} + \frac{2\eta (1+\beta^2) \cdot \varepsilon}{\beta^4 (1+\eta)} \cdot \frac{1}{1-1/(1+\eta)}\,,
    \end{align*}
    and simplifying the last term we conclude the statement.
\end{proof}

\section{Self-Play Policy Gradients}\label{app:self_play_pg}
In this section, we study a symmetric preference game with KL regularization that appears as a single inner problem of the proximal point method:
\begin{equation}\label{eq:self_play_problem}
    \max_{\pi \in \policies}\min_{\pi' \in \policies} \left\{ \cP^{\tpi}_{\lambda}(\pi \succ \pi') \triangleq \cP(\pi \succ \pi') - \lambda \KL_{\rho}(\pi \Vert \tpi) + \lambda \KL_{\rho}(\pi' \Vert \tpi)\right\}\,.
\end{equation}
We solve this sub-problem using a self-play policy gradient method.

\paragraph{Policy parametrization.}
We consider a general policy parameterization: for a parameter $\theta \in \Theta \subseteq \R^d$ let $\theta \mapsto \pi_\theta \in \policies$ be a corresponding policy. A canonical example is the standard softmax parametrization with $\Theta = \R^{\cX \times \cY}$ and $\pi_\theta(y|x) \propto \rme^{\theta_{x,y}}$. 

\paragraph{Best-response objective.}
For an arbitrary competitor policy $\pi \in \policies$ and a parameter $\theta \in \R^d$, define
\[
    J^{\tpi}(\theta; \pi) \triangleq \cP(\pi \succ \pi_{\theta}) + \lambda \KL_{\rho}(\pi_\theta \Vert \tpi)\,.
\]
We will repeatedly relate $J^{\tpi}$ to the value function $V^{\tpi}_\lambda$ and its best-response value $V^{\tpi,\star}_\lambda$ (see Appendix~\ref{app:analysis_pp}).

The corresponding parametrized best-response operator and value are defined as
\begin{equation}\label{eq:best_response_self_play}
     \theta^\star_\pi \triangleq \argmin_{\theta \in \Theta} J^{\tpi}(\theta; \pi), \qquad J^{\tpi,\star}(\pi) \triangleq \min_{\theta \in \Theta} J^{\tpi}(\theta;\pi)\,.
\end{equation}
Our objective instead is to minimize the exploitability gap (or suboptimality)  $\subopt^{\tpi}_{\lambda}(\pi_\theta) = \max_{\pi' \in \policies}\{ \tfrac{1}{2} - \cP^{\tpi}_{\lambda}(\pi_\theta \succ \pi')\}$. Mirror Prox approximates the solution to \eqref{eq:self_play_problem} by two applications of the best-response operator. However, it requires solving \eqref{eq:best_response_self_play} exactly, which is infeasible beyond the tabular setting.

\paragraph{Best-response approximation via gradient step.}
Rather than solving \eqref{eq:best_response_self_play} exactly, we approximate it using a single gradient step:
\begin{equation}\label{eq:approximate_self_play}
    \theta_{t+1} = \cT(\theta_{t} - \gamma_t \cdot g_t)\,, \qquad 
    \pi_{t+1} = \pi_{\theta_{t+1}}\,,
\end{equation}
where $(\gamma_t)_{t \geq 0}$ is a sequence of step-sizes, $\cT$ is a projection-like operator that guarantees policy improvement, and $g_t$ is a stochastic estimator of $\nabla J^{\tpi}(\theta_t; \pi_t)$ computed from a mini-batch of size $B_t \ge 1$. For an example of $g_t$, we refer to Appendix~\ref{app:gradient_estimator}.

\paragraph{Parametrization assumptions.}

Let $\Theta_{\cT} \triangleq  \{ \cT(\theta) \mid \theta \in \Theta \}$ denote the image of the improvement operator $\cT$, and assume $\theta_0 \in \Theta_{\cT}$. Following \citep{yuan2022general}, we provide a set of assumptions on the parametrization $\theta \mapsto \pi_\theta$.

\begin{assumption}[Lipschitzness of parametrization]\label{ass:lipschitz_param}
    For all $\theta,\theta' \in \Theta$
    \[
        \norm{\pi_\theta - \pi_{\theta'}}_{1,\rho} \leq G \norm{\theta - \theta'}_2
    \]
\end{assumption}

\begin{assumption}[Smoothness]\label{ass:smoothness}
    For any $\pi \in \policies$, the function $J^{\tpi}(\theta;\pi)$ is $L_{\tpi, \lambda}$-smooth, i.e., for all $\theta, \theta' \in \Theta$:
    \[
        - \frac{L_{\tpi, \lambda}}{2} \Vert \theta - \theta'\Vert^2_2 \leq J^{\tpi}(\theta' ; \pi) - J^{\tpi}(\theta; \pi) - \langle \nabla J^{\tpi}(\theta; \pi), \theta' - \theta \rangle \leq \frac{L_{\tpi, \lambda}}{2} \Vert \theta - \theta'\Vert^2_2\,,
    \]
\end{assumption}
We notice that smoothness implies the following useful inequality (see, e.g., \citealt[Theorem~2.1.5]{nesterov2018lectures}) 
\[
    \Vert \nabla J^{\tpi}(\theta; \pi_\theta)  \Vert^2_2 \leq 2 L_{\tpi,\lambda}(J^{\tpi}(\theta; \pi_\theta) - J^{\tpi,\star}(\pi_\theta)) \leq 2 L_{\tpi,\lambda}  \cdot \subopt^{\tpi}_{\lambda}(\pi_\theta)\,.
\]

\begin{assumption}[Approximate Polyak–Łojasiewicz inequality]\label{ass:pl}
    The function $J^{\tpi}(\theta;\pi)$ satisfies the approximate version of Polyak–Łojasiewicz (PL) inequality: there exists a constant $\pl_{\tpi,\lambda} \in (0, L_{\tpi,\lambda}]$ such that for any $\theta \in \Theta_{\cT}$ 
    \[
        \epspl + \Vert \nabla J^{\tpi}(\theta; \pi_\theta)\Vert^2_2 \geq 2\pl_{\tpi,\lambda} \cdot \subopt^{\tpi}_\lambda(\pi_\theta)\,.
    \]
\end{assumption}
In particular, additive error $\epspl$ naturally appears in this bound since, in general, $V^{\tpi,\star}_{\lambda}(\pi_\theta) \not = J^{\tpi,\star}(\pi_\theta)$ and, thus, we can guarantee only
\[
  \subopt^{\tpi}_\lambda(\pi_\theta) = V^{\tpi}_{\lambda}(\pi_\theta;\pi_\theta) - V^{\tpi,\star}_{\lambda}(\pi_\theta) \geq J^{\tpi}(\theta;\pi_\theta) - J^{\tpi,\star}(\pi_\theta)\,,
\]
since the result of minimization of value over $\policies_\Theta \triangleq \{\pi_\theta\}_{\theta \in \Theta}$ and the full policy class $\policies$ might be different.

\begin{assumption}[Improvement operator]\label{ass:improvement}
    For any $\theta \in \Theta$, the operator $\cT$ does not increase exploitability:
    \[
        \subopt^{\tpi}_{\lambda}(\pi_{\cT(\theta)}) \leq \subopt^{\tpi}_{\lambda}(\pi_{\theta})\,.
    \]
\end{assumption}
This assumption is always satisfied by $\cT = \operatorname{Id}$; however, for some particular cases, it is necessary to consider non-trivial $\cT$ to satisfy other assumptions (primarily Assumption~\ref{ass:pl}). We notice that $\cT$ is \emph{not} a projection in the usual sense: since our objective $\theta \mapsto \subopt^{\tpi}_{\lambda}(\pi_{\theta})$ is non-convex, it means that a usual (Euclidean) projection may be detrimental for the current optimization progress.

\begin{assumption}[Mini-batch gradient noise]\label{ass:gradient_noise}
    Let $(\cF_t)_{t \geq 0}$ be the filtration induced by the algorithm iterates: 
    $\cF_t \triangleq \sigma(\{ \theta_j, \pi_j\}_{j \leq t})$. 
    For each $t$, let $B_t \ge 1$ denote the mini-batch size used to construct $g_t$, and define the gradient bias and noise
    \[
        \xi_t \triangleq g_t - \E[g_t | \cF_{t}]\,, \qquad b_t \triangleq \E[g_t | \cF_{t}] -  \nabla J^{\tpi}(\theta_t; \pi_{t})\,.
    \]
    Then, for any $t \in \N$, conditionally on $\cF_t$, the following holds:
    \begin{itemize}
        \item[(i)] Bounded bias:
        \[
            \norm{b_t}_2 \leq \varepsilon_{\grad}\,.
        \]
        \item[(ii)] Subgaussian tails with variance proxy scaling as $1/B_t$: there exists a constant $\sigma^2_{\tpi,\lambda} > 0$ such that for all $u \in \R^d$ and $s \in \R$,
        \[
            \log \E\big[\exp\big(s \cdot \langle u, \xi_t \rangle \big)\big|\cF_t\big] 
            \leq \frac{s^2 \sigma^2_{\tpi,\lambda} \cdot \norm{u}_2^2}{2 B_t}\,.
        \]
        \item[(iii)] Subexponential tails for the squared norm with variance proxy scaling as $1/B_t$: 
        there exists a constant $v^2_{\tpi,\lambda} > 0$ such that for all 
        $s \in [0, B_t / v^2_{\tpi,\lambda}]$,
        \[
            \log \E\big[\exp\big(s \norm{\xi_t}_2^2\big) \big|\, \cF_t\big] 
            \leq s \cdot \frac{v^2_{\tpi,\lambda}}{B_t}\,.
        \]
    \end{itemize}
\end{assumption}

\begin{remark}
Assumption~\ref{ass:gradient_noise} captures the standard $1/B_t$ variance reduction from averaging a mini-batch of $B_t$ independent samples.  For example, suppose $g_t = \frac{1}{B_t} \sum_{i=1}^{B_t} \hat g(\theta_t, \pi_{\theta_t}; Z_{t,i})$, where conditionally on $\cF_t$ the random variables  $\{Z_{t,i}\}_{i=1}^{B_t}$ are independent and the single-sample noises  $\hat g(\theta_t, \pi_{\theta_t}; Z_{t,i}) - \E_{Z_{t,i}}[\hat g(\theta_t, \pi_{\theta_t}; Z_{t,i})]$ are subgaussian with variance proxy $\bar\sigma^2$ (possibly depending on the dimension).  Then the averaged noise $\xi_t$ is subgaussian with variance proxy $\bar\sigma^2/B_t$, and $\norm{\xi_t}^2$ is also subexponential of the form $\bar v^2 / B_t$ for some $\bar{v}^2 > 0$ (see, e.g., \citealt{jin2019short}).  This covers, in particular, the cases of bounded per-sample gradients and purely Gaussian noise.
\end{remark}

In Appendix~\ref{app:softmax_parametrization}, we explicitly verify these assumptions of a simple softmax parametrization, and in Appendix~\ref{app:fisher_parametrization}, we verify them for a general family of policies under compatible Fisher-nondegenerate parameterization.

\subsection{Convergence Guarantees}\label{app:convergence_sppg}

In the following, we analyze two cases separately: the deterministic and the stochastic. Before that, we prove the following important result, which we split into two parts for the sake of better result exposition.

\begin{lemma}[Descent Lemma I]\label{lem:descent_lemma_I}

Assume Assumptions~\ref{ass:lipschitz_param}-\ref{ass:smoothness}-\ref{ass:improvement}.
For the iterates of policy-gradient self-play~\eqref{eq:approximate_self_play},
for any sequence of learning rates $(\gamma_t)_{t \geq 0}$ and for any $t \geq 0$, it holds that
\[
  \sqrt{\subopt^{\tpi}_{\lambda}(\pi_{t+1})}
  \le
  \sqrt{\frac{\gamma_t^2 G^2}{8\lambda} \norm{g_t}^2_{2}}
  + \sqrt{
      \subopt^{\tpi}_{\lambda}(\pi_t)
      -\gamma_t  \big\langle \nabla J^{\tpi}(\theta_{t}; \pi_t), g_t \big\rangle
      +  \Bigl(L_{\tpi,\lambda}
              +\frac{G^2}{4\lambda}\Bigr)
         \frac{\gamma_t^2}{2} \Vert  g_t \Vert^2_{2}
    }\,.
\]
\end{lemma}
\begin{remark}[On the additional term.]
Compared to standard descent-lemma arguments for a fixed smooth objective, Lemma~\ref{lem:descent_lemma_I} contains an extra term. This term reflects that the best-response objective is \emph{dynamic}: after each update, the competitor policy $\pi_t$ changes, so the objective function $\theta \mapsto J^{\tpi}(\theta;\pi_t)$ shifts across iterations.
\end{remark}

\begin{proof}
We decompose one iteration of~\eqref{eq:approximate_self_play} as
\[
  \theta_t^+ \triangleq \theta_t - \gamma_t g_t,
  \qquad
  \theta_{t+1} = \cT(\theta_t^+), \qquad
  \pi_t^+ \triangleq \pi_{\theta_t^+}\,.
\]

\textbf{Step 1: Smoothness in $\theta$.} By Assumption~\ref{ass:smoothness}, for any $t \in \N$,
\[
  J^{\tpi}(\theta_{t}^+; \pi_t)
  \leq
  J^{\tpi}(\theta_{t}; \pi_t)
  + \big\langle \nabla J^{\tpi}(\theta_{t}; \pi_t),
               \theta_{t}^+ - \theta_{t} \big\rangle
  + \frac{L_{\tpi,\lambda}}{2} \Vert \theta_{t}^+ - \theta_{t} \Vert^2_{2}\,.
\]
Using $\theta_t^+ - \theta_t = -\gamma_t g_t$, we obtain
\begin{equation}\label{eq:descent_lemma_eq1}
  J^{\tpi}(\theta_{t}^+; \pi_t)
  \leq
  J^{\tpi}(\theta_{t}; \pi_t)
  - \gamma_t \big\langle \nabla J^{\tpi}(\theta_{t}; \pi_t), g_t \big\rangle
  + \frac{L_{\tpi,\lambda} \gamma^2_t}{2} \Vert  g_t \Vert^2_{2}\,.
\end{equation}
Subtracting $V^{\tpi,\star}_{\lambda}(\pi_t)$ from both sides and using Lemma~\ref{lem:exploitability_via_value_context}, which gives
\[
  J^{\tpi}(\theta_t;\pi_t) - V^{\tpi,\star}_{\lambda}(\pi_t) = V^{\tpi}_{\lambda}(\pi_t; \pi_t) - V^{\tpi,\star}_{\lambda}(\pi_t) = \subopt^{\tpi}_{\lambda}(\pi_t)\,,
\]
we obtain
\begin{equation}\label{eq:descent_lemma_eq1b}
  V^{\tpi}_{\lambda}(\pi_t^+; \pi_t) - V^{\tpi,\star}_{\lambda}(\pi_t)
  \leq
  \subopt^{\tpi}_{\lambda}(\pi_t)
  - \gamma_t \big\langle \nabla J^{\tpi}(\theta_{t}; \pi_t), g_t \big\rangle
  + \frac{L_{\tpi,\lambda} \gamma^2_t}{2} \Vert  g_t \Vert^2_{2}\,.
\end{equation}

\textbf{Step 2: Relating $J^{\tpi}$ to $\subopt^{\tpi}_{\lambda}(\pi_t^+)$.} Using the identity $\subopt^{\tpi}_\lambda(\pi_\theta) = V^{\tpi}_{\lambda}(\pi_\theta; \pi_\theta) - V^{\tpi,\star}_{\lambda}(\pi_\theta)$, we have
\[
  V^{\tpi}_{\lambda}(\pi_t^+; \pi_t) - V^{\tpi,\star}_{\lambda}(\pi_t)
  =
  \underbrace{V^{\tpi}_{\lambda}(\pi_t^+; \pi_t^+) - V^{\tpi,\star}_{\lambda}(\pi_t^+)}_{
    \subopt^{\tpi}_\lambda(\pi_t^+)
  }
  + \bigl(V^{\tpi}_{\lambda}(\pi_t^+; \pi_t) - V^{\tpi}_{\lambda}(\pi_t^+; \pi_t^+)\bigr)
  + \bigl(V^{\tpi,\star}_{\lambda}(\pi_t^+) - V^{\tpi,\star}_{\lambda}(\pi_t)\bigr)\,.
\]
Substituting this identity into~\eqref{eq:descent_lemma_eq1b} yields
\begin{equation}\label{eq:descent_lemma_eq2}
  \begin{split}
    \subopt^{\tpi}_{\lambda}(\pi_{t}^+)
    \leq\;
    &\subopt^{\tpi}_{\lambda}(\pi_t)
    - \gamma_t  \big\langle \nabla J^{\tpi}(\theta_{t}; \pi_t), g_t \big\rangle
    +  \frac{L_{\tpi,\lambda} \gamma_t^2}{2} \Vert  g_t \Vert^2_{2}  \\
    &\quad
    + \bigl(V^{\tpi}_{\lambda}(\pi_t^+; \pi_t^+) - V^{\tpi}_{\lambda}(\pi_t^+; \pi_t)\bigr)
    + \bigl(V^{\tpi,\star}_{\lambda}(\pi_t) - V^{\tpi,\star}_{\lambda}(\pi_t^+)\bigr)\,.
  \end{split}
\end{equation}

By definition of $V^{\tpi}_{\lambda}$,
\[
  V^{\tpi}_{\lambda}(\pi_t^+; \pi_t^+) - V^{\tpi}_{\lambda}(\pi_t^+; \pi_t)
  = \cP(\pi_{t}^+ - \pi_t \succ \pi_{t}^+)\,.
\]
Moreover, the lower bound of Lemma~\ref{lem:br_value_smoothness_context} yields
\[
  V^{\tpi,\star}_{\lambda}(\pi_{t}^+) - V^{\tpi,\star}_{\lambda}(\pi_t)
  \geq
  \cP(\pi_{t}^+ - \pi_t \succ \nu_{t}^+)\,,
\]
where $\nu_{t}^+$ is a best response to $\pi_{t}^+$. Thus
\[
  V^{\tpi,\star}_{\lambda}(\pi_t) - V^{\tpi,\star}_{\lambda}(\pi_{t}^+)
  \leq
  -\cP(\pi_{t}^+ - \pi_t \succ \nu_{t}^+)\,.
\]
Substituting these into~\eqref{eq:descent_lemma_eq2} and using bilinearity
of $\cP$, we obtain
\begin{equation}\label{eq:descent_lemma_eq3}
    \subopt^{\tpi}_{\lambda}(\pi_{t}^+)
    \leq \subopt^{\tpi}_{\lambda}(\pi_t)
    - \gamma_t  \big\langle \nabla J^{\tpi}(\theta_{t}; \pi_t), g_t \big\rangle
    +  \frac{L_{\tpi,\lambda} \gamma_t^2}{2} \Vert  g_t \Vert^2_{2} 
    + \cP(\pi_{t}^+ - \pi_t \succ \pi_{t}^+ - \nu_{t}^{+})\,.
\end{equation}

\textbf{Step 3: Bounding the preference term.} Lemma~\ref{lem:lipschitz_preference_context} implies
\[
  \cP(\pi_{t}^+ - \pi_t \succ \pi_{t}^+ - \nu_{t}^{+})
  \leq \frac{1}{2}
  \norm{\pi_{t}^+ - \pi_t}_{1,\rho} \cdot
  \norm{\nu_{t}^{+} - \pi_{t}^+}_{1,\rho}\,.
\]
Assumption~\ref{ass:lipschitz_param} gives $\norm{\pi_{t}^+ - \pi_{t}}_{1,\rho} \leq G \norm{\theta_{t}^+-\theta_t}_{2}
 = G\gamma_t \norm{g_t}_{2}$.
By Pinsker's inequality and Lemma~\ref{lem:value_strong_convexity_context},
\[
  \norm{\nu_{t}^{+} - \pi_{t}^+}_{1,\rho}^2
  \leq 2\KL_{\rho}(\pi_{t}^+ \Vert \nu_{t}^{+} )
  \leq \frac{2}{\lambda} \cdot \subopt^{\tpi}_{\lambda}(\pi_{t}^+)\,,\quad \text{ so }\quad \norm{\nu_{t}^{+} - \pi_{t}^+}_{1,\rho}
  \le \sqrt{\frac{2}{\lambda}
             \,\subopt^{\tpi}_{\lambda}(\pi_{t}^+)}\,.
\]
Consequently,
\[
  \cP(\pi_{t}^+ - \pi_t \succ \pi_{t}^+ - \nu_{t}^{+})
  \leq \frac{G\gamma_t}{2} \norm{g_t}_{2} \cdot
      \sqrt{\frac{2}{\lambda} \subopt^{\tpi}_{\lambda}(\pi_{t}^+)}\,.
\]
As a result,~\eqref{eq:descent_lemma_eq3} becomes
\begin{equation}\label{eq:descent_lemma_eq4}
  \begin{split}
    \subopt^{\tpi}_{\lambda}(\pi_{t}^+)
    \leq\;
    &\subopt^{\tpi}_{\lambda}(\pi_t)
    -\gamma_t  \big\langle \nabla J^{\tpi}(\theta_{t}; \pi_t), g_t \big\rangle
    +  \frac{\gamma_t^2 \cdot L_{\tpi,\lambda}}{2} \Vert  g_t \Vert^2_{2} \\
    &\quad
    + \frac{G\gamma_t}{2} \norm{g_t}_{2} \cdot
      \sqrt{\frac{2}{\lambda} \subopt^{\tpi}_{\lambda}(\pi_{t}^+)}\,.
  \end{split}
\end{equation}

Let $S \triangleq \sqrt{\subopt^{\tpi}_{\lambda}(\pi_{t}^+)}$. Then
\eqref{eq:descent_lemma_eq4} can be written as
\[
  S^2
  - 2 S \sqrt{\frac{G^2\gamma_t^2}{8\lambda} \norm{g_t}_{2}^2}
  \le
  \subopt^{\tpi}_{\lambda}(\pi_t)
  -\gamma_t  \big\langle \nabla J^{\tpi}(\theta_{t}; \pi_t), g_t \big\rangle
  +  \frac{\gamma_t^2 \cdot L_{\tpi,\lambda}}{2} \Vert  g_t \Vert^2_{2}\,.
\]
Adding $\frac{G^2\gamma_t^2}{8\lambda} \norm{g_t}^2_{2}$ to both sides
and completing the square yields
\[
  \biggl( S - \sqrt{\frac{G^2\gamma_t^2}{8\lambda} \norm{g_t}^2_{2}}
  \biggr)^2
  \le
  \subopt^{\tpi}_{\lambda}(\pi_t)
  -\gamma_t  \big\langle \nabla J^{\tpi}(\theta_{t}; \pi_t), g_t \big\rangle
  +  \Bigl(L_{\tpi,\lambda}
          +\frac{G^2}{4\lambda}\Bigr)
     \frac{\gamma_t^2}{2} \Vert  g_t \Vert^2_{2}\,.
\]
Taking square roots, we get
\begin{align*}
  \sqrt{\subopt^{\tpi}_{\lambda}(\pi_t^+)}
  &\le
  \sqrt{\frac{\gamma_t^2 G^2}{8\lambda} \norm{g_t}^2_{2}}\\
  &\quad+ \sqrt{
      \subopt^{\tpi}_{\lambda}(\pi_t)
      -\gamma_t  \big\langle \nabla J^{\tpi}(\theta_{t}; \pi_t), g_t \big\rangle
      +  \Bigl(L_{\tpi,\lambda}
              +\frac{G^2}{4\lambda}\Bigr)
         \frac{\gamma_t^2}{2} \Vert  g_t \Vert^2_{2}
    }\,.
\end{align*}

By Assumption~\ref{ass:improvement}, $\subopt^{\tpi}_{\lambda}(\pi_{t+1}) \le \subopt^{\tpi}_{\lambda}(\pi_t^+)$, so
\begin{align*}
  \sqrt{\subopt^{\tpi}_{\lambda}(\pi_{t+1})}
  &\le
  \sqrt{\frac{\gamma_t^2 G^2}{8\lambda} \norm{g_t}^2_{2}}\\
  &\quad + \sqrt{
      \subopt^{\tpi}_{\lambda}(\pi_t)
      -\gamma_t  \big\langle \nabla J^{\tpi}(\theta_{t}; \pi_t), g_t \big\rangle
      +  \Bigl(L_{\tpi,\lambda}
              +\frac{G^2}{4 \lambda}\Bigr)
         \frac{\gamma_t^2}{2} \Vert  g_t \Vert^2_{2}
    }\,.
\end{align*}
\end{proof}

\begin{lemma}[Descent Lemma II]\label{lem:descent_lemma_II}
    
Assume Assumptions~\ref{ass:lipschitz_param}-\ref{ass:smoothness}-\ref{ass:pl}-\ref{ass:improvement} and  $ \lambda \pl_{\tpi,\lambda} \geq G^2$. For the iterates of policy-gradient self-play~\eqref{eq:approximate_self_play},
for any sequence of learning rates $(\gamma_t)_{t \geq 0}$ satisfying $\gamma_t \leq 1/(2 L_{\tpi,\lambda}),$ and for any $t \geq 0$, it holds that
\begin{align*}
  \subopt^{\tpi}_{\lambda}(\pi_{t+1})
  &\leq
  \Bigl(1-\frac{\gamma_t \pl_{\tpi,\lambda}}{2}\Bigr)
  \subopt^{\tpi}_{\lambda}(\pi_t)
  - \alpha_t \,\big\langle \nabla J^{\tpi}(\theta_t; \pi_t), \xi_t \big\rangle \\
  &\qquad+ \frac{3 L_{\tpi,\lambda}}{2} \cdot \gamma_t^2 (\Vert  \xi_t \Vert^2_{2} + \Vert b_t \Vert_2^2) + \frac{3 G^2}{4\lambda \, \pl_{\tpi,\lambda}} \cdot \gamma_t(\Vert  \xi_t \Vert^2_{2}
  + \Vert b_t \Vert_2^2) \\
  &\qquad+ \frac{3}{4} \cdot \gamma_t (\epspl + \norm{b_t}_2^2)\,,
\end{align*}
where $\xi_t \triangleq g_t - \E[g_t | \cF_{t}]$ and $b_t \triangleq \E[g_t | \cF_{t}] -  \nabla J^{\tpi}(\theta_t; \pi_t)$ are the gradient noise and bias respectively, $ \pl_{\tpi,\lambda}$ is defined in Assumption~\ref{ass:pl}, and
\[
  \alpha_t \triangleq \frac{1-\gamma_t \pl_{\tpi,\lambda}/2}{1-\gamma_t \pl_{\tpi,\lambda}} \cdot
    \gamma_t \cdot \Bigl(
      1 - \gamma_t L_{\tpi,\lambda}
        - \frac{G^2}{2\lambda \pl_{\tpi,\lambda}}
    \Bigr)\,.
\]
\end{lemma}
\begin{remark}
    In this Lemma, we have introduced an additional condition on $\lambda$ to be a large enough constant. Although hypothetically it might be possible that $\pl_{\tpi,\lambda}$ decreasing with a value of $\lambda$, in any known examples $\pl_{\tpi,\lambda}$ non-decreases with an increase of $\lambda$, making an inequality $\lambda \pl_{\tpi,\lambda} \geq 0$ solvable. We refer to Proposition~\ref{prop:softmax_verification} and Proposition~\ref{prop:fisher_verification} for examples.
\end{remark}

\begin{proof}

Starting from Lemma~\ref{lem:descent_lemma_I}, one can take squares and apply an inequality
$(a+b)^2 \le (1+1/A_t)a^2 + (1+A_t)b^2$ for any $A_t>0$:
\begin{align}
  &\subopt^{\tpi}_{\lambda}(\pi_{t+1})
  \le
  \frac{1+A_t}{A_t}\cdot
  \frac{\gamma_t^2 G^2}{8\lambda} \norm{g_t}^2_{2} \label{eq:descent_lemma_pre_noise}
  \\
  &\qquad + (1+A_t)\biggl(
      \subopt^{\tpi}_{\lambda}(\pi_t)
      -\gamma_t  \big\langle \nabla J^{\tpi}(\theta_{t}; \pi_t), g_t \big\rangle
      +  \Bigl(L_{\tpi,\lambda}
              +\frac{G^2}{4\lambda}\Bigr)
         \frac{\gamma_t^2}{2} \Vert  g_t \Vert^2_{2}
    \biggr)\notag \\
  &=
  (1+A_t)\biggl(
    \subopt^{\tpi}_{\lambda}(\pi_t)
    -\gamma_t  \big\langle \nabla J^{\tpi}(\theta_{t}; \pi_t), g_t \big\rangle 
    + \Bigl(L_{\tpi,\lambda}
            +\frac{G^2}{4\lambda}\left( 1+\frac{1}{A_t}\right)\Bigr)
      \frac{\gamma_t^2}{2} \Vert  g_t \Vert^2_{2}
  \biggr)\,. \notag
\end{align}

To simplify expressions, we also upper bound $1+1/A_t$ by $2+1/A_t$ in the last term. The reasons for this loose bound will be more evident lately.

Define the gradient noise $\xi_t \triangleq g_t - \E[g_t | \cF_{t}]$ and bias $b_t \triangleq \E[g_t | \cF_{t}] -  \nabla J^{\tpi}(\theta_t; \pi_t)$.
Then
\[
  \norm{g_t}_{2}^2 = \norm{\nabla J^{\tpi}(\theta_t; \pi_t) + b_t + \xi_t}_{2}^2 \leq \norm{\nabla J^{\tpi}(\theta_t; \pi_t)}^2_{2}
    + 2\norm{\xi_t}_{2}^2 + 2 \norm{b_t}_2^2
    + 2 \big\langle \nabla J^{\tpi}(\theta_t; \pi_t),\xi_t + b_t \big\rangle\,.
\]
Substituting this into~\eqref{eq:descent_lemma_pre_noise} yields
\begin{align*}
  \subopt^{\tpi}_{\lambda}(\pi_{t+1})
  &\leq
  (1+A_t)\Bigl(
    \subopt^{\tpi}_{\lambda}(\pi_t)
    - \gamma_t \bigl(
        1 - \tfrac{\gamma_t}{2}(
              L_{\tpi,\lambda}
              +\tfrac{G^2}{4\lambda} \cdot \tfrac{2A_t+1}{A_t})
      \bigr)
      \big\|\nabla J^{\tpi}(\theta_t; \pi_t)\big\|_2^2 \\
  &\hspace{4em}
    - \gamma_t \bigl(
        1 - \gamma_t(
              L_{\tpi,\lambda}
              +\tfrac{G^2}{4\lambda}  \cdot \tfrac{2A_t+1}{A_t}
      \bigr)
      \big\langle \nabla J^{\tpi}(\theta_t; \pi_t), b_t \big\rangle \\
    &\hspace{4em}
    - \gamma_t \bigl(
        1 - \gamma_t(
              L_{\tpi,\lambda}
              +\tfrac{G^2}{4\lambda}
              \cdot \tfrac{2A_t+1}{A_t}
      \bigr)
      \big\langle \nabla J^{\tpi}(\theta_t; \pi_t), \xi_t \big\rangle \\
  &\hspace{4em}
    + \Bigl(
        L_{\tpi,\lambda}
        +\tfrac{G^2}{4\lambda} \cdot \tfrac{2A_t+1}{A_t}\Bigr)
      \gamma_t^2 (\Vert  \xi_t \Vert^2_{2} + \norm{b_t}_2^2)
  \Bigr)\,.
\end{align*}
Now, we set
\[
  A_t
  = \frac{\gamma_t \pl_{\tpi,\lambda}}
         {2 (1 - \gamma_t \pl_{\tpi,\lambda})}\,,
\]
where $\pl_{\tpi,\lambda}$ is defined in Assumption~\ref{ass:pl}. Note that $A_t$ is positive since $\gamma_t \le 1/(2L_{\tpi,\lambda}) \leq 1/(2\pl_{\tpi,\lambda}) $ thanks to Assumption~\ref{ass:smoothness}, Assumption~\ref{ass:pl}.
One checks that
\[
  1+A_t
  = \frac{1-\gamma_t \pl_{\tpi,\lambda}/2}{1-\gamma_t \pl_{\tpi,\lambda}},
  \qquad
  \frac{1 + 2A_t}{4\lambda A_t}
  = \frac{1-\gamma_t \pl_{\tpi,\lambda}}{2\lambda \gamma_t \pl_{\tpi,\lambda}} + \frac{1}{2\lambda} = \frac{1}{2\lambda \gamma_t \pl_{\tpi,\lambda}}\,.
\]
Simplifying the constants, we obtain
\begin{align*}
  \begin{split}
  \subopt^{\tpi}_{\lambda}(\pi_{t+1})
  &\leq
  \frac{1-\gamma_t \pl_{\tpi,\lambda}/2}{1-\gamma_t \pl_{\tpi,\lambda}}
  \Bigl(
    \subopt^{\tpi}_{\lambda}(\pi_t)
    \\
    &\qquad\qquad\qquad\qquad- \gamma_t \Bigl(
        1 - \frac{\gamma_t L_{\tpi,\lambda}}{2}
        - \frac{G^2}{4\lambda\pl_{\tpi,\lambda}}
      \Bigr)
      \big\|\nabla J^{\tpi}(\theta_t; \pi_t)\big\|_2^2  \\
    &\qquad\qquad\qquad\qquad- \gamma_t \Bigl(
      1 - \gamma_t  L_{\tpi,\lambda} 
        - \frac{G^2}{2\lambda \pl_{\tpi,\lambda}}
    \Bigr) \big\langle \nabla J^{\tpi}(\theta_t; \pi_t), b_t \big\rangle
  \Bigr)
  \end{split} & \triangleq \termA
  \\
  &\quad
  - \frac{1-\gamma_t \pl_{\tpi,\lambda}/2}{1-\gamma_t \pl_{\tpi,\lambda}}
    \gamma_t \Bigl(
      1 - \gamma_tL_{\tpi,\lambda} 
        - \frac{G^2}{2\lambda \pl_{\tpi,\lambda}}
    \Bigr)
    \big\langle \nabla J^{\tpi}(\theta_t; \pi_t), \xi_t \big\rangle &  \triangleq \termB \\
  &\quad
  + \frac{1-\gamma_t \pl_{\tpi,\lambda}/2}{1-\gamma_t \pl_{\tpi,\lambda}}
    \Bigl(
        L_{\tpi,\lambda} 
      + \frac{G^2}{2\lambda \gamma_t \pl_{\tpi,\lambda}}
    \Bigr)
    \frac{\gamma_t^2}{2}\Vert  \xi_t + b_t \Vert^2_{2}\,. & \triangleq \termC
\end{align*}

Next, we verify that under our assumptions, the first (deterministic) term contracts. First, we verify that $1 - \gamma_t \cdot  L_{\tpi,\lambda} - G^2/(2\lambda\pl_{\tpi,\lambda}) \geq 0$. Indeed, since $\gamma_t \leq 1/(2L_{\tpi,\lambda})$ and $\lambda \pl_{\tpi,\lambda} \geq G^2$, we have
\[
    1 - \gamma_t \cdot L_{\tpi,\lambda} - \frac{G^2}{2\lambda\pl_{\tpi,\lambda}} \geq 1 - \frac{1}{2} - \frac{1}{2} \geq 0\,.
\]
Thus, the coefficient in front of $\big\langle \nabla J^{\tpi}(\theta_t; \pi_t), b_t \big\rangle$ is non-negative. By Cauchy--Schwarz and an inequality $ab \leq a^2/2 + b^2/2$, we have
\begin{align*}
  - \gamma_t \Bigl(1 - \gamma_t \cdot L_{\tpi,\lambda} - \tfrac{G^2}{2\lambda\pl_{\tpi,\lambda}} \Bigr) \big\langle \nabla J^{\tpi}(\theta_t; \pi_t), b_t \big\rangle
  &\leq
  \gamma_t \Bigl(\tfrac{1}{2} - \tfrac{\gamma_t \cdot L_{\tpi,\lambda}}{2} - \tfrac{G^2}{4\lambda\pl_{\tpi,\lambda}} \Bigr)  \big\|\nabla J^{\tpi}(\theta_t; \pi_t)\big\|_2^2
  \\
  &\qquad\qquad\qquad+ \frac{\gamma_t}{2} \|b_t\|_2^2\,,
\end{align*}
and substituting this into $\termA$ yields
\[  
  \termA \leq \frac{1-\gamma_t \pl_{\tpi,\lambda}/2}{1-\gamma_t \pl_{\tpi,\lambda}} \Bigl( \subopt^{\tpi}_{\lambda}(\pi_t)
    - \frac{\gamma_t}{2}
      \big\|\nabla J^{\tpi}(\theta_t; \pi_t)\big\|_2^2 + \frac{\gamma_t}{2} \|b_t\|_2^2
    \Bigr)\,.
\]
Thus, the coefficient in front of $\big\|\nabla J^{\tpi}(\theta_t; \pi_t)\big\|_2^2$ is positive. By Assumption~\ref{ass:pl},
\[
  \big\|\nabla J^{\tpi}(\theta_t;\pi_t)\big\|_2^2 \geq 2\pl_{\tpi,\lambda}\, \subopt^{\tpi}_\lambda(\pi_t) - \epspl\,,
\]
therefore $\termA$ can be bounded as
\[
  \termA \leq (1- \gamma_t \pl_{\tpi,\lambda}/2)\subopt^{\tpi}_{\lambda}(\pi_t) + \frac{1-\gamma_t \pl_{\tpi,\lambda}/2}{1-\gamma_t \pl_{\tpi,\lambda}} \cdot \left( \frac{\gamma_t}{2} \epspl +  \frac{\gamma_t}{2} \|b_t\|_2^2 \right)\,.
\]
To bound the last term, we use the fact that since $L_{\tpi,\lambda} \geq \pl_{\tpi,\lambda}$, we have $\gamma_t \pl_{\tpi,\lambda} \leq 1/2$ and thus $(1-\gamma_t \pl_{\tpi,\lambda}/2) \leq 3/2 \cdot(1-\gamma_t \pl_{\tpi,\lambda})$.

For $\termB$, since the sign of $\langle \nabla J^{\tpi}(\theta_t; \pi_t), \xi_t\rangle$ can be arbitrary, we cannot provide a clear upper bound on this term and we only define a coefficient in front of it as $\alpha_t$:
\[
  \termB = -\alpha_t \cdot \langle \nabla J^{\tpi}(\theta_t; \pi_t), \xi_t \rangle
\]
For $\termC$, we also use the bound  $(1-\gamma_t \pl_{\tpi,\lambda}/2) \leq 3/2 \cdot(1-\gamma_t \pl_{\tpi,\lambda})$, thus
\[
\termC \leq \frac{3 L_{\tpi,\lambda}}{2} \cdot \gamma_t^2 (\norm{\xi_t}_2^2 + \norm{b_t}_2^2) + \frac{3G^2}{4 \lambda \pl_{\tpi,\lambda}} \cdot \gamma_t  (\norm{\xi_t}_2^2 + \norm{b_t}_2^2) \,.
\]
Combining all the bounds, we conclude the statement.
\end{proof}

\begin{remark}[On the necessity of the $\gamma_t\|\xi_t\|_2^2$ term]\label{rem:necessity_gamma_xi_sq}
Lemma~\ref{lem:descent_lemma_II} contains a noise contribution of order
$\gamma_t\|\xi_t\|_2^2$, whereas in the usual PL analysis of SGD on a \emph{fixed} smooth objective
one typically pays only $\cO(\gamma_t^2\|\xi_t\|_2^2)$.
This linear-in-$\gamma_t$ variance penalty is intrinsic to our setting.

The key difference is that our progress measure
\[
  \subopt^{\tpi}_\lambda(\pi_t)
  \;=\;
  V^{\tpi}_{\lambda}(\pi_t;\pi_t) - V^{\tpi,\star}_{\lambda}(\pi_t)
\]
uses the best-response value $V^{\tpi,\star}_{\lambda}(\pi_t)$ as a \emph{moving baseline}.
After each update, the opponent changes from $\pi_t$ to $\pi_{t+1}$, so $V^{\tpi,\star}_{\lambda}(\pi_t)$ drifts.
Controlling this drift is exactly what produces the extra square-root term in Lemma~\ref{lem:descent_lemma_I}; see \eqref{eq:descent_lemma_eq4}, which contains (up to constants)
\[
  \gamma_t \,\|g_t\|_2\,\sqrt{\subopt^{\tpi}_\lambda(\pi_t^+)}.
\]

With stochastic gradients $g_t=\nabla J^{\tpi}(\theta_t;\pi_t)+\xi_t$, this term injects noise at the level $s_t\triangleq \sqrt{\subopt^{\tpi}_\lambda(\pi_t)}$ as an additive perturbation $\Theta(\gamma_t\|\xi_t\|_2)$.
Squaring to return to $\subopt_t=s_t^2$ produces a cross term of size $\Theta(\gamma_t\,s_t\|\xi_t\|_2)$.
To close a recursion in $\subopt_t$, this cross term must be absorbed into the deterministic PL decrease, whose scale is only $\Theta(\gamma_t\pl_{\tpi,\lambda})\,s_t^2$.
Applying Young's inequality therefore leaves an unavoidable remainder of order $\Theta(\gamma_t/\pl_{\tpi,\lambda})\|\xi_t\|_2^2$ (with constants also involving $1/\lambda$ through the KL-based control of the baseline drift), which explains the $\gamma_t\|\xi_t\|_2^2$ term in Lemma~\ref{lem:descent_lemma_II}.

In particular, this term vanishes when $\xi_t\equiv 0$, but in the stochastic case it becomes the dominant noise contribution unless $\E\|\xi_t\|_2^2$ decays (e.g., via growing mini-batches as in Assumption~\ref{ass:gradient_noise} or via variance reduction).
\end{remark}

\begin{proposition}[Deterministic rates]\label{prop:spg_selfplay_deterministic}
Assume Assumptions~\ref{ass:lipschitz_param}-\ref{ass:smoothness}-\ref{ass:pl}-\ref{ass:improvement} and  $\lambda \pl_{\tpi,\lambda} \geq G^2$. For the iterates of policy-gradient self-play~\eqref{eq:approximate_self_play} with exact gradients (i.e., $g_t = \nabla J^{\tpi}(\theta_t; \pi_t)$), let a sequence $(\gamma_t)_{t \geq 0}$ be a constant such that $\gamma_t \equiv \gamma \leq 1/(2 L_{\tpi,\lambda})$
and for any $t \geq 0$, it holds that
\begin{align*}
  \subopt^{\tpi}_{\lambda}(\pi_{t})
  &\leq
  \Bigl(1- \frac{\gamma \pl_{\tpi,\lambda}}{2}\Bigr)^t
  \subopt^{\tpi}_{\lambda}(\pi_0) + \frac{3 \cdot \epspl}{2 \pl_{\tpi,\lambda}}\,,
\end{align*}
where $\pl_{\tpi,\lambda}$ is defined in Assumption~\ref{ass:pl}.
\end{proposition}
\begin{proof}
  Directly follows from Lemma~\ref{lem:descent_lemma_II} since in this case $\xi_t \equiv b_t \equiv 0$, combined with the fact that $L_{\tpi,\lambda} \geq \pl_{\tpi, \lambda}$ and, consequently, $\gamma \pl_{\tpi,\lambda}/2 < 1$ due to the choice of a step-size.
\end{proof}

\begin{proposition}[Stochastic rates]\label{prop:spg_selfplay_stochastic}
    Assume Assumptions~\ref{ass:lipschitz_param}-\ref{ass:smoothness}-\ref{ass:pl}-\ref{ass:improvement}, Assumption~\ref{ass:gradient_noise} and $ \lambda \pl_{\tpi,\lambda} \geq G^2$. Define $\kappa_{\tpi,\lambda} = \frac{L_{\tpi,\lambda}}{\pl_{\tpi,\lambda}} \geq 1$.  Define sequences $(\gamma_t)_{t\geq 0}$ and $(B_t)_{t \geq 0}$ as follows
    \[   
        \gamma_t = \frac{4t + 32\kappa_{\tpi,\lambda} - 2}{\pl_{\tpi,\lambda}(t + 8\kappa_{\tpi,\lambda})^2}  = \Theta\left( \frac{1}{\pl_{\tpi,\lambda} t} \right)\,, \qquad B_t =  \left\lceil \frac{t + 8 \kappa_{\tpi,\lambda}}{\pl_{\tpi,\lambda}} \right\rceil = \Theta\left( \frac{t}{\pl_{\tpi,\lambda}} \right)\,.
    \]
    
    Let $\delta \in (0,1)$. Then, for the iterates of policy-gradient self-play~\eqref{eq:approximate_self_play}, for any $t \geq 0$, with probability at least $1-\delta$ it holds that
    \begin{align*}
        \subopt^{\tpi}_\lambda(\pi_t) &\leq \frac{64 \cdot \kappa^2_{\tpi,\lambda} \log(\rme/\delta)}{(t + 8\kappa_{\tpi,\lambda} - 1)^2}\subopt^{\tpi}_\lambda(\pi_0) \\
        &\qquad + \frac{ 24 \cdot  \kappa_{\tpi,\lambda} \log(\rme/\delta) \cdot \log(1+t/(2\kappa_{\tpi,\lambda}))}{(t + 8 \kappa_{\tpi,\lambda} - 1)^2}\left( \frac{675 \cdot \sigma^2_{\tpi,\lambda}}{49} + 2 v^2_{\tpi,\lambda} \right) \\
        &\qquad + \frac{6G^2\cdot v^2_{\tpi,\lambda} \log(\rme/\delta)}{\lambda  \pl_{\tpi,\lambda} \cdot (t + 8 \kappa_{\tpi,\lambda} - 1)} + \frac{6\log(\rme/\delta)}{\pl_{\tpi,\lambda}} \cdot \left(\epspl +  15/7 \cdot \varepsilon_{\grad}^2\right) \,.
    \end{align*}
\end{proposition}
\begin{remark}[On the growing mini-batch condition]
Compared with a standard high-probability analysis of stochastic gradient descent under a PL condition \citep{madden2024high}, our bound requires a growing mini-batch size $B_t \asymp t$. This is not merely a technical artifact: it compensates for a new source of noise compared to the classical PL-SGD setting, as discussed in Remark~\ref{rem:necessity_gamma_xi_sq}. In particular, the linear-in-$\gamma_t$ noise term in Lemma~\ref{lem:descent_lemma_II} would lead to a noise floor of the iterates if a constant mini-batch size were used even with decaying learning rates.

As remarked in \citep[Appendix~D]{madden2024high}, a similar issue appears in the analysis of projected non-convex SGD, and can be addressed either via a comparable growing mini-batch condition \citep{ghadimi2016mini} or via variance-reduction techniques \citep{reddi2016proximal}, which we do not employ in our setting.
\end{remark}

\begin{proof}
    First, we notice that a proposed sequence of learning rates satisfies the assumption of Lemma~\ref{lem:descent_lemma_II}. Indeed,
    \[
        \gamma_t = \frac{4 t + 32\kappa_{\tpi,\lambda} - 2}{\pl_{\tpi,\lambda}(t + 8\kappa_{\tpi,\lambda})^2} \leq \gamma_0 = \frac{32\kappa_{\tpi,\lambda}-2}{\pl_{\tpi,\lambda} \cdot 64 \cdot \kappa_{\tpi,\lambda}^2} \leq \frac{1}{2\pl_{\tpi,\lambda} \cdot \kappa_{\tpi,\lambda}} = \frac{1}{2L_{\tpi,\lambda}}\,.
    \]
    Thus, Lemma~\ref{lem:descent_lemma_II} implies
    \begin{equation}\label{eq:stoch_rates:apply_descent_lemma}
      \begin{split}
      \subopt^{\tpi}_{\lambda}(\pi_{t+1})
      &\leq
      \Bigl(1-\frac{\gamma_t \pl_{\tpi,\lambda}}{2}\Bigr)
      \subopt^{\tpi}_{\lambda}(\pi_t)
      - \alpha_t \,\big\langle \nabla J^{\tpi}(\theta_t; \pi_t), \xi_t \big\rangle \\
      &\qquad + \frac{3}{2}L_{\tpi,\lambda} \cdot \gamma_t^2 (\Vert  \xi_t \Vert^2_{2} + \Vert b_t \Vert_2^2) + \frac{3G^2}{4\lambda \, \pl_{\tpi,\lambda}} \cdot \gamma_t (\Vert  \xi_t \Vert^2_{2} + \Vert b_t \Vert_2^2) \\
      &\qquad + \frac{3}{4}\gamma_t (\epspl + \norm{b_t}_2^2)\,.
      \end{split}
    \end{equation}
    We notice that under our choice of $(\gamma_t)_{t \geq0}$ it holds
    \[
         1-\frac{\gamma_t \pl_{\tpi,\lambda}}{2} = \frac{(t + 8\kappa_{\tpi,\lambda})^2 - 2 (t + 8\kappa_{\tpi,\lambda}) + 1}{(t + 8\kappa_{\tpi,\lambda})^2} = \frac{(t-1 + 8\kappa_{\tpi,\lambda})^2}{(t+8\kappa_{\tpi,\lambda})^2}.
    \]
    Thus, we can multiply both sides of \eqref{eq:stoch_rates:apply_descent_lemma} by $(t+8\kappa_{\tpi,\lambda})^2$ and denote $X_t = (t + 8\kappa_{\tpi,\lambda} - 1)^2 \cdot \subopt^{\tpi}_\lambda(\pi_t)$
    \begin{align*}
        X_{t+1} &\leq X_t - \alpha_t (t+8\kappa_{\tpi,\lambda})^2 \big\langle \nabla J^{\tpi}(\theta_t; \pi_t), \xi_t \big\rangle \\
        &\quad + (t+8\kappa_{\tpi,\lambda})^2 \left( \frac{3}{2}L_{\tpi,\lambda} \gamma_t^2 + \frac{3\gamma_t G^2}{4 \lambda \pl_{\tpi,\lambda}} \right)(\Vert  \xi_t \Vert^2_{2} + \Vert b_t \Vert_2^2) + \frac{3 \gamma_t}{4}(\epspl + \norm{b_t}_2^2)\,.
    \end{align*}
    Next, we denote $Y_t = -\alpha_t (t+8\kappa_{\tpi,\lambda})^2 \big\langle \nabla J^{\tpi}(\theta_t; \pi_t), \xi_t \big\rangle$ and for this bound we have
    \begin{align*}
        |-\alpha_t (t+8\kappa_{\tpi,\lambda})^2| &= \frac{1-\gamma_t \pl_{\tpi,\lambda}/2}{1-\gamma_t \pl_{\tpi,\lambda}} \cdot
        \gamma_t \cdot \left(
    1 - \gamma_t L_{\tpi,\lambda}
      - \frac{G^2}{2\lambda \pl_{\tpi,\lambda}}
  \right) \cdot (t +8\kappa_{\tpi,\lambda})^2 \\
  &\leq \frac{3}{2} \gamma_t (t + 8\kappa_{\tpi,\lambda})^2 \leq 12 \cdot\frac{t + 8 \kappa_{\tpi,\lambda} - 1/2}{\pl_{\tpi,\lambda}} \leq \frac{90 \cdot (t + 8\kappa_{\tpi,\lambda} -1)}{7\pl_{\tpi,\lambda}}\,,
    \end{align*}
    thus, Assumption~\ref{ass:gradient_noise} and Assumption~\ref{ass:pl} imply for any $s \in \R$
    \begin{align*}
        \log\E[\exp(s Y_t) | \cF_{t}] &\leq \frac{s^2}{2} \cdot (\alpha_t (t+8\kappa_{\tpi,\lambda})^2)^2 \frac{\sigma^2_{\tpi,\lambda}}{B_t} \cdot \norm{\nabla J^{\tpi}(\theta_t; \pi_t)}_2^2 \\
        &\leq \frac{s^2}{2} \cdot \underbrace{\frac{2 \cdot 8100 \cdot L_{\tpi,\lambda}}{49 \cdot  \pl_{\tpi,\lambda} \cdot (t + 8\kappa_{\tpi,\lambda})} \cdot \sigma^2_{\tpi,\lambda}}_{\tilde{B}^2_t} \underbrace{(t + 8 \kappa_{\tpi,\lambda} - 1)^2 \subopt_{\lambda}^{\tpi}(\pi_t)}_{X_t}\,,
    \end{align*}
    where for the last inequality we applied Assumption~\ref{ass:smoothness}. Finally, we define 
    \begin{align*}
      Z_t &\triangleq (t+8\kappa_{\tpi,\lambda})^2\left(\frac{3}{2}L_{\tpi,\lambda} \gamma_t^2 + \frac{3\gamma_t G^2}{4\lambda \, \pl_{\tpi,\lambda}}\right)\Vert  \xi_t \Vert^2_{2} &\triangleq Z_{t,1} \\
      &\qquad + (t+8\kappa_{\tpi,\lambda})^2\left(\frac{3}{2}L_{\tpi,\lambda} \gamma_t^2 + \frac{3\gamma_t G^2}{4\lambda \, \pl_{\tpi,\lambda}}\right)\Vert b_t \Vert^2_{2}  &\triangleq Z_{t,2} \\
      &\qquad + \frac{3}{4}  (t+8\kappa_{\tpi,\lambda})^2\gamma_t \cdot (\epspl + \norm{b_t}_2^2)  &\triangleq Z_{t,3}
    \end{align*}
    To evaluate this tail behavior, we first bound the first term
    \[
        Z_{t,1} = (t+8\kappa_{\tpi,\lambda})^2\left(\frac{3}{2}L_{\tpi,\lambda} \gamma_t^2 + \frac{3\gamma_t G^2}{4\lambda \pl_{\tpi,\lambda}}\right) \Vert  \xi_t \Vert^2_{2} \leq \left(\frac{24 L_{\tpi,\lambda}}{(\pl_{\tpi,\lambda})^2} + \frac{ 3 (t + 8\kappa_{\tpi,\lambda})G^2}{\lambda (\pl_{\tpi,\lambda})^2}\right) \Vert  \xi_t \Vert^2_{2}\,,
    \]
    and the similar bound of $Z_{t,2}$ holds, where we replace $\Vert \xi_t \Vert^2_{2}$ by $\norm{b_t}_2^2$, however, we perform one additional simplification that comes from a bound $G^2 \leq \lambda \pl_{\tpi,\lambda}$, thus
    \[
        Z_{t,2} \leq \left(\frac{24 \kappa_{\tpi,\lambda}}{\pl_{\tpi,\lambda}} + \frac{ 3 (t + 8\kappa_{\tpi,\lambda})}{\pl_{\tpi,\lambda}}\right) \norm{b_t}_2^2\,.
    \]

    For the last term, we have
    \[
      Z_{t,3} = \frac{3}{4}  (t+8\kappa_{\tpi,\lambda})^2\gamma_t (\epspl + \norm{b_t}_2^2) \leq \frac{3(t + 8 \kappa_{\tpi,\lambda})}{\pl_{\tpi,\lambda}} (\epspl + \norm{b_t}_2^2)\,.
    \]
    thus we can apply Assumption~\ref{ass:gradient_noise} and achieve
    \begin{align*}
        &\log\, \E[\exp(s Z_t) | \cF_{t}] \leq s \cdot \left( \frac{24 L_{\tpi,\lambda}}{(\pl_{\tpi,\lambda})^2} + \frac{ 3 (t + 8\kappa_{\tpi,\lambda})G^2}{\lambda (\pl_{\tpi,\lambda})^2} \right)  \frac{v^2_{\tpi,\lambda}}{B_t} \\
        &\qquad + s \cdot \left(\frac{24 \kappa_{\tpi,\lambda}}{\pl_{\tpi,\lambda}} + \frac{ 3 (t + 8\kappa_{\tpi,\lambda})}{\pl_{\tpi,\lambda}}\right) \varepsilon_{\grad}^2 + s \cdot \frac{3(t + 8 \kappa_{\tpi,\lambda})}{\pl_{\tpi,\lambda}} (\epspl + \varepsilon_{\grad}^2) \\
        &\leq s \biggl[ \left( \frac{24 \kappa_{\tpi,\lambda}}{t + 8\kappa_{\tpi,\lambda}} + \frac{3 G^2}{\lambda \pl_{\tpi,\lambda}} \right) v^2_{\tpi,\lambda} + \frac{24 \kappa_{\tpi,\lambda} \cdot \varepsilon_{\grad}^2}{\pl_{\tpi,\lambda}} +  \frac{3 \cdot (t + 8 \kappa_{\tpi,\lambda})}{\pl_{\tpi,\lambda}} (\epspl + 2 \varepsilon_{\grad}^2) \biggr] \triangleq s \cdot \tilde{C}_t
    \end{align*}
    for a range of $s \in [0, 1/\tilde{C}_t)$.
    Thus, \citet[Theorem 9]{madden2024high} implies that for our system of inequalities $X_{t+1} \leq X_t + Y_t + Z_t$, any $t \in \N$ with probability at least $1-\delta$ it holds
    \[
        X_t \leq K_t \log(\rme/\delta)\,,
    \]
    where $K_t$ is a sequence that satisfies the following equations: $K_{t+1}^2 \geq (K_t + 2 \tilde{C}_t) K_{t+1} + \tilde{B}_t^2 K_t$ and $K_0 \geq X_0$, which is satisfied for a sequence $K_{t+1} = K_t + 2 \tilde{C}_t + \tilde{B}_t^2$ and $K_0 = 64 \kappa^2_{\tpi,\lambda} \subopt^{\tpi}_{\lambda}(\pi_0) \geq X_0$. In particular, we have
    \begin{align*}
        K_t &= 64 \kappa_{\tpi,\lambda}^2 \subopt^{\tpi}_{\lambda}(\pi_0) + \frac{16200 \cdot \kappa_{\tpi,\lambda}}{49} \cdot \sigma^2_{\tpi,\lambda}  \cdot \sum_{s=0}^{t-1} \frac{1}{s + 8\kappa_{\tpi,\lambda}} + 48 \kappa_{\tpi,\lambda} \cdot v^2_{\tpi,\lambda} \cdot \sum_{s=0}^{t-1} \frac{1}{s + 8\kappa_{\tpi,\lambda}} \\
        &\quad + 6 t \cdot  \left( \frac{v^2_{\tpi,\lambda} \cdot G^2}{\lambda \cdot \pl_{\tpi,\lambda}} + \frac{4 \kappa_{\tpi,\lambda} \cdot \varepsilon_{\grad}^2}{\pl_{\tpi,\lambda}} \right)+ \frac{6}{\pl_{\tpi,\lambda}}\left( \epspl + 2\varepsilon_{\grad}^2 \right) \cdot \sum_{s=0}^{t-1} (s + 8\kappa_{\tpi,\lambda}) \\
        &\leq 64 \kappa_{\tpi,\lambda}^2 \subopt^{\tpi}_{\lambda}(\pi_0) + 24 \kappa_{\tpi,\lambda}\log(1 + t/(2\kappa_{\tpi,\lambda})) \cdot \left(\frac{675 \sigma^2_{\tpi,\lambda}}{49} + 2v^2_{\tpi,\lambda} \right) \\
        &\qquad + 6 t \cdot  \left( \frac{v^2_{\tpi,\lambda} \cdot G^2}{\lambda \cdot \pl_{\tpi,\lambda}} + \frac{4 \kappa_{\tpi,\lambda} \cdot \varepsilon_{\grad}^2}{\pl_{\tpi,\lambda}} \right) +  \frac{6}{\pl_{\tpi,\lambda}}\left( \epspl + 2 \varepsilon_{\grad}^2 \right) \cdot \frac{t(t + 16 \kappa_{\tpi,\lambda} - 1)}{2}\,.
    \end{align*}
    Before proceeding further, we notice that since $\kappa_{\tpi,\lambda} \geq 1$, we can use the following bounds for any $t \geq 0$:
    \[
        \frac{t}{(t + 8 \kappa_{\tpi,\lambda} - 1)^2} \leq \frac{1}{t + 8 \kappa_{\tpi,\lambda} - 1}\,,\qquad \frac{4 t \kappa_{\tpi,\lambda}}{(t + 8 \kappa_{\tpi,\lambda} - 1)^2} \leq \frac{1}{7}\,, \qquad \frac{t(t + 16 \kappa_{\tpi,\lambda} - 1)}{2 (t + 8 \kappa_{\tpi,\lambda} - 1)^2} \leq 1\,.
    \]
    
    Finally, we divide both bounds on $X_t$ and $K_t$ by $(t+8\kappa_{\tpi,\lambda}-1)^2$ and thus we have the following bound for exploitability that holds for any $t \in \N$ with probability at least $1-\delta$
    \begin{align*}
        \subopt^{\tpi}_\lambda(\pi_t) &\leq \frac{64 \cdot \kappa^2_{\tpi,\lambda} \log(\rme/\delta)}{(t + 8\kappa_{\tpi,\lambda}-1)^2}\subopt^{\tpi}_\lambda(\pi_0) \\
        &\qquad + \frac{ 24 \cdot  \kappa_{\tpi,\lambda} \log(\rme/\delta) \cdot \log(1+t/(2\kappa_{\tpi,\lambda}))}{(t + 8 \kappa_{\tpi,\lambda}-1)^2}\left( \frac{675 \cdot \sigma^2_{\tpi,\lambda}}{49} + 2v^2_{\tpi,\lambda} \right) \\
        &\qquad + \frac{6G^2\cdot v^2_{\tpi,\lambda} \log(\rme/\delta)}{\lambda  \pl_{\tpi,\lambda} \cdot (t + 8 \kappa_{\tpi,\lambda}-1)} + \frac{6 (\epspl + 15/7 \cdot \varepsilon_{\grad}^2) \log(\rme/\delta)}{\pl_{\tpi,\lambda}}\,.
    \end{align*}
\end{proof}

\subsection{Gradient Estimator}\label{app:gradient_estimator}

As a particular instance, we employ the clipped gradient estimator that allows to control the variance of the estimator via assumptions on the parameterization only.

Given $\theta \in \Theta$, a context $x\in \cX$, two actions $y,y' \in \cY$, and an unbiased estimate $p \in [0,1]$ of $\cP(y \succ y' | x)$, define the advantage estimator
\[
    A^{\pi_\theta}(x,y,y',p) = 1/2 - p +\lambda \log \frac{\pi_\theta(y|x)}{\tpi(y|x)} - \lambda \log \frac{\pi_\theta(y'|x)}{\tpi(y'|x)}\,.
\]
Then, for a threshold parameter $M > 0$, we define the clipped advantage estimator
\[    
  \tilde{A}^{\pi_\theta}_M(x,y,y',p) = \clip_{-M,M}\left( A^{\pi_\theta}(x,y,y',p) \right) =
  \begin{cases}
    -M & \text{if } A^{\pi_\theta}(x,y,y',p) < -M\,, \\
    A^{\pi_\theta}(x,y,y',p) & \text{if } A^{\pi_\theta}(x,y,y',p) \in [-M,M]\,, \\
    M & \text{if } A^{\pi_\theta}(x,y,y',p) > M\,.
  \end{cases}
\]
Then we define two versions of the stochastic gradient estimator:
\begin{align}
    G(\theta |x,y,y',p) &\triangleq \frac{1}{2} \cdot (\nabla \log \pi_\theta(y|x) - \nabla \log \pi_\theta(y'|x)) A^{\pi_\theta}(x,y,y',p)\, \label{eq:unclipped_single_gradient_estimator_self_play_pg}\,,\\
    G_M(\theta | x,y,y',p) &\triangleq \frac{1}{2} \cdot (\nabla \log \pi_\theta(y|x) - \nabla \log \pi_\theta(y'|x)) \tilde{A}^{\pi_\theta}_M(x,y,y',p) \label{eq:clipped_single_gradient_estimator_self_play_pg}\,.
\end{align}
Then we have the following properties of this estimator.
\begin{lemma}\label{lem:noise_properties_general_initial}
   Assume $\norm{\nabla \log \pi_\theta(y|x)}_2^2 \leq M_g^2$ for any $x \in \supp(\rho), y \in \cY, \theta \in \Theta_{\cT}$. Then, the stochastic gradient estimator defined in \eqref{eq:unclipped_single_gradient_estimator_self_play_pg} satisfies the following properties:

   \begin{itemize}[leftmargin=1em]
    \item A gradient estimator $G(\theta |x,y,y',p)$ is unbiased, i.e., for any $\theta$ it holds $\E[G(\theta|x,y,y',p)] = \nabla J^{\tpi}(\theta; \pi_\theta)$.
    \item The clipped gradient estimator $G_M(\theta | x,y,y',p)$ is biased, and its bias satisfies
    \begin{align*}
        \norm{\E[G_M(\theta | x,y,y',p)] - \nabla J^{\tpi}(\theta; \pi_\theta)}_2 &\leq 2M_g\lambda(\E_{x\sim \rho}\left[ \left( \log\frac{1}{\tpi_{\min}(x)} - \frac{M-1/2}{2\lambda} \right)_+ \right] \\
        &+ 2M_g\lambda\rme^{-(M-1/2) / (2\lambda)}\,.
    \end{align*}
    \item For any $\theta$, it holds almost surely that
    \[
        \norm{G_M(\theta | x,y,y',p) - \E[G_M(\theta | x,y,y',p)]}_2 \leq 2 M_g \cdot M\,.
    \]
   \end{itemize}
\end{lemma}
\begin{proof}
  For the first part of the statement, we notice that by linearity of expectation, zero expectation of the score function $\E[\nabla \log \pi_\theta(y|x)] = 0$, and symmetry of preferences, we have
\begin{align*}
    \E[G(\theta | x,y,y',p)]
    &= \E\left[ \nabla \log \pi_\theta(y|x) \left(\cP(\pi \succ y \mid x) + \lambda \log \frac{\pi_\theta(y|x)}{\tpi(y|x)} \right) \right]\,,
\end{align*}
    where $x \sim \rho, y,y' \sim \pi_\theta(\cdot|x)$, and $\pi = \pi_\theta$. The last expression is a standard REINFORCE estimator \citep{williams1992simple} for a regularized RL problem $J^{\tpi}(\theta; \pi_\theta)$, and it is known that the last expression is equal to $\nabla J^{\tpi}(\theta; \pi_\theta)$.

    For the second part of the statement, we use the unbiasedness of $G(\theta | x,y,y',p)$ and the definition of the clipped estimator to write
    \begin{align*}
        \|\E[G_M&(\theta | x,y,y',p)] - \nabla J^{\tpi}(\theta; \pi_\theta)\|_2 = \|\E[G_M(\theta | x,y,y',p)] - \E[G(\theta | x,y,y',p)]\|_2 \\
        & = \left\|\E\left[ \frac{1}{2} \cdot (\nabla \log \pi_\theta(y|x) - \nabla \log \pi_\theta(y'|x)) \left( \tilde{A}^{\pi_\theta}_M(x,y,y',p) - A^{\pi_\theta}(x,y,y',p) \right) \right] \right\|_2 \\
        & \leq \frac{1}{2}\E\left[ \left\| \nabla \log \pi_\theta(y|x) - \nabla \log \pi_\theta(y'|x) \right\|_2 \cdot | \tilde{A}^{\pi_\theta}_M(x,y,y',p) - A^{\pi_\theta}(x,y,y',p) | \right] \\
        & \leq M_g \cdot \E\left[ | \tilde{A}^{\pi_\theta}_M(x,y,y',p) - A^{\pi_\theta}(x,y,y',p) | \right]\,.
    \end{align*}
    Next, we notice that by definition of the clipping operator it holds
    \[
        | \tilde{A}^{\pi_\theta}_M(x,y,y',p) - A^{\pi_\theta}(x,y,y',p) | = ( |A^{\pi_\theta}(x,y,y',p)| - M)_+\,,
    \]
    where $(x)_+ = \max\{0,x\}$. Next, define $\ell_\theta(x,y) = \log\frac{\pi_\theta(y|x)}{\tpi(y|x)}$ as the log-ratio random variable. Then, we can rewrite the advantage as
    \[
        A^{\pi_\theta}(x,y,y',p) = 1/2 - p + \lambda (\ell_\theta(x,y) - \ell_\theta(x,y'))\,.
    \]
    Assuming $M \geq 1$, we can use the triangle inequality to write
    \begin{align*}
        \E\left[ ( |A^{\pi_\theta}(x,y,y',p)| - M)_+ \right] &= \E\left[ ( |1/2 - p + \lambda (\ell_\theta(x,y) - \ell_\theta(x,y'))| - M)_+ \right] \\
        &\leq \E\left[ ( |1/2 - p| + \lambda |\ell_\theta(x,y)| + \lambda|\ell_\theta(x,y')| - M)_+ \right] \\
        &\leq \E\left[ ( 1/2 + \lambda |\ell_\theta(x,y)| +\lambda |\ell_\theta(x,y')| - M)_+ \right] \\
        &\leq 2\lambda \cdot \E\left[ \left( |\ell_\theta(x,y)| - \frac{M - 1/2}{2\lambda}\right)_+ \right] \,,
    \end{align*}
    where in the last inequality we used the fact that $(a + b - c)_+ \leq (a - c/2)_+ + (b - c/2)_+$ for any $a,b,c \geq 0$, and the fact that $y$ and $y'$ are i.i.d. samples from $\pi_\theta(\cdot|x)$.
    Next, defining $M' = (M - 1/2)/(2\lambda)$, we can further bound
    \[
      \E\left[ \left( |\ell_\theta(x,y)| - M' \right)_+ \right] \leq \E\left[ \left( \ell_\theta(x,y) - M' \right)_+ \right] +  \E\left[ \left( -\ell_\theta(x,y) - M' \right)_+ \right]\,.
    \]
    For the first term, we apply the following inequality; $\ell_\theta(x,y) = \log\frac{\pi_\theta(y|x)}{\tpi(y|x)} \leq \log \frac{1}{\tpi(y|x)} \leq \log \frac{1}{\tpi_{\min}(x)}$ for any $x \in \cX, y \in \cY$ since $\pi_\theta(y|x) \leq 1$. Thus, we have
    \[
      \E\left[ \left( \ell_\theta(x,y) - M' \right)_+ \right] \leq \E_{x \sim \rho}\left[ \left( \log\frac{1}{\tpi_{\min}(x)} - M' \right)_+\right]\,.
    \]
    Next, we study the second term. Let us fix a context $x \in \cX$ and define the set $\cY_{x}^{-} = \{ y \in \cY: \pi_\theta(y|x) / \tpi(y|x) \leq \rme^{-M'} \}$. Then, we have
    \begin{align*}
      \E\left[ \left( -\ell_\theta(x,y) - M' \right)_+ \right] &= \E_{x \sim \rho}\left[ \sum_{y \in \cY_x^-} \pi_\theta(y|x) \left( \log \frac{\tpi(y|x)}{\pi_\theta(y|x)} - M' \right) \right] \\
      &= \E_{x \sim \rho}\left[ \sum_{y \in \cY_x^-} \pi_\theta(y|x)  \log \frac{\tpi(y|x)}{\rme^{M'} \cdot\pi_\theta(y|x)} \right]\,,
    \end{align*}
    and, applying an inequality $\log(x) \leq x-1$ for any $x \geq 0$, we achieve
    \[
      \E\left[ \left( -\ell_\theta(x,y) - M' \right)_+ \right] \leq \E_{x \sim \rho}\left[ \sum_{y \in \cY_x^-} \pi_\theta(y|x) \left( \frac{\tpi(y|x)}{\rme^{M'} \cdot\pi_\theta(y|x)} - 1 \right) \right] \leq \rme^{-M'}\,.
    \]
    Finally, the last statement of the lemma follows directly from the definition of the clipped estimator.
\end{proof}

Before proceeding further, we notice that the bias of the clipped estimator can be made arbitrarily small by increasing the clipping threshold $M$; however, its specific range depends on the behavior of the reference policy $\tpi$. In particular, if $\tpi$ has a small minimum probability $\tpi_{\min}(x)$ for some contexts $x$, then the bias can be large for moderate values of $M$. Thus, we need to introduce additional assumptions on the reference policy to control this bias effectively.

\begin{lemma}\label{lem:noise_properties_bias}
  Assume $\norm{\nabla \log \pi_\theta(y|x)}_2^2 \leq M_g^2$ for any $x \in \supp(\rho), y \in \cY, \theta \in \Theta_{\cT}$. Then, for any desired bias level $\varepsilon_{\grad} > 0$, the clipped gradient estimator defined in \eqref{eq:clipped_single_gradient_estimator_self_play_pg} satisfies the following properties.
  \begin{itemize}[leftmargin=1em]
    \item If the reference policy satisfies $\E_{x\sim \rho}[\log^2(1/\tpi_{\min}(x))] \leq V_{\tpi}$ for some constant $V_{\tpi} > 0$, then under the choice $M = M_2(\varepsilon_{\grad}) \triangleq 1/2 + 4\lambda^2 M_g (V_{\tpi} + 1) /\varepsilon_{\grad}$ the bias is bounded as $\varepsilon_{\grad}$ and it holds
    \[
        \norm{G_{M}(\theta | x,y,y',p) - \E[G_{M}(\theta | x,y,y',p)]}_2 \leq D_2(\varepsilon_{\grad})\,,
    \]
    where $D_2(\varepsilon_{\grad}) \triangleq  M_g \cdot \left( 1 + \frac{8\lambda^2 M_g (V_{\tpi} + 1)}{\varepsilon_{\grad}} \right)$.
    \item If the reference policy satisfies $\E_{x \sim \rho}[1/\tpi_{\min}(x)] \leq D_{\tpi}$ for some constant $D_{\tpi} > 0$, then under the choice $M = M_{\infty}(\varepsilon_{\grad}) \triangleq 1/2 + 2\lambda \log\left( \frac{2 M_g \lambda (D_{\tpi} + 1)}{\varepsilon_{\grad}} \right)$ the bias is bounded as $\varepsilon_{\grad}$ and it holds
    \[
        \norm{G_{M}(\theta | x,y,y',p) - \E[G_{M}(\theta | x,y,y',p)]}_2 \leq D_{\infty}(\varepsilon_{\grad})\,,
    \]
    where $D_{\infty}(\varepsilon_{\grad}) \triangleq M_g \cdot \left( 1 + 4\lambda \log\left( \frac{2 M_g \lambda (D_{\tpi} + 1)}{\varepsilon_{\grad}} \right) \right)$.
  \end{itemize}
\end{lemma}
\begin{proof}
  We start from the bias bound provided in Lemma~\ref{lem:noise_properties_general_initial}:
  \begin{align*}
        \norm{\E[G_M(\theta | x,y,y',p)] - \nabla J^{\tpi}(\theta; \pi_\theta)}_2 &\leq 2M_g\lambda \E_{x\sim \rho}\left[ \left( \log\frac{1}{\tpi_{\min}(x)} - \frac{M-1/2}{2\lambda} \right)_+ \right] \\
        &\qquad + 2M_g\lambda  \rme^{-(M-1/2) / (2\lambda)}\,.
  \end{align*}
  Let us define $M' = (M - 1/2)/(2\lambda)$ and then study the first term inside the parentheses and represent it as
  \[
      \E_{x\sim \rho}\left[ \left( \log\frac{1}{\tpi_{\min}(x)} - M' \right)_+ \right] \leq \int_{M'}^\infty \P\left( \log\frac{1}{\tpi_{\min}(x)} \geq u \right) \rmd u\,.
  \]
  For the first part of the statement, we apply Markov's inequality to write
  \[
      \P\left( \log\frac{1}{\tpi_{\min}(x)} \geq u \right) \leq \frac{\E_{x\sim \rho}[\log^2(1/\tpi_{\min}(x))]}{u^2} \leq \frac{V_{\tpi}}{u^2}\,,
  \]
  and thus
  \[
      \E_{x\sim \rho}\left[ \left( \log\frac{1}{\tpi_{\min}(x)} - M' \right)_+ \right] \leq \int_{M'}^\infty \frac{V_{\tpi}}{u^2} \rmd u = \frac{V_{\tpi}}{M'}\,.
  \]
  To control the bias at level $\varepsilon_{\grad}$, we set $M'$ such that $2 M_g \lambda (V_{\tpi}/M' + \rme^{-M'}) \leq \varepsilon_{\grad}$. Since $\rme^{-M'} \leq 1/M'$, it is sufficient to set $M' = 2 M_g \lambda (V_{\tpi} + 1)/\varepsilon_{\grad}$, which leads to the choice of $M_2(\varepsilon_{\grad})$ in the statement of the lemma.

  For the second part of the statement, we again start from the integral representation and apply Markov's inequality to write
  \[
      \P\left( \log\frac{1}{\tpi_{\min}(x)} \geq u \right) = \P\left( \frac{1}{\tpi_{\min}(x)} \geq \rme^{u} \right) \leq \frac{\E_{x \sim \rho}[1/\tpi_{\min}(x)]}{\rme^{u}} \leq \frac{D_{\tpi}}{\rme^{u}}\,,
  \]
  and thus
  \[
      \E_{x\sim \rho}\left[ \left( \log\frac{1}{\tpi_{\min}(x)} - M' \right)_+ \right] \leq \int_{M'}^\infty \frac{D_{\tpi}}{\rme^{u}} \rmd u = D_{\tpi} \cdot \rme^{-M'}\,.
  \]
  As a result, we have
  \[
      \norm{\E[G_M(\theta | x,y,y',p)] - \nabla J^{\tpi}(\theta; \pi_\theta)}_2 \leq 2M_g\lambda\left( D_{\tpi} \cdot \rme^{-M'} + \rme^{-M'}\right) = 2 M_g \lambda (D_{\tpi} + 1) \cdot \rme^{-M'}\,.
  \]
  To control the bias at level $\varepsilon_{\grad}$, we set $M'$ such that $2 M_g \lambda (D_{\tpi} + 1) \cdot \rme^{-M'} \leq \varepsilon_{\grad}$, which leads to the choice of $M_{\infty}(\varepsilon_{\grad})$ in the statement of the lemma.

\end{proof}

Next, we provide the properties of the corresponding mini-batch stochastic gradient estimator. At iteration $t$ we draw an i.i.d.\ mini-batch $(x_j,y_j, y'_j,p_j)_{j\in[B_t]}$ with $x_j \sim \rho$, $y_j,y'_j \sim \pi_{\theta_t}(\cdot|x_j)$ and $p_j$ an unbiased estimator of $\cP(y_j \succ y'_j | x_j)$, and set
\begin{equation}\label{eq:stochastic_gradient_estimator_self_play_pg}
    g_t = \frac{1}{B_t} \sum_{j=1}^{B_t} G_{M_k(\varepsilon_{\grad})}(\theta_t | x_j,y_j,y'_j,p_j)\,,
\end{equation}
where $\varepsilon_{\grad}$ is a desired bias level and $k \in \{2,\infty\}$ indicates which choice of the clipping threshold from Lemma~\ref{lem:noise_properties_bias} is used. Let $(\cF_t)_{t\ge0}$ be the filtration generated by the iterates:
$\cF_t = \sigma(\{\theta_k, \pi_k\}_{k\le t})$.

\begin{lemma}\label{lem:noise_properties_general}
    Assume that the conditions of Lemma~\ref{lem:noise_properties_bias} hold true for some desired bias level $\varepsilon_{\grad} > 0$ and $k \in \{2,\infty\}$. Then, the stochastic gradient estimator defined in \eqref{eq:stochastic_gradient_estimator_self_play_pg} satisfies for any $\cF_t$-measurable vector $u \in \R^{d}$
    \begin{align*}
        \forall s \in \R: &\log\E\left[ \exp\left( s \langle u, \xi_t\rangle\right) | \cF_t\right] \leq \frac{s^2 D_{k}^2(\varepsilon_{\grad})}{2B_t} \norm{u}_2^2\,, \\
         \forall s \in [0, B_t/(6 D^2_{k}(\varepsilon_{\grad}))]: &\log\E\left[ \exp(s \norm{\xi_t}_2^2) | \cF_t\right] \leq 6 s \cdot D^2_k(\varepsilon_{\grad})/B_t\,.
    \end{align*}
    In particular, it satisfies Assumption~\ref{ass:gradient_noise} with a bias $\varepsilon_{\grad}$, subgaussian constant $\sigma^2_{\tpi,\lambda} = D_k^2(\varepsilon_{\grad})$ and subexponential constant $v^2_{\tpi,\lambda} = 6D_k^2(\varepsilon_{\grad})$.
\end{lemma}
\begin{proof} 
    We have to compute an exact subgaussianity constant for a random vector $\xi_t$ bounded almost surely by $D_k(\varepsilon_{\grad})$, which is standard (see, e.g., \cite{vershynin2018high}) and we write the proof for the sake of completeness.
    Let us fix a vector $u \in \R^d$ and define $\xi^1_t,\ldots,\xi^{B_t}_t$ as the corresponding components of the noise vector $\xi^j_t = G_{M_k}(\theta_t | x_j,y_j, y'_j,p_j) - \E[G_{M_k}(\theta_t | x_j,y_j, y'_j,p_j)]$. Then, we define $X_i = \langle u, \xi^i_t\rangle$ as a centered bounded random variable with $|X_i| \leq \norm{u}_2 D_k(\varepsilon_{\grad})$, and thus (conditional) Hoeffding's lemma \citep{hoeffding1963probability} implies
    \[
        \E[\rme^{s X_i} | \cF_t] \leq \exp\left(\frac{s^2 \cdot \norm{u}_2^2 D_k^2(\varepsilon_{\grad})}{2}\right)\qquad \forall s \in \R\,,
    \]
    and, using independence, we have $\E[\rme^{s \langle u, \xi_t\rangle} | \cF_t] \leq \rme^{s^2 \norm{u}_2^2 D_k^2 (\varepsilon_{\grad})/(2B_t)}$. To control the norm of $\xi_t$, we apply \citet[Theorem 3.5]{pinelis1994optimum}:
    \[
        \P[\norm{\xi_t}_2 \geq u | \cF_t] \leq 2 \exp\left(- \frac{u^2 \cdot B_t}{2 D_k^2(\varepsilon_{\grad})} \right)\,,
    \]
    thus, defining $v^2 = D_k^2(\varepsilon_{\grad})/B_t$, for a $s > 0$ the following holds
    \begin{align*}
        \E[\exp(s \norm{\xi_t}_2^2)| \cF_t] &= \int_0^\infty \P[\exp(s \norm{\xi_t}_2^2) > u|\cF_t] \rmd u \\
        &\leq 1 +  \int_1^\infty \P\left[\norm{\xi_t}_2 > \sqrt{\frac{1}{s} \log u }| \cF_t\right] \leq 1 + 2 \int_1^{\infty} u^{-\frac{1}{2 v^2 s}} \rmd u\,.
    \end{align*}
    The right-hand side is finite as $s < 1/(4v^2) < 1/(2v^2)$, thus we have
    \[
        \E[\exp(s \norm{\xi_t}_2^2)] \leq 1 + 2 \frac{1}{\frac{1}{2 v^2 s} - 1} =  \frac{1 + 2 v^2s}{1 - 2v^2 s} \leq \exp(6 v^2s)\,,
    \]
    where in the last inequality we used an inequality $(1+x)/(1-x) \leq \exp(3x)$ for $x \leq [0, 1/2]$.
\end{proof}

\subsection{Verification for the Softmax Parametrization}\label{app:softmax_parametrization}

In this section, we verify Assumptions~\ref{ass:lipschitz_param}--\ref{ass:gradient_noise} for the standard softmax parametrization with an appropriate improvement operator and stochastic gradient estimator. For simplicity, we assume a context-free setting, although a generalization to contextual setting is straightforward due to separated optimization over any $x \in \cX$.

\paragraph{Parametrization.}
We fix $\Theta = \R^{\cY}$ and define
\begin{equation}\label{eq:softmax_param_definition}
    \pi_\theta = \softmax(\theta), \qquad  \pi_\theta(y) = \frac{\exp(\theta_y)}{\sum_{y'} \exp(\theta_{y'})}\,.
\end{equation}

We also fix a reference policy $\nu \in \simplex_{\cY}$ with full support, and define an improvement operator $\cT := \cT^{\nu}_\tau$ in Section~\ref{sec:improvement_projection} below, where $\tau \in (0,\tau_0]$ is chosen as in Lemma~\ref{lem:improvement_projection_parameters}. We will show that this choice of $\cT$ satisfies Assumption~\ref{ass:improvement} and yields a uniform PL constant for Assumption~\ref{ass:pl}.

\begin{proposition}[Softmax parametrization satisfies Assumptions~\ref{ass:lipschitz_param}--\ref{ass:gradient_noise}]
\label{prop:softmax_verification}
Fix any full-support reference policy $\nu \in \simplex_{\cY}$, a constant $\varepsilon_{\grad} > 0$, and let 
$\tau = \tau_0$ be as in Lemma~\ref{lem:improvement_projection_parameters}. 
Consider the softmax parametrization~\eqref{eq:softmax_param_definition} with 
improvement operator $\cT := \cT^\nu_\tau$ and stochastic gradient estimator 
$g_t$ defined in~\eqref{eq:stochastic_gradient_estimator_self_play_pg}. 
Then, Assumptions~\ref{ass:lipschitz_param}--\ref{ass:gradient_noise} hold with constants
\begin{align*}
  G &= 1,\quad
  L_{\tpi,\lambda} = \tfrac{5}{2}\bigl(1 + \lambda \log(1/\tpi_{\min})\bigr) + \lambda(4 + \log |\cY|), \quad \pl_{\tpi,\lambda} = \lambda\rme^{-2/\lambda} \cdot c_{\nu}^2, \quad \epspl = 0\,,\\
  &\varepsilon_{\grad} = \varepsilon_{\grad}, \qquad\qquad\qquad
   \sigma^2_{\tpi,\lambda} = D_{\tpi,\lambda}^2(\varepsilon_{\grad}),\qquad\qquad\qquad
  v^2_{\tpi,\lambda} = 6D_{\tpi,\lambda}^2(\varepsilon_{\grad})\,,
\end{align*}
where 
\begin{align*}
  c_{\nu} &\triangleq \exp\left( \min\left\{ - 2\norm{\log \tpi - \log \nu}_{\spann}, \log( \nu_{\min}/(1+\nu_{\min}))\right\}\right) \cdot \nu_{\min}\,,\\
   D_{\tpi,\lambda}^2(\varepsilon_{\grad}) &\triangleq  2\left(1 + 4\lambda \log\left( \frac{2\sqrt{2} \cdot \lambda (1 + \tpi_{\min})}{\tpi_{\min}} \right) + 4\lambda \log(1/\varepsilon_{\grad})\right)^2\,.
\end{align*}
and $\nu_{\min} = \min_{y} \nu(y)$, $\tpi_{\min} = \min_y \tpi(y)$.
\end{proposition}
\begin{proof}
Assumption~\ref{ass:lipschitz_param} follows from 
Lemma~\ref{lem:softmax_lipschitz}. 
Assumption~\ref{ass:smoothness} follows from 
Lemma~\ref{lem:softmax_smoothness}. 
Assumption~\ref{ass:improvement} and the uniform PL condition of 
Assumption~\ref{ass:pl} follow from 
Lemmas~\ref{lem:improvement_projection_parameters} and 
Corollary~\ref{cor:uniform_pl_softmax}. 
Finally, Assumption~\ref{ass:gradient_noise} follows from 
Lemma~\ref{lem:noise_properties_softmax}.
\end{proof}

\begin{corollary}[SPG iteration and sample complexity: tabular softmax]\label{cor:spg_complexity_softmax}
Fix $\delta\in(0,1)$ and $\varepsilon\in(0,1)$.
Consider the context-free tabular softmax parametrization from
Appendix~\ref{app:softmax_parametrization} with improvement operator
$\cT=\cT^\nu_{\tau_0}$ from Lemma~\ref{lem:improvement_projection_parameters},
and run SPG~\eqref{eq:approximate_self_play} with $(\gamma_t,B_t)$ as in
Proposition~\ref{prop:spg_selfplay_stochastic}.
Assume $\tpi_{\min}=\min_{y\in\cY}\tpi(y)>0$ and let $\pl_{\tpi,\lambda}=\lambda\rme^{-2/\lambda}c_\nu^2$
be as in Corollary~\ref{cor:uniform_pl_softmax} (so $\lambda\pl_{\tpi,\lambda}=\lambda^2\rme^{-2/\lambda}c_\nu^2$).
Assume $\lambda\pl_{\tpi,\lambda}\ge 1$ (e.g., it suffices that
$\lambda\ge \max\{2,\rme^{1/2}/c_\nu\}$).

Choose the clipped pairwise estimator from Lemma~\ref{lem:noise_properties_softmax} with bias level
\[
\varepsilon_{\grad}^2=\frac{7\,\pl_{\tpi,\lambda}\,\varepsilon}{180\log(\rme/\delta)}
\qquad\text{(so in particular }M
=\tfrac12+2\lambda\log\!\bigl(\tfrac{2\sqrt2\,\lambda(1+\tpi_{\min})}{\tpi_{\min}\varepsilon_{\grad}}\bigr)
=\Theta(\log(1/\varepsilon))\text{)}.
\]
Then with probability at least $1-\delta$ the iterate $\pi_T$ satisfies
$\subopt^{\tpi}_\lambda(\pi_T)\le \varepsilon$ (equivalently, $\pi_T$ is an $\varepsilon$-VNW for the
$\lambda$-regularized game) once
\[
N_{\mathrm{iter}}(\varepsilon,\delta)=T
=\tcO\!\left(\frac{v^2_{\tpi,\lambda}}{\lambda \pl_{\tpi,\lambda}}\cdot\frac{\log(\rme/\delta)}{\varepsilon}\right),
\qquad
N_{\mathrm{sample}}(\varepsilon,\delta)=\sum_{t=0}^{T-1}B_t
=\tcO\!\left(\frac{T^2}{\pl_{\tpi,\lambda}}\right).
\]
In particular (since here $v^2_{\tpi,\lambda}=6D^2_{\tpi,\lambda}(\varepsilon_{\grad})=\tcO(\log^2(1/\varepsilon))$),
$N_{\mathrm{iter}}(\varepsilon)=\tcO(\varepsilon^{-1})$ and $N_{\mathrm{sample}}(\varepsilon)=\tcO(\varepsilon^{-2})$
up to polylog factors.
\end{corollary}

\begin{proof}
By Proposition~\ref{prop:softmax_verification} and Lemma~\ref{lem:noise_properties_softmax},
Assumptions~\ref{ass:lipschitz_param}--\ref{ass:gradient_noise} hold with $G=1$,
$\epspl=0$, and the stated $\pl_{\tpi,\lambda}$ and $v^2_{\tpi,\lambda}$.
The condition $\lambda\pl_{\tpi,\lambda}\ge G^2$ ensures Proposition~\ref{prop:spg_selfplay_stochastic} applies.
With the chosen $\varepsilon_{\grad}$, the bias term in Proposition~\ref{prop:spg_selfplay_stochastic} satisfies
\[
\frac{6\log(\rme/\delta)}{\pl_{\tpi,\lambda}}\cdot\frac{15}{7}\varepsilon_{\grad}^2 \le \varepsilon/2.
\]
Taking $T=\tcO\!\big(\frac{v^2_{\tpi,\lambda}}{\lambda\pl_{\tpi,\lambda}}\cdot\frac{\log(\rme/\delta)}{\varepsilon}\big)$
makes the remaining (decaying) terms in Proposition~\ref{prop:spg_selfplay_stochastic} at most $\varepsilon/2$,
hence $\subopt^{\tpi}_\lambda(\pi_T)\le\varepsilon$.
Finally, $B_t=\lceil(t+8\kappa_{\tpi,\lambda})/\pl_{\tpi,\lambda}\rceil$ implies
$\sum_{t<T}B_t=\tcO(T^2/\pl_{\tpi,\lambda})$.
\end{proof}

\subsubsection{Lipschitzness and smoothness}

\begin{lemma}[Lipschitz parametrization, Lemma~24 of \citealt{mei2020global}]
\label{lem:softmax_lipschitz}
For any $\theta, \theta' \in \Theta$ it holds
\[
  \|\pi_\theta - \pi_{\theta'}\|_1 
  \leq \|\theta - \theta'\|_{\mathrm{span}} 
  \leq \|\theta - \theta'\|_2.
\]
In particular, Assumption~\ref{ass:lipschitz_param} holds with $G = 1$.
\end{lemma}

\begin{lemma}[Smoothness, cf.\ Lemmas~2 and 14 of \citealt{mei2020global}]
\label{lem:softmax_smoothness}
The softmax parametrization~\eqref{eq:softmax_param_definition} satisfies 
Assumption~\ref{ass:smoothness} with
\[
  L_{\tpi,\lambda} = \tfrac{5}{2}\bigl(1 + \lambda \log(1/\tpi_{\min})\bigr)
        + \lambda(4 + \log |\cY|)\,,
\]
where $\tpi_{\min} = \min_{y} \tpi(y)$.
\end{lemma}

\begin{proof}
The result follows from \citet{mei2020global} by (i) considering a one-state bandit MDP, (ii) absorbing the cross-entropy part of the KL term into the reward and rescaling, and (iii) taking the discount factor $\gamma = 0$. 
\end{proof}

\begin{lemma}[Non-Uniform Polyak–Łojasiewicz inequality, Lemma~15 of \citealt{mei2020global}]\label{lem:pl_br_softmax}
    A function $J^{\tpi}(\theta; \pi)$ satisfies non-uniform Polyak–Łojasiewicz inequality with a constant $c_{\lambda}(\theta) = \lambda \min_{y} \pi_\theta^2(y)$, i.e., for any $\theta \in \Theta$ and $\pi \in \policies$
    \[
        \Vert \nabla J^{\tpi}(\theta; \pi)  \Vert_{2}^2 \geq 2 c_{\lambda}(\theta)(J^{\tpi}(\theta; \pi) - J^{\tpi,\star}(\pi))\,.
    \]
\end{lemma}
\begin{proof}
    Follows from Lemma~15 of \cite{mei2020global} taking $S=1, \gamma=0$.
\end{proof}

We note that, according to \cite{mei2020global}, this parameterization, by default, violates the usual PL-condition (Assumption~\ref{ass:pl}) because its coefficient depends on the minimal probability of the current policy, which can become arbitrarily small during the algorithm iterates. To address this, we introduce an improvement operator that enables control over the minimum probability.

\subsubsection{Improvement Operator and Uniform PL}\label{sec:improvement_projection}

We use a minimal-probability truncation-style argument, similar to that of \citet{zhang2021sample} and \citet{labbi2025global}. Let us fix a policy $\nu \in \policies$ with full support and a constant $\tau > 0$. For any policy $\pi$ define a ratio w.r.t. $\nu$ as $r_\nu(y) \triangleq \pi(y) / \nu(y)$, the low-ratio set $\cY_{\tau}(\pi) \triangleq \{ y \in \cY \mid r_\nu(y) < \tau\}$, and the max-ratio action $y_{\max} = \argmax_{y \in \cY} r_\nu(y)$, where ties are resolved arbitrarily. 

Next, we define the improved ratios $r'_\nu$ by
\begin{equation}\label{eq:def_improvement_projection_policy}
    r'_\nu(y) \triangleq \begin{cases}
        \tau & y \in \cY_{\tau}(\pi)\,,  \\
        r_\nu(y) - \frac{1}{\nu(y_{\max})}\sum_{z \in \cY_\tau(\pi)} \nu(z)(\tau - r_\nu(z))& y = y_{\max}\,, \\
        r_\nu(y) & \text{otherwise}\,.
    \end{cases}
\end{equation}
and the corresponding policy update $\pi^+ = \cU^{\nu}_\tau(\pi)$ via $\pi^+(y) \triangleq \nu(y) \cdot r'_\nu(y)$. It is easy to verify that $\pi^+$ defines a correct probability distribution. 

\begin{lemma}\label{lem:improvement_projection_policy}
    Let
    \[
      \tau_0 \;\triangleq\;
      \min\Bigl\{
        \exp\bigl(-\tfrac{1}{\lambda} - 2 \norm{\log \tpi - \log \nu}_{\spann}\bigr),\;
        (1 + 1/\nu_{\min})^{-1}
      \Bigr\},
    \]
    where $\nu_{\min} = \min_{y \in \cY} \nu(y)$. Then for any $\tau \in (0, \tau_0]$ and for $\pi^+ = \cU^{\nu}_\tau(\pi)$ it holds
    \[
        \subopt^{\tpi}_{\lambda}(\pi^+) \leq \subopt^{\tpi}_{\lambda}(\pi)\,, \qquad \forall y \in \cY: \pi^+(y) \geq \tau \cdot \nu(y)\,.
    \]
\end{lemma}
\begin{proof}
    Without loss of generality, we can assume $|\cY_\tau(\pi)| > 0$ since otherwise $\pi^+ = \pi$ and the second condition is trivially satisfied. Let us define $\delta_\tau(\pi) = \sum_{y \in \cY_\tau(\pi)} \nu(y) \cdot(\tau - r_\nu(y)) \geq 0$. In terms of this accumulated ratio, we have
    \begin{equation}\label{eq:proj_proof_tv_via_delta}
        \norm{\pi - \pi^+}_1 = \sum_{y \in \cY} \nu(y) |r_\nu(y) - r'_\nu(y)| = \sum_{y \in \cY_\tau(\pi)} \nu(y) \cdot(r'_\nu(y) - r_\nu(y)) + \nu(y_{\max})(r_\nu(y) - r'_\nu(y)) = 2 \delta_\tau(\pi)\,.
    \end{equation}

    Then we notice that $\subopt^{\tpi}_{\lambda}(\pi^+) \leq \subopt^{\tpi}_{\lambda}(\pi)$ follows from the following inequality for any competitor policy $\pi^{c} \in \policies$
    \[
         \Delta(\pi^{c}) \triangleq \cP^{\tpi}_{\lambda}(\pi^+ \succ \pi^{c}) - \cP^{\tpi}_\lambda(\pi \succ \pi^{c}) \geq 0\,.
    \]
    To show that, we apply the following decomposition
    \begin{align*}
        \Delta(\pi^{c}) = \cP^{\tpi}_{\lambda}(\pi^+ \succ \pi^{c}) - \cP^{\tpi}_\lambda(\pi \succ \pi^{c}) &= \underbrace{\sum_{y \in \cY} (\pi^+(y) - \pi(y)) \cP(y \succ \pi^{c})}_{T_1} \\
        &\qquad+ \lambda \underbrace{\left[\KL(\pi \Vert \tpi) -  \KL(\pi^+ \Vert \tpi)\right]}_{T_2}\,.
    \end{align*}
    For the first term, we have the following lower bound
    \[
        T_1 = \sum_{y \in \cY} (\pi^+(y) - \pi(y)) \cdot \cP(y \succ \pi^{c})\geq - \norm{\pi^+ - \pi}_1 \norm{\cP(y \succ \pi^{c})}_{\spann} \geq - \delta_\tau(\pi)\,,
    \]
    since $\pi^+ - \pi$ is orthogonal to $\bOne$, $\norm{\cP(y \succ \pi^{c})}_{\spann} \leq 1/2$, and \eqref{eq:proj_proof_tv_via_delta} provides an expression for a $\ell_1$-norm.

    For the second term, we first notice that for any policy $\pi'$ it holds
    \[
        \KL(\pi' \Vert \tpi) = \sum_{y\in\cY} \pi'(y) \log \frac{\pi'(y)}{\nu(y)} + \sum_{y \in \cY} \pi'(y) \log \frac{\nu(y)}{\tpi(y)} = \KL(\pi' \Vert \nu) + \sum_{y \in \cY} \pi'(y) \log\frac{\nu(y)}{\tpi(y)} \,,
    \]
    thus
    \[
        T_2 = \underbrace{\KL(\pi \Vert \nu) - \KL(\pi^+ \Vert \nu)}_{T_{3}} + \sum_{y \in \cY} (\pi(y) - \pi^+(y)) \log \frac{\nu(y)}{\tpi(y)} \geq T_{3} - 2\delta_\tau(\pi) \norm{\log \nu - \log \tpi}_{\spann}\,.
    \]
    Finally, to analyze $T_{3}$, we reformulate KL-divergence as an $f$-divergence for $f(x) \triangleq x \log x -(x-1)$ and have
    \begin{align*}
        T_3 = \sum_{y\in \cY} \nu(y) [ f(r_\nu(y)) - f(r'_\nu(y))] &= \underbrace{\sum_{y \in \cY_\tau(\pi)}  \nu(y) [ f(r_\nu(y)) - f(r'_\nu(y))]}_{T_{3,1}} \\
        &+ \underbrace{\nu(y_{\max})(f(r_\nu(y_{\max})) - f(r'_\nu(y_{\max})))}_{T_{3,2}}\,.
    \end{align*}
    Due to convexity of $f$, we have for any $u,v \in \R$: $f(u) - f(v) \geq f'(v) (u-v) = \log v \cdot(u-v)$, and thus
    \[
        T_{3,1} = \sum_{y \in \cY_\tau(\pi)}  \nu(y) [ f(r_\nu(y)) - f(r'_\nu(y))] \geq \sum_{y \in \cY_\tau(\pi)} \nu(y) f'(r'_\nu(y))(r_\nu(y) - r'_\nu(y))\,.
    \]
    For all $y \in \cY_\tau(\pi)$ it holds $r_\nu(y) < r'_\nu(y)$ and $r'_\nu(y) \geq \tau$ and, as a result, $f'(r'_\nu(y)) \geq f'(\tau) = \log(\tau)$. Thus, for $\tau \leq 1$,
    \[
        T_{3,1}  \geq \log(1/\tau) \cdot \sum_{y \in \cY_\tau(\pi)} \nu(y) (r'_\nu(y) - r_\nu(y)) = \log(1/\tau)  \cdot \delta_\tau(\pi)\,.
    \]
    To analyze the second term, let us define $ \nu(\cY_\tau(\pi)) = \sum_{y \in \cY_\tau(\pi)} \nu(y)$. Then we have
    \begin{align*}
        1 = \sum_{y \in \cY} \nu(y) r_\nu(y) &= \sum_{y \in \cY_\tau(\pi)} \nu(y) \underbrace{r_\nu(y)}_{\leq \tau} + \sum_{y \not\in \cY_\tau(\pi)} \nu(y) \underbrace{r_\nu(y)}_{\leq r_\nu(y_{\max})} \\
        &\leq \nu(\cY_\tau(\pi)) \cdot \tau + (1-\nu(\cY_\tau(\pi))) \cdot r_\nu(y_{\max})\,.
    \end{align*}
    As a result, we have
    \[
        r_\nu(y_{\max}) \geq \frac{1 - \nu(\cY_\tau(\pi)) \tau}{1-\nu(\cY_\tau(\pi))} = 1 + \frac{\nu(\cY_\tau(\pi))}{1-\nu(\cY_\tau(\pi))} \cdot (1-\tau) \geq 1 + \nu(\cY_\tau(\pi))(1-\tau)\,.
    \]

    At the same time, we have $\delta_\tau(\pi) \leq \tau \nu(\cY_\tau(\pi))$, and thus
    \[
        r'_\nu(y_{\max}) \geq 1 + \nu(\cY_\tau(\pi)) \cdot (1-\tau) - \frac{\delta_\tau(\pi)}{\nu(y_{\max})} \geq 1 + \nu(\cY_\tau(\pi)) \cdot \left(1 - \tau\left(1 + \frac{1}{\nu_{\min}}\right)\right)\,,
    \]
    Taking $\tau \leq (1 + 1/\nu_{\min})^{-1}$, we have $r'_\nu(y_{\max}) \geq 1$, and, thus, by the same convexity argument
    \[
        T_{3,2} \geq \nu(y_{\max}) \underbrace{\log(r'_\nu(y_{\max}))}_{\geq 0} \underbrace{(r_{\nu}(y_{\max}) - r'_{\nu}(y_{\max}))}_{\geq 0} \geq 0\,.
    \]
     Combining all bounds together, we have
    \begin{align*}
        \Delta(\pi^{c}) &\geq -\left(1 + 2\lambda \norm{\log \nu - \log \tpi}_{\spann} \right) \delta_\tau(\pi) + \lambda\log(1/\tau) \cdot \delta_\tau(\pi) \\
        &\geq \delta_\tau(\pi) \cdot \left( \lambda \log(1/\tau)  - 1 - 2\lambda \norm{\log \nu - \log \tpi}_{\spann} \right)\,,
    \end{align*}
    which is non-negative for all 
    \[
        0 < \tau \leq \exp\left\{- \frac{1}{\lambda} -  2\norm{\log \tpi -\log \nu}_{\spann} \right\}.
    \]
    The second part of the statement for $y \in \cY_\tau(\pi)$ follows automatically, whereas for $y_{\max}$ it follows from a bound $r'_\nu(y_{\max}) \geq 1$.
\end{proof}

We then lift this operator to parameter space as an operator $\cT^{\nu}_{\tau} \colon \R^{|\cY|} \to \R^{|\cY|}$, defined as
\[
    \cT^{\nu}_{\tau}(\theta) = \log \cU^\nu_\tau(\pi_\theta)\,.
\]
Alternatively, we can define it (equivalently) as follows
\begin{equation}\label{eq:def_improvement_projection_params}
    \cT^{\nu}_{\tau}(\theta)(y) \triangleq \theta(y) + \begin{cases}
        \log( \tau \cdot \nu(y) / \pi_\theta(y))\,, & \pi_\theta(y) / \nu(y) < \tau\,, \\
        \log\left( 1 - \sum_{y'} [\nu(y') \cdot \tau - \pi_\theta(y')]_+/{\pi_\theta(y)} \right) & y=\argmax_{y' \in \cY} \frac{\pi_\theta(y')}{\nu(y')}\,, \\
        0\,,& \text{otherwise}\,.
    \end{cases}
\end{equation}
where $[x]_+ = \max\{x,0\}$. For this operator, the following lemma is a straightforward consequence.
\begin{lemma}\label{lem:improvement_projection_parameters}
    Let
    \[
      \tau_0 \;\triangleq\;
      \min\Bigl\{
        \exp\bigl(-\tfrac{1}{\lambda} - 2 \norm{\log \tpi - \log \nu}_{\spann}\bigr),\;
        (1 + 1/\nu_{\min})^{-1}
      \Bigr\},
    \]
    where $\nu_{\min} = \min_{y \in \cY} \nu(y)$. Then for any $\tau \in (0, \tau_0]$ and for $\theta^+ = \cT^{\nu}_{\tau}(\theta)$ it holds
    \[
        \subopt_{\lambda}^{\tpi}(\pi_{\theta^+}) \leq \subopt_{\lambda}^{\tpi}(\pi_\theta)\,, \qquad \forall y \in \cY: \pi_{\theta^+}(y) \geq \tau \cdot \nu(y)\,.
    \]
    In particular, Assumption~\ref{ass:improvement} holds for $\cT = \cT^{\nu}_{\tau_0}$ for any full-support distribution $\nu$.
\end{lemma}
\begin{proof}
    Follows directly from a fact that $\pi_{\theta^+} = \cU^\nu_\tau(\pi_\theta)$ and Lemma~\ref{lem:improvement_projection_policy}.
\end{proof}

\begin{corollary}[Uniform Polyak–Łojasiewicz inequality]\label{cor:uniform_pl_softmax}
    Let $\cT=\cT^{\nu}_{\tau}$ be an improvement operator defined in Lemma~\ref{lem:improvement_projection_parameters}. For any $\theta \in \Theta_{\cT}$ the following inequality holds almost surely
    \[
        \norm{\nabla J^{\tpi}(\theta; \pi_\theta)}_{2}^2 \geq 2c^{\nu}_{\lambda, \tau} \cdot \subopt^{\tpi}_{\lambda}(\pi_\theta)\,,
    \]
    where $c^{\nu}_{\lambda,\tau} =  \lambda \tau^2 \nu_{\min}^2$. In particular, for a choice $\tau = \tau_0$, we have the following simplified factor
    \[
        \pl_{\tpi,\lambda} = \lambda\rme^{-2/\lambda} \cdot c_{\nu}^2, \qquad c_{\nu} =  \exp\left( \min\left\{ - 2\norm{\log \tpi - \log \nu}_{\spann}, \log( \nu_{\min}/(1+\nu_{\min}))\right\}\right) \cdot \nu_{\min}\,,
    \]
    so Assumption~\ref{ass:pl} holds with $\pl_{\tpi,\lambda}$ and $\epspl = 0$.
\end{corollary}
\begin{proof}
    For any $\theta \in \Theta_{\cT}$, Lemma~\ref{lem:improvement_projection_parameters} implies
    \[
        \pi_\theta(y) \geq \tau \nu(y), \qquad \min_{y \in \cY} \pi_\theta(y) \geq \tau \cdot \nu_{\min}\,.
    \]
    Thus, by Lemma~\ref{lem:pl_br_softmax}, we have
    \[
        \norm{\nabla J^{\tpi}(\theta; \pi_\theta)}_{2}^2 \geq 2 \lambda \min_y\pi^2_\theta(y) \cdot \subopt^{\tpi}_{\lambda}(\pi_\theta) \geq 2\lambda \tau^2 \nu_{\min}^2 \cdot \subopt^{\tpi}_{\lambda}(\pi_\theta)\,.
    \]
    After plugging-in the $\tau = \tau_0$ we have
    \[
        c^\nu_{\lambda, \tau_0} = \lambda \min\{ \exp(-2/\lambda) \cdot \exp(-4\norm{\log \tpi - \log \nu}_{\spann}) , [\nu_{\min}/(1+\nu_{\min})]^{2} \} \nu_{\min}^2\,.
    \]
    Since $\exp(-2/\lambda) \leq 1$, we can simplify the expression and achieve the following bound
    \[
        c^\nu_{\lambda,\tau_0} \geq \pl_{\tpi,\lambda} = \lambda\rme^{-2/\lambda} \cdot \exp\left( \min\left\{ -4 \norm{\log \tpi - \log \nu}_{\spann}, 2 \log( \nu_{\min}/(1+\nu_{\min}))\right\}\right) \cdot \nu_{\min}^2
    \]
\end{proof}

\subsubsection{Gradient Estimation and Noise Bounds}\label{app:softmax_gradient_estimator}

For softmax parametrization, we use the truncated pairwise policy gradient estimator, discussed in Appendix~\ref{app:gradient_estimator} and in this section we verify the conditions of Lemma~\ref{lem:noise_properties_general} for this parameterization.
\begin{lemma}\label{lem:noise_properties_softmax}
Let us consider the context-free softmax parametrization of policies $\pi_\theta$. Then, the mini-batch estimator $g_t$ in~\eqref{eq:stochastic_gradient_estimator_self_play_pg} with the choice $k=\infty$ satisfies Assumption~\ref{ass:gradient_noise} with bias $\varepsilon_{\grad}$, subgaussian constant $\sigma^2_{\tpi,\lambda}=D_{\tpi,\lambda}^2(\varepsilon_{\grad})$, and subexponential constant $v^2_{\tpi,\lambda}=6D_{\tpi,\lambda}^2(\varepsilon_{\grad})$, where
\begin{align*}
  D_{\tpi,\lambda}^2(\varepsilon_{\grad})
  &\triangleq
  2\Bigl(
     1 + 4\lambda \log\Bigl(
       \frac{2\sqrt{2}\,\lambda(1+\tpi_{\min})}{\tpi_{\min}}
     \Bigr)
     + 4\lambda \log(1/\varepsilon_{\grad})
  \Bigr)^2
  \\
  &=
  2\Bigl(
     1 + 4\lambda \log\Bigl(
       \frac{2\sqrt{2}\,\lambda(1+\tpi_{\min})}{\tpi_{\min} \cdot \varepsilon_{\grad}}
     \Bigr)
  \Bigr)^2,
\end{align*}
and $\tpi_{\min}=\min_{y\in\cY}\tpi(y)$.
\end{lemma}
\begin{proof} 
  For this lemma, we verify the conditions of Lemma~\ref{lem:noise_properties_general}. First of all, we notice that $\norm{\nabla \log \pi_\theta(y|x)}_2^2 \leq 2$ since for softmax parameterization we have the following identity $[\nabla \log \pi_\theta(y)]_{y'} = \ind\{ y = y'\} - \pi_\theta(y')$. Next, we see that since $\tpi(y) > 0$ for any $y \in \cY$, then the second part of Lemma~\ref{lem:noise_properties_general} holds with $D_{\tpi} = 1/\tpi_{\min}$.

  Thus, for any $\varepsilon_{\grad} > 0$, Assumption~\ref{ass:gradient_noise} is satisfied with a bias bound $\varepsilon_{\grad}$, subgaussian constant $\sigma^2_{\tpi,\lambda}= D^2_{\tpi, \lambda}$ and subexponential constant $v^2_{\tpi,\lambda} = 6D^2_{\tpi, \lambda}$, where
  \[
      D^2_{\tpi,\lambda} \triangleq 2\left(1 + 4\lambda \log\left( \frac{2\sqrt{2} \cdot \lambda(1 + \tpi_{\min})}{\tpi_{\min}} \right) + 4\lambda \log(1/\varepsilon_{\grad})\right)^2\,.
  \]
\end{proof}

\subsection{Verification under Compatible Fisher Non-Degenerate parameterization}\label{app:fisher_parametrization}

In this section, we state another set of assumptions under which Assumptions~\ref{ass:lipschitz_param}--\ref{ass:gradient_noise} hold. In the following, we assume $\Theta = \R^d$ for some $d > 0$. We start from the general assumption on the policy $\tpi$ which depends only on this policy and does not depend on the parametrization.

\begin{assumption}[Reference-policy regularity]\label{ass:reference_policy}
Define the per-context minimum mass of the reference policy
\[
\tpi_{\min}(x)\triangleq \min_{y\in\cY}\tpi(y| x).
\]
Assume that $\tpi_{\min}(x)>0$ for $\rho$-almost every $x$, and that one of the
following moment conditions holds:
\begin{itemize}
  \item[(i)] Log-moment condition:
  \[
  \E_{x\sim\rho}\!\left[\log^2\!\Bigl(\frac{1}{\tpi_{\min}(x)}\Bigr)\right]\le V_{\tpi}
  \quad\text{for some }V_{\tpi}>0.
  \]
  \item[(ii)] Inverse-mass condition:
  \[
  \E_{x\sim\rho}\!\left[\frac{1}{\tpi_{\min}(x)}\right]\le D_{\tpi}
  \quad\text{for some }D_{\tpi}>0.
  \]
\end{itemize}
\end{assumption}

This assumption controls how often the reference policy assigns extremely small probability to some action. It is used to bound moments (and clipping bias) of log-ratio terms such as $\log(\pi_\theta(y| x)/\tpi(y| x))$ that appear in the KL-regularized policy-gradient estimator.

Condition~(ii) is strictly stronger and implies~(i). Indeed, for all $u\ge 1$, $(\log u)^2 \le \frac{4}{\rme^2}\,u,$
so taking $u=1/\tpi_{\min}(x)$ yields
\[
  \E_{x\sim\rho}\!\left[\log^2\!\Bigl(\frac{1}{\tpi_{\min}(x)}\Bigr)\right] \le \frac{4}{\rme^2}\,\E_{x\sim\rho}\!\left[\frac{1}{\tpi_{\min}(x)}\right] \le \frac{4}{\rme^2}D_{\tpi}.
\]
We state both variants because (i) is sufficient for some bounds (e.g., smoothness), whereas (ii) can yield slightly cleaner bias/threshold choices for clipped estimators.

\begin{assumption}[Bounded score function]\label{ass:bounded_score}
For all $\theta \in \Theta$, the score and the log-policy Hessian are bounded:
\[
  \max_{x \in \supp(\rho), y \in \cY}\|\nabla_\theta \log \pi_\theta(y|x)\|_{2}^2 \le M_g^2,\qquad
  \E_{x \sim \rho, y \sim \pi_\theta(\cdot | x)}\left[ \|\nabla^2_\theta \log \pi_\theta(y|x)\|_{\op}^2\right] \le M_h^2.
\]
\end{assumption}
This assumption is standard in literature of convergence of policy gradient methods (see, e.g., \citet{papini2018stochastic,huang2020momentum,ding2022global,yuan2022general,fatkhullin2023stochastic}). 

\begin{assumption}[Fisher Non-Degeneracy]\label{ass:fisher_matrix}
    The Fisher information matrix is non-degenerate for every $\theta \in \Theta_{\cT}$:
    \[
      F_{\rho}(\theta) \triangleq \E_{x \sim \rho, y \sim \pi_\theta(\cdot | x)}[\nabla \log \pi_\theta(y | x) [\nabla \log \pi_\theta(y | x)]^\top] \succeq \pl_F I.
    \]
\end{assumption}

This assumption is natural for any natural policy gradient analysis, and it holds for certain exponential family parameterized policies \citep[Section 8]{ding2022global} and for certain neural policy families. We refer to a discussion in \citet[Section B.2]{liu2020improved}.

\begin{assumption}[Compatible parametrization]\label{ass:compatible} There exists $\varepsilon_{\bias} \geq 0$ such that for $u_\star(\theta) = F_\rho(\theta)^{-1} \nabla J^{\tpi}(\theta; \pi_\theta)$ it holds for any $\theta \in \Theta_{\cT}$
\[
  \E_{x \sim \rho}\left[ \left\| \cP(\pi_\theta \succ \cdot ) + \lambda \log \frac{\pi_\theta(\cdot|x)}{\tpi(\cdot|x)} - u_\star(\theta)^\top \nabla \log \pi_\theta(\cdot|x) \right\|_{\spann}^2\right]  \leq \varepsilon_{\bias}\,.
\]
\end{assumption}
We note that this assumption corresponds to an Assumption~4.6 of \citet{ding2022global} for $\varepsilon_{\bias} = 0$ and corresponds to the classical assumption of \citet{sutton1999policy} as well as actively used for analysis of natural policy gradient methods \cite{kakade2001natural,agarwal2021theory}. We note that in our case, the norm is more restrictive than $\pi_\theta(\cdot |x)$-weighted $\ell_2$ norm, however we argue that even a stronger assumption is possible to satisfy in practice by using an expressive enough parametrization.

Under these assumptions, we can prove the following result.

\begin{proposition}[Compatible Fisher non-degenerate parameterization satisfies Assumptions~\ref{ass:lipschitz_param}--\ref{ass:gradient_noise}]\label{prop:fisher_verification}
Consider a general parametrization $\theta \mapsto \pi_\theta$ and an improvement operator $\cT$ which satisfy Assumptions~\ref{ass:bounded_score}, Assumption~\ref{ass:fisher_matrix}, and Assumption~\ref{ass:compatible}, and assume that a reference policy $\tpi$ satisfies Assumption~\ref{ass:reference_policy}-(i). Then Assumption~\ref{ass:lipschitz_param}--\ref{ass:improvement} hold with 
\[
  G = M_g,\quad L_{\tpi,\lambda} = \left( \tfrac{1}{2} + \lambda\sqrt{2 + V_{\tpi}} \right) (M_g^2 + M_h) + \lambda M_g^2,\quad
  \pl_{\tpi,\lambda} = \frac{\lambda \pl_F^2}{2M_g^2},\quad \epspl = \frac{\varepsilon_{\bias} \cdot \pl_F^2}{M_g^2}\,.
\]

Additionally, for any $\varepsilon_{\grad} > 0$, a truncated pairwise estimator \eqref{eq:stochastic_gradient_estimator_self_play_pg} with $k=2$ satisfies Assumption~\ref{ass:gradient_noise} with a bias of level $\varepsilon_{\grad}$ and a subgaussian and subexponential parameters
\begin{align*}
  \sigma^2_{\tpi,\lambda} =  M_g^2 \left( 1 + \frac{8 \lambda^2 M_g (V_{\tpi} + 1)}{\varepsilon_{\grad}} \right)^2\,, \qquad v^2_{\tpi,\lambda} = 6M_g^2 \left( 1 + \frac{8 \lambda^2 M_g (V_{\tpi} + 1)}{\varepsilon_{\grad}} \right)^2\,.
\end{align*}
If additionally a reference policy $\tpi$ satisfies Assumption~\ref{ass:reference_policy}-(ii), then a truncated pairwise estimator \eqref{eq:stochastic_gradient_estimator_self_play_pg} with $k=+\infty$ satisfies Assumption~\ref{ass:gradient_noise} with a bias level $\varepsilon_{\grad} > 0$ and other factors
\[
  \sigma^2_{\tpi,\lambda} =   M_g^2 \left(1 + 4 \lambda \log\left( \frac{2\lambda \cdot M_g (D_{\tpi} + 1) }{\varepsilon_{\grad}} \right) \right)^2\,, \qquad v^2_{\tpi,\lambda} =  6\sigma^2_{\tpi,\lambda}\,.
\]

\end{proposition}
\begin{proof}
  Follows from Lemma~\ref{lem:lipschitz_fisher}, Lemma~\ref{lem:smoothness_fisher}, Lemma~\ref{lem:pl_fisher}, and Lemma~\ref{lem:noise_properties_fisher}.
\end{proof}

\begin{corollary}[SPG iteration and sample complexity: compatible Fisher parameterization]\label{cor:spg_complexity_fisher}
Fix $\delta\in(0,1)$ and $\varepsilon\in(0,1)$.
Assume the log-moment condition of Assumption~\ref{ass:reference_policy}(i) with constant $V_{\tpi}$,
and assume Assumptions~\ref{ass:bounded_score}--\ref{ass:compatible} hold with constants
$M_g,M_h,\pl_F$ and $\varepsilon_{\bias}$.
Let $\pl_{\tpi,\lambda}=\lambda\pl_F^2/(2M_g^2)$ and $\epspl=\varepsilon_{\bias}\pl_F^2/M_g^2$ as in
Proposition~\ref{prop:fisher_verification}.
Assume $\lambda\pl_{\tpi,\lambda}\ge G^2$ (here $G=M_g$), i.e.,
\[
\lambda \ge \lambda_{\min}^{\mathrm{F}} \triangleq \sqrt{2}\,\frac{M_g^2}{\pl_F}.
\]

Run SPG~\eqref{eq:approximate_self_play} with $(\gamma_t,B_t)$ as in Proposition~\ref{prop:spg_selfplay_stochastic},
and use the clipped estimator from Proposition~\ref{prop:fisher_verification} with $k=2$ at bias level
\[
\varepsilon_{\grad}^2=\frac{7\,\pl_{\tpi,\lambda} \cdot \varepsilon}{180\log(\rme/\delta)}
\qquad\text{(so }M
=\tfrac12+\tfrac{4\lambda^2 M_g (V_{\tpi}+1)}{\varepsilon_{\grad}}
=\Theta(1/\sqrt{\varepsilon})\text{)}.
\]
Then with probability at least $1-\delta$, once
\begin{align*}
N_{\mathrm{iter}}(\varepsilon,\delta)=T
&=\tcO\!\left(\frac{G^2 v^2_{\tpi,\lambda}}{\lambda \pl_{\tpi,\lambda}}\cdot\frac{\log(\rme/\delta)}{\varepsilon}\right)
=\tcO\!\left(\frac{1}{\varepsilon^2}\right),
\\
N_{\mathrm{sample}}(\varepsilon,\delta)&=\sum_{t=0}^{T-1}B_t
= \tcO(T^2/{\pl_{\tpi,\lambda}}) = \tcO(\varepsilon^{-4}),
\end{align*}
the final iterate $\pi_T$ satisfies the approximate VNW guarantee
\[
\subopt^{\tpi}_\lambda(\pi_T)
\le \varepsilon + \frac{6\log(\rme/\delta)}{\pl_{\tpi,\lambda}}\epspl
= \varepsilon + \frac{12\log(\rme/\delta)}{\lambda}\,\varepsilon_{\bias}.
\]
\end{corollary}

\begin{proof}
Proposition~\ref{prop:fisher_verification} provides $G=M_g$, $\pl_{\tpi,\lambda}$, $\epspl$,
and shows Assumption~\ref{ass:gradient_noise} holds with
$v^2_{\tpi,\lambda}=6M_g^2\bigl(1+8\lambda^2 M_g(V_{\tpi}+1)/\varepsilon_{\grad}\bigr)^2$.
The condition $\lambda\pl_{\tpi,\lambda}\ge G^2$ allows applying
Proposition~\ref{prop:spg_selfplay_stochastic}.
With the chosen $\varepsilon_{\grad}$,
$\frac{6\log(\rme/\delta)}{\pl_{\tpi,\lambda}}\cdot\frac{15}{7}\varepsilon_{\grad}^2\le \varepsilon/2$.
Also, since $\varepsilon_{\grad}^2=\Theta(\varepsilon)$, we have $v^2_{\tpi,\lambda}=\Theta(1/\varepsilon)$,
so taking $T=\tcO\!\big(\frac{G^2 v^2_{\tpi,\lambda}}{\lambda\pl_{\tpi,\lambda}}
\cdot\frac{\log(\rme/\delta)}{\varepsilon}\big)=\tcO(\log(\rme/\delta)/\varepsilon^2)$
makes the remaining terms in Proposition~\ref{prop:spg_selfplay_stochastic} at most $\varepsilon/2$,
yielding
$\subopt^{\tpi}_\lambda(\pi_T)\le \varepsilon + \frac{6\log(\rme/\delta)}{\pl_{\tpi,\lambda}}\epspl$.
Finally, summing $B_t=\lceil(t+8\kappa_{\tpi,\lambda})/\pl_{\tpi,\lambda}\rceil$ gives
$\sum_{t<T}B_t=\tcO(T^2/\pl_{\tpi,\lambda})$.
\end{proof}
\subsubsection{Lipschitzness and smoothness}

We note that this type of results for general policy parameterizations are typically derived for unregularized reinforcement learning problems \citep{papini2018stochastic,yuan2022general}. As a result, it is necessary to prove modified versions.

\begin{lemma}[Lipschitzness]\label{lem:lipschitz_fisher}
Assume Assumption~\ref{ass:bounded_score}. For any $\theta, \theta' \in \Theta$ it holds
\[
  \|\pi_\theta - \pi_{\theta'}\|_{1,\rho} 
  \leq M_g \|\theta - \theta'\|_2.
\]
In particular, Assumption~\ref{ass:lipschitz_param} holds with a constant $G=M_g$.
\end{lemma}
\begin{proof}
  From Lemma~24 of \citet{mei2020global} we have for any $x \in \cX$
  \[
    \norm{\pi_\theta(\cdot | x) - \pi_{\theta'}(\cdot | x)}_1 \leq \norm{\log \pi_{\theta}(\cdot | x) - \log \pi_{\theta'}(\cdot | x)}_{\spann}\,.
  \]
  Next, by Taylor expansion we have
  \[
    \log \pi_{\theta}(\cdot | x) - \log \pi_{\theta'}(\cdot | x) = \nabla \log \pi_{\tilde{\theta}}(\cdot | x)^\top (\theta - \theta')\,,
  \]
  where $\tilde{\theta}$ is some point on the line segment between $\theta$ and $\theta'$. By our assumption on the parameter set $\Theta$ we have $\tilde{\theta} \in \Theta$, thus, applying Assumption~\ref{ass:bounded_score},
  \begin{align*}
    \|\pi_\theta - \pi_{\theta'}\|_{1,\rho}^2 &= \E_{x \sim \rho}[\norm{\pi_\theta(\cdot | x) - \pi_{\theta'}(\cdot | x)}_1^2] \leq \E_{x \sim \rho}[\norm{\nabla \log \pi_{\tilde{\theta}}(\cdot | x)^\top (\theta - \theta')}_\infty^2] \\
    &\leq \E_{x \sim \rho}\left[ \max_{y \in \cY}\norm{\nabla \log \pi_{\tilde{\theta}}(y | x)}_2^2\right] \cdot \norm{\theta - \theta'}_2^2 \leq M_g^2 \norm{\theta - \theta'}_2^2\,.
  \end{align*}
\end{proof}

\begin{lemma}[Smoothness]
\label{lem:smoothness_fisher}
Assume Assumption~\ref{ass:bounded_score} and Assumption~\ref{ass:reference_policy}-(i). Then $\theta\mapsto J^{\tpi}(\theta;\pi)$ is $L_{\tpi,\lambda}$-smooth on $\Theta$ with
\[
  L_{\tpi,\lambda} = \left( \tfrac{1}{2} + \lambda\sqrt{2 + V_{\tpi}} \right) (M_g^2 + M_h) + \lambda M_g^2\,,
\]
i.e., satisfies Assumption~\ref{ass:smoothness} with a constant $L_{\tpi,\lambda}$.
\end{lemma}
\begin{proof}

To simplify the notation, define
\[
  s_\theta(x,y)\triangleq \nabla_\theta \log \pi_\theta(y|x),
  \qquad
  H_\theta(x,y)\triangleq \nabla_\theta^2 \log \pi_\theta(y|x).
\]
Then the standard identities give
\[
  \nabla_\theta \pi_\theta(y|x) = \pi_\theta(y|x)\, s_\theta(x,y),
  \qquad
  \nabla_\theta^2 \pi_\theta(y|x)
  = \pi_\theta(y|x)\big(s_\theta(x,y)s_\theta(x,y)^\top + H_\theta(x,y)\big).
\]
Moreover, since $\sum_y \pi_\theta(y|x)=1$, differentiating twice yields the “Bartlett identity”
\begin{equation}\label{eq:bartlett_identity}
  \E_{y\sim\pi_\theta(\cdot|x)}\big[s_\theta(x,y)s_\theta(x,y)^\top + H_\theta(x,y)\big]=0,
  \qquad \forall x \in \cX,\theta \in \Theta.
\end{equation}

Next, define the per-sample regularized reward
\[
  r_\theta(x,y) \triangleq \cP(\pi \succ y \mid x) + \lambda \log \frac{\pi_\theta(y|x)}{\tpi(y|x)}\,.
\]
Using $\nabla_\theta \ell_\theta(x,y)=s_\theta(x,y)$ and $\E_{y\sim\pi_\theta}[s_\theta]=0$,
one obtains the usual policy-gradient form
\[
  \nabla_\theta J^{\tpi}(\theta;\pi)
  =
  \E_{x\sim\rho}\E_{y\sim\pi_\theta(\cdot|x)}\Big[s_\theta(x,y)\, r_\theta(x,y)\Big]\,.
\]
Differentiating once more gives
\begin{align*}
  \nabla_\theta^2 J^{\tpi}(\theta;\pi)
  &=
  \E_{x\sim\rho}\E_{y\sim\pi_\theta(\cdot|x)}
  \Big[r_\theta(x,y)\big(s_\theta s_\theta^\top + H_\theta\big)\Big]
  \;+\;
  \lambda\,\E_{x\sim\rho}\E_{y\sim\pi_\theta(\cdot|x)}\Big[s_\theta s_\theta^\top\Big],
\end{align*}
where we used $\nabla_\theta r_\theta(x,y)=\lambda s_\theta(x,y)$.

Fix $x \in \cX$. Because of \eqref{eq:bartlett_identity}, for any scalar $c(x)$ we have
\[
  \E_{y\sim\pi_\theta(\cdot|x)}\Big[c(x)\,(s_\theta s_\theta^\top + H_\theta)\Big]=0.
\]
Therefore we may replace $r_\theta(x,y)$ by a centered version without changing the expectation:
pick any $a(x),b_\theta(x)\in\R$ and write
\[
  r_\theta(x,y)
  =
  \big(u_\pi(x,y)-a(x)\big) + \lambda\big(\ell_\theta(x,y)-b_\theta(x)\big)
  \;+\; \big(a(x)+\lambda b_\theta(x)\big),
\]
and the last constant term vanishes when multiplied by $(s_\theta s_\theta^\top + H_\theta)$ and averaged over $y$. Let $S_\theta(x,y)\triangleq s_\theta s_\theta^\top + H_\theta$. Now use Cauchy–Schwarz to bound
\begin{align*}
  \Big\|\E_{x,y}\big[r_\theta(x,y) S_\theta(x,y)\big]\Big\|_{\op}
  &=
  \Big\|\E_{x,y}\big[(r_\theta(x,y)-a(x)-\lambda b_\theta(x))\, S_\theta(x,y)\big]\Big\|_{\op} \\
  &\le
  \E_{x,y}\big[|r_\theta(x,y)-a(x)-\lambda b_\theta(x)|\, \|S_\theta(x,y)\|_{\op}\big] \\
  &\le \sqrt{\E_{x,y}\left[ ( r_\theta(x,y)-a(x)-\lambda b_\theta(x))^2 \right] \cdot \E_{x,y}[\|S_\theta(x,y)\|_{\op}^2]}\,.
\end{align*}
For the Hessian term, we have by Minkowski inequality since $\|S_\theta\|_{\op}\le \|s_\theta\|_2^2 + \|H_\theta\|_{\op}$
\[
  \sqrt{\E_{x,y}[\|S_\theta(x,y)\|_{\op}^2]}
  \le
  \sqrt{\E_{x,y}[\|s_\theta(x,y)\|_2^4]} + \sqrt{\E_{x,y}[\|H_\theta(x,y)\|_{\op}^2]} \leq M_g^2 + M_h\,.
\]
For the reward variance term, we first apply Minkowski inequality again:
\begin{align*}
  \sqrt{\E_{x,y}\left[ ( r_\theta(x,y)-a(x)-\lambda b_\theta(x))^2 \right]} &\le
  \sqrt{\E_{x,y}\left[ ( u_\pi(x,y)-a(x))^2 \right]} \\
  &\qquad + \lambda \sqrt{\E_{x,y}\left[ (\ell_\theta(x,y)-b_\theta(x))^2 \right]}\,.
\end{align*}
For the first term, since $u_\pi(x,y)\in[0,1]$, we can pick $a(x)=1/2$ to get
\[
  \sqrt{\E_{x,y}\left[ ( u_\pi(x,y)-a(x))^2 \right]} \leq \frac{1}{2}.
\]
For the second term, we take $b_\theta(x) = 0$ and decompose as follows:
\begin{align*}
  \E_{x,y}\left[ (\ell_\theta(x,y)-b_\theta(x))^2 \right] &=
  \E_{x}\E_{y\sim\pi_\theta(\cdot|x)}\left[ \left(\log \frac{\pi_\theta(y|x)}{\tpi(y|x)}\right)^2_+\right] \\
  &\qquad + \E_{x}\E_{y\sim\pi_\theta(\cdot|x)}\left[\left(-\log \frac{\pi_\theta(y|x)}{\tpi(y|x)}\right)^2_+ \right]\,.
\end{align*}
For the positive part, we use a bound $\pi_\theta(y|x) \leq 1$ to get
\[
  \E_{x}\E_{y\sim\pi_\theta(\cdot|x)}\left[ \left(\log \frac{\pi_\theta(y|x)}{\tpi(y|x)}\right)^2_+\right]
  \le
  \E_{x}\E_{y\sim\pi_\theta(\cdot|x)}\left[ \log^2 \frac{1}{\tpi(y|x)}\right]
  \le
  V_{\tpi},
\]
where the last inequality follows from Assumption~\ref{ass:reference_policy}-(i). For the negative part, we use a bound $\log(x) \leq x- 1$ for any $x > 0$ and get the following
\begin{align*}
  \E_{x}\E_{y\sim\pi_\theta(\cdot|x)}\left[\left(-\log \frac{\pi_\theta(y|x)}{\tpi(y|x)}\right)^2_+ \right] &=  \E_{x}\E_{y\sim\pi_\theta(\cdot|x)}\left[\left(\log \frac{\tpi(y|x)}{\pi_\theta(y|x)}\right)^2_+ \right] \\
  &= 2 \int_0^{\infty} u\P_{x \sim \rho, y \sim \pi_\theta(\cdot|x)}\left[\log\frac{ \tpi(y|x)}{\pi_\theta(y|x)} \geq u\right] \rmd u\,.
\end{align*}
To bound the last probability, we notice that for $Z = \log\frac{ \tpi(y|x)}{\pi_\theta(y|x)}$ we have $\E[\rme^Z] = 1$, thus, by Markov's inequality 
\[
  \P_{x \sim \rho, y \sim \pi_\theta(\cdot|x)}\left[\log\frac{ \tpi(y|x)}{\pi_\theta(y|x)} \geq u\right] \leq \rme^{-u}\,,
\]
thus 
\[
  \E_{x}\E_{y\sim\pi_\theta(\cdot|x)}\left[\left(-\log \frac{\pi_\theta(y|x)}{\tpi(y|x)}\right)^2_+ \right] \leq 2 \int_0^{\infty} u \rme^{-u} \rmd u = 2\,.
\]

For the remaining term,
\[
  \Big\|\lambda\,\E_{x,y}[s_\theta s_\theta^\top]\Big\|_{\op}
  \le
  \lambda\,\E_{x,y}\|s_\theta s_\theta^\top\|_{\op}
  =
  \lambda\,\E_{x,y}\|s_\theta\|_2^2
  \le \lambda M_g^2.
\]
Combining all the bounds yields the uniform Hessian bound
\[
  \|\nabla_\theta^2 J^{\tpi}(\theta;\pi)\|_{\op}
  \le \left( \tfrac{1}{2} + \lambda\sqrt{2 + V_{\tpi}} \right) (M_g^2 + M_h) + \lambda M_g^2\,.
\]
A uniform bound on the Hessian implies $L_{\tpi,\lambda}$-smoothness (e.g.\ by the standard
second-order Taylor remainder bound), proving the lemma.
\end{proof}

\subsubsection{Polyak–Łojasiewicz inequality}

\begin{lemma}[Gradient comparison]\label{lem:fisher_gradient_comparison}
  Under Assumptions~\ref{ass:bounded_score}-\ref{ass:fisher_matrix}-\ref{ass:compatible}, for any $\theta \in \Theta$ the following inequality holds
  \[
      \norm{\nabla V_{\lambda}^{\tpi}(\pi_\theta,\pi_\theta)}_{\spann, \rho}^2 \leq 2\varepsilon_{\bias} + \frac{2M_g^2}{\pl_F^2} \norm{\nabla J^{\tpi}(\theta; \pi_\theta)}_2^2\,.
  \]
\end{lemma}
\begin{proof}
  First, we notice that $[\nabla V_{\lambda}^{\tpi}(\pi_\theta,\pi_\theta)](x,y) = \cP(\pi_\theta \succ y \mid x) + \lambda( 1 + \log \frac{\pi_\theta(y|x)}{\tpi(y|x)})$, thus, Assumption~\ref{ass:compatible} implies $\norm{\nabla V_{\lambda}^{\tpi}(\pi_\theta,\pi_\theta) - (u_\star)^\top \nabla \log \pi_\theta(y|x)}_{\spann, \rho}^2 \leq \varepsilon_{\bias}$, thus, applying triangle inequality and an inequality $(a+b)^2 \leq 2a^2 + 2b^2$,
  \begin{align*}
      \norm{\nabla V_{\lambda}^{\tpi}(\pi_\theta,\pi_\theta)}_{\spann,\rho}^2 &\leq 2 \varepsilon_{\bias} +  2\norm{(u_\star)^\top \nabla \log \pi_\theta(y|x)}_{\spann,\rho}^2 \\
      &\leq 2 \varepsilon_{\bias} + 2 \norm{(u_\star)^\top \nabla \log \pi_\theta(y|x)}_{\infty,\rho}^2 \\
      &\leq 2 \varepsilon_{\bias} + 2 \max_{y \in \cY}\,\norm{F_\rho(\theta)^{-1} \nabla J^{\tpi}(\theta;\pi_\theta)}_2^2 \cdot \norm{\nabla \log \pi_\theta(y|x)}_{2,\rho}^2\\
      &\leq 2 \varepsilon_{\bias} + \frac{2M_g^2}{\pl_F^2} \norm{\nabla J^{\tpi}(\theta;\pi_\theta)}_2^2\,.
  \end{align*}
\end{proof}

\begin{lemma}[PL inequality]\label{lem:pl_fisher}
    Assume Assumptions~\ref{ass:bounded_score}-\ref{ass:fisher_matrix}-\ref{ass:compatible}. Then for any $\theta \in \Theta_{\cT}$ the following holds
    \[
        \lambda \pl_F^2 / M_g^2 \cdot \subopt^{\tpi}_{\lambda}(\pi_\theta) \leq \norm{\nabla J^{\tpi}(\theta;\pi_\theta)}_2^2 + \frac{\varepsilon_{\bias} \cdot \pl_F^2}{M_g^2}\,.
    \]
    In particular, Assumption~\ref{ass:pl} is satisfied with $\pl_{\tpi,\lambda} = 0.5 \cdot \lambda (\pl_F / M_g)^2$ and $\epspl = \varepsilon_{\bias} \cdot (\pl_F / M_g)^2$\,.
\end{lemma}
\begin{proof}
    First, we follow a standard proof for PL-inequality in the strongly convex case, initially context-wise.

    Let us fix a context $x \in \supp(\rho)$ and a competitor policy $\mu \in \policies$, and then we can define a context-wise value as $V^{\tpi}_{\lambda}(\pi,\mu | x) = \cP(\mu \succ \pi \mid x) + \lambda \KL(\pi(x) \Vert \tpi(x))$, such that $V^{\tpi}_{\lambda}(\pi,\mu) = \E_{x \sim \rho}\left[V^{\tpi}_{\lambda}(\pi,\mu | x) \right]$.

    Let us denote $\pi^{\star}_\mu(\cdot | x) = \argmin_{\pi} V^{\tpi}_{\lambda}(\pi, \mu | x)$ for all $x \in \supp(\rho)$ simultaneously. We notice that this policy is the same as $\pi^{\star}_\mu = \argmin_{\pi} V^{\tpi}_{\lambda}(\pi, \mu)$ due to the additive structure of subproblems with different contexts $x$. Since $V^{\tpi}_{\lambda}(\pi,\mu|x)$ is $\lambda$-strongly convex in $\pi(x)$ with respect to $\ell_1$-norm, thus, for any $\pi$
    \[
        V^{\tpi,\star}_\lambda(\mu | x) \geq V^{\tpi}_{\lambda}(\pi, \mu | x) + \langle [\nabla_{\pi} V^{\tpi}_{\lambda}(\pi, \mu)](x), \pi^\star_\mu(x)- \pi(x)  \rangle + (\lambda/2) \norm{\pi(x) - \pi^\star_\mu(x) }_1^2\,.
    \]
    After rearranging the terms, noticing that $\langle \bOne, \pi^\star_\mu(x) - \pi(x) \rangle = 0$ and using a bound $ab \leq a^2/(2\lambda) + \lambda b^2/2$
    \begin{align*}
        V^{\tpi}_{\lambda}(\pi, \mu|x) - V^{\tpi,\star}_\lambda(\mu|x) &\leq \big\langle [\nabla_{\pi} V^{\tpi}_{\lambda}(\pi, \mu)](x) + c(x) \bOne, \pi^\star_\mu(x) - \pi(x) \big\rangle - (\lambda/2) \norm{\pi(x) - \pi^\star_\mu(x)}_1^2 \\
        &\leq \frac{1}{2\lambda}\norm{[\nabla_{\pi}V^{\tpi}_{\lambda}(\pi, \mu)](x) + c(x) \bOne}_{\infty}^2\,,
    \end{align*}
    where $c \colon \cX \to \R$ is an arbitrary baseline function. Minimizing the right-hand side over $c(x)$ and taking the expectation over $x \sim \rho$ we have
    \[
      V^{\tpi}_{\lambda}(\pi, \mu) - V^{\tpi,\star}_\lambda(\mu) \leq \frac{1}{2\lambda}\norm{\nabla_{\pi} V^{\tpi}_{\lambda}(\pi, \mu)}_{\spann,\rho}^2\,,
    \]
    and, taking $\pi = \mu = \pi_\theta$, we get
    \[
        \subopt^{\tpi}_{\lambda}(\pi_\theta) \leq \frac{1}{2\lambda} \norm{\nabla V_{\lambda}^{\tpi}(\pi_\theta,\pi_\theta)}_{\spann,\rho}^2\,.
    \]
    Finally, applying Lemma~\ref{lem:fisher_gradient_comparison} to derive
    \[
        \subopt^{\tpi}_{\lambda}(\pi_\theta) \leq \frac{M_g^2}{\lambda\pl_F^2} \Vert \nabla J^{\tpi}(\theta; \pi_\theta)  \Vert_{2}^2 + \frac{\varepsilon_{\bias}}{\lambda}\,.
    \]
    By rearranging the terms, we conclude the proof.
\end{proof}

\subsubsection{Gradient Estimation and Noise Bounds}

For general parametrization, we also use the truncated pairwise policy gradient estimator, discussed in Appendix~\ref{app:gradient_estimator}. Next, we verify the conditions of Lemma~\ref{lem:noise_properties_general} for this parameterization.
\begin{lemma}\label{lem:noise_properties_fisher}
    Assume Assumptions~\ref{ass:bounded_score} for parameterization and Assumption~\ref{ass:reference_policy}-(i) for a reference policy. Then, for any $\varepsilon_{\grad} > 0$, a mini-batch gradient estimator $g_t$ defined in \eqref{eq:stochastic_gradient_estimator_self_play_pg} with $k=2$ satisfies Assumption~\ref{ass:gradient_noise} with a bias level $\varepsilon_{\grad}$, subgaussian constant $\sigma^2_{\tpi,\lambda} = D^2_2(\varepsilon_{\grad})$, $v^2_{\tpi,\lambda} = 6 D_2^2(\varepsilon_{\grad})$, where 
    \[
      D_2^2(\varepsilon_{\grad}) = M_g^2 \left( 1 + \frac{8 \lambda^2 M_g (V_{\tpi} + 1)}{\varepsilon_{\grad}} \right)^2\,.
    \]

    If, additionally, Assumption~\ref{ass:reference_policy}-(ii) holds, then a mini-batch gradient estimator $g_t$ defined in \eqref{eq:stochastic_gradient_estimator_self_play_pg} with $k=+\infty$ also satisfies Assumption~\ref{ass:gradient_noise} with a bias level $\varepsilon_{\grad}$, subgaussian constant $\sigma^2_{\tpi,\lambda} =  D^2_\infty(\varepsilon_{\grad})$, $v^2_{\tpi,\lambda} = 6 D^2_\infty(\varepsilon_{\grad})$, where
    \[
      D_\infty^2(\varepsilon_{\grad}) = M_g^2 \left(1 + 4 \lambda \log\left( \frac{2\lambda \cdot M_g (D_{\tpi} + 1) }{\varepsilon_{\grad}} \right) \right)^2\,.
    \]
\end{lemma}
\begin{proof} 
    Directly follows from Lemma~\ref{lem:noise_properties_general}.
\end{proof}

\section{Proximal Point Method with Self-Play Policy Gradients}\label{app:pp-spg}

We now instantiate the proximal point (PP) method from
Section~\ref{app:analysis_pp} using the self-play policy gradient (SPG) procedure from Section~\ref{app:self_play_pg} as an inexact inner solver.

\subsection{Algorithm Description}\label{app:pp-spg-description}

\paragraph{Setup.}
Let $\piref \in \policies$ be a reference policy and let
$\theta_0 \in \R^{d}$ be initial parameters such that
$\pi_0 = \pi_{\theta_0} = \piref$.
We fix:
(i) the regularization strength $\beta > 0$ of the original game,
(ii) the proximal point step size $\eta > 0$, and
(iii) a sequence of self-play learning rates
$(\gamma_{k,t})_{k \ge 0,\, t \ge 0}$ and batch sizes $(B_{k,t})_{k \geq 0,\, t \geq 0}$.
For convenience, we also denote $\beta_{\target} \;\triangleq\; \beta / \eta$.

The PP update at outer iteration $k$ aims to compute
\begin{align*}
  \pi_{k+1}
  \approx
  \argmax_{\pi \in \policies}
  \min_{\pi' \in \policies}
  \Bigl\{
    \cP(\pi \succ \pi')
    &- \beta \KL_{\rho}(\pi \Vert \piref)
    + \beta \KL_{\rho}(\pi' \Vert \piref)\\
    &- \beta_{\target} \KL_{\rho}(\pi \Vert \pi_k)
    + \beta_{\target} \KL_{\rho}(\pi' \Vert \pi_k)
  \Bigr\}.
\end{align*}
We realize this update approximately by running self-play policy gradients (SPG) on the inner two-player game above, using a finite number $T_k$ of self-play steps.

\paragraph{Inner objective.}
At outer iteration $k$, we keep $\pi_k$ fixed and define for any $\theta \in \Theta$ and competitor policy $\pi \in \policies$ the local objective
\[
  J_k(\theta; \pi) \triangleq \cP(\pi \succ \pi_\theta) + \beta \KL_{\rho}(\pi_\theta \Vert \piref) + \beta_{\target} \KL_{\rho}(\pi_\theta \Vert \pi_k),
\]
where $\pi_\theta$ is a parametrized policy satisfying Assumption~\ref{ass:lipschitz_param}-\ref{ass:improvement}.
For a given $\pi$, minimizing $J_k(\theta; \pi)$ in $\theta$
corresponds to computing a regularized best response against $\pi$
with an additional proximal term toward $\pi_k$.

\paragraph{Stochastic gradients.}
At inner iterate $(k,t)$ we maintain parameters $\theta_{k,t}$ and the corresponding policy $\pi_{k,t} = \pi_{\theta_{k,t}}$. We assume access to a stochastic gradient estimator $g_{k,t}$ which satisfies Assumption~\ref{ass:gradient_noise}.

\paragraph{Example: pairwise REINFORCE estimator.}
As a particular instance, we can use a truncated pairwise REINFORCE estimator, analogous to one defined in Appendix~\ref{app:gradient_estimator}. At inner iterate $(k,t)$, given a batch size $B_{k,t}$, we sample contexts $x_{k,t,j} \sim \rho$, independent pairs $(y_{k,t,j}, y'_{k,t,j}) \sim \pi_{k,t}(\cdot\mid x_{k,t,j}) \otimes \pi_{k,t}(\cdot \mid x_{k,t,j})$ and obtain unbiased estimates $p_{k,t,j}$ of $\cP(y_{k,t,j} \succ y'_{k,t,j} \mid x_{k,t,j})$.
Define the single-sample estimator
\begin{align}
\label{eq:pp_spg_grad}
&G_k(\theta \mid x,y,y',p)
 = \frac{1}{2}
    \bigl(
      \nabla_\theta \log \pi_\theta(y|x)
      - \nabla_\theta \log \pi_\theta(y'|x)
    \bigr) \times \\
 & \times
    \clip_{[-M_{k},M_{k}]}\Bigl[
      \tfrac{1}{2} - p
      + \beta \bigl(
          \log \tfrac{\pi_\theta(y|x)}{\piref(y|x)}
          - \log \tfrac{\pi_\theta(y'|x)}{\piref(y'|x)}
        \bigr)
      + \beta_{\target} \bigl(
          \log \tfrac{\pi_\theta(y|x)}{\pi_k(y|x)}
          - \log \tfrac{\pi_\theta(y'|x)}{\pi_k(y'|x)}
        \bigr)
    \Bigr], \notag
\end{align}
where $\beta_{\target} = \beta / \eta$ and $M_{k}$ is a clipping threshold. The mini-batch estimator at step $(k,t)$ is
\[
  g_{k,t} \;\triangleq\; \frac{1}{B_{k,t}} \sum_{j=1}^{B_{k,t}} G_k\bigl(\theta_{k,t} \mid x_{k,t,j}, y_{k,t,j}, y'_{k,t,j}, p_{k,t,j}\bigr),
\]
which is a (biased) estimate of
$\nabla J_k(\theta_{k,t}; \pi_{k,t})$. As shown in Appendix~\ref{app:gradient_estimator}, this choice satisfies Assumption~\ref{ass:gradient_noise} under softmax and compatible Fisher non-degenerate parameterization.

\paragraph{Self-play update and improvement.}
To ensure uniform exploration and obtain a uniform PL constant, we apply the improvement operator $\cT_k$ for the outer step $k$ satisfying Assumption~\ref{ass:improvement} (e.g., $\cT_k \equiv \cT^\nu_\tau$ from Section~\ref{sec:improvement_projection} in the softmax case). Given $g_{k,t}$, the inner update is
\begin{equation}\label{eq:pp_selfplay_clean}
  \theta_{k,t+1}
  \;=\;
  \cT_k\bigl(\theta_{k,t} - \gamma_{k,t}\, g_{k,t}\bigr),
  \qquad
  \pi_{k,t+1} = \pi_{\theta_{k,t+1}}.
\end{equation}

\paragraph{Outer update.}
At the end of the inner loop, after $T_k$ self-play steps, we set
\[
  \theta_{k+1} \;\triangleq\; \theta_{k,T_k},
  \qquad
  \pi_{k+1} \;\triangleq\; \pi_{k,T_k},
\]
and proceed to the next PP iteration.
After $K$ outer iterations, the policy $\pi_K$ is our approximation of
the Nash equilibrium of the original $\beta$-regularized preference game. We summarize the procedure in Algorithm~\ref{alg:pp_spg}.

\begin{algorithm}[t]
  \caption{Proximal Point Method with Self-Play Policy Gradients (PP--SPG)}
  \label{alg:pp_spg}
  \begin{algorithmic}[1]
    \REQUIRE Reference policy $\piref$; regularization parameter $\beta > 0$;
             proximal step size $\eta > 0$; number of outer iterations $K$;
             inner iteration lengths $(T_k)_{k=0}^{K-1}$; learning rates
             $(\gamma_{k,t})_{k,t}$; batch sizes $(B_{k,t})_{k,t}$.
    \ENSURE Approximate von Neumann winner policy $\pi_K$.

    \STATE Initialize parameters $\theta_0$ such that
           $\pi_0 = \pi_{\theta_0} = \piref$.
    \STATE $\beta_{\target} \gets \beta / \eta$.

    \FOR{$k = 0$ \textbf{to} $K-1$} 
      \STATE \textcolor{blue}{\# Outer (proximal point) loop}
      \STATE $\theta_{k,0} \gets \theta_k$; \quad $\pi_{k,0} \gets \pi_k$.
      \FOR{$t = 0$ \textbf{to} $T_k - 1$} 
      \STATE \textcolor{blue}{\# Inner self-play loop}
        \FOR{$j = 1$ \textbf{to} $B_{k,t}$}
          \STATE Sample context $x_{k,t,j} \sim \rho$.
          \STATE Sample $y_{k,t,j} \sim \pi_{k,t}$ and
                 $y'_{k,t,j} \sim \pi_{k,t}$ independently.
          \STATE Obtain an unbiased estimate
                 $p_{k,t,j}$ of
                 $\cP\bigl(y_{k,t,j} \succ y'_{k,t,j} 
                  \mid x_{k,t,j}\bigr)$.
          \STATE Set
                 $G_k^{(j)} \gets
                   G_k\bigl(\theta_{k,t} \mid x_{k,t,j},
                            y_{k,t,j}, y'_{k,t,j}, p_{k,t,j}\bigr)$
                 using Eq.~\eqref{eq:pp_spg_grad}.
        \ENDFOR
        \STATE $g_{k,t} \gets \frac{1}{B_{k,t}} \sum_{j=1}^{B_{k,t}} G_k^{(j)}$.
        \STATE $\theta_{k,t+1} \gets \cT_k\bigl(\theta_{k,t} - \gamma_{k,t}\, g_{k,t}\bigr)$.
        \STATE $\pi_{k,t+1} \gets \pi_{\theta_{k,t+1}}$.
      \ENDFOR
      \STATE $\theta_{k+1} \gets \theta_{k,T_k}$; \quad
             $\pi_{k+1} \gets \pi_{k,T_k}$.
    \ENDFOR

    \STATE \textbf{return} $\pi_K$.
  \end{algorithmic}
\end{algorithm}

\subsection{Analysis of PP--SPG}\label{app:analysis_pp_spg}

In this section, we combine the approximate proximal point analysis from Section~\ref{app:analysis_pp} with the self-play policy-gradient results of Section~\ref{app:self_play_pg} to obtain convergence guarantees for Algorithm~\ref{alg:pp_spg}.

\subsubsection{Inner PP game as a regularized preference game}

Recall that the outer PP iteration (indexed by $k$) aims to compute
\begin{align*}
  \pi_{k+1}
  \approx
  \argmax_{\pi \in \policies} \min_{\pi' \in \policies}
  \Bigl\{
    \cP(\pi \succ \pi')
    &- \beta \KL_{\rho}(\pi \Vert \piref)
    + \beta \KL_{\rho}(\pi' \Vert \piref)
    \\
    &- \beta_{\target} \KL_{\rho}(\pi \Vert \pi_k)
    + \beta_{\target} \KL_{\rho}(\pi' \Vert \pi_k)
  \Bigr\}\,,
\end{align*}
where $\beta_{\target} = \beta / \eta$ and $\eta>0$ is the PP step size.
As in Section~\ref{app:analysis_pp}, we focus on the value function of the $\min$-player (the regularized best-response objective)
\[
  V_k(\pi,\mu) \triangleq \cP(\mu \succ \pi) + \beta \KL_{\rho}(\pi \Vert \piref) + \beta_{\target} \KL_{\rho}(\pi \Vert \pi_k)\,.
\]

The self-play inner loop at outer iterate $k$ is run on the objective
\[
  J_k(\theta; \pi)
  \;\triangleq\;
  \cP(\pi \succ \pi_\theta)
  + \beta \KL_{\rho}(\pi_\theta \Vert \piref)
  + \beta_{\target} \KL_{\rho}(\pi_\theta \Vert \pi_k),
\]
where $\pi$ plays the role of the competitor policy in the self-play game against $\pi_\theta$. To analyze the inner loop using the results from Section~\ref{app:self_play_pg}, we first show that the inner game at outer step $k$ can be expressed as a usual $\lambda$-regularized preference game with an appropriate anchor policy.

\begin{lemma}[Inner game as $\lambda$-regularized preference game]\label{lem:inner_game_equiv}
Let $\lambda \triangleq \beta + \beta_{\target} = \beta \Bigl(1 + \frac{1}{\eta}\Bigr)$ and define the geometric mixture
\begin{align*}
  \forall (x,y) \in \cX \times \cY: \tpi_k(y|x)
  &\propto
  [\piref(y|x)]^{\beta/\lambda} \,
  [\pi_k(y|x)]^{\beta_{\target}/\lambda}
  \\
  &=
  [\piref(y|x)]^{\eta/(1+\eta)} \,
  [\pi_k(y|x)]^{1/(1+\eta)}.
\end{align*}
Then, up to additive constants independent of $\pi$ and $\theta$,
\begin{align*}
  V_k(\pi,\mu)
  &= \cP(\mu \succ \pi)
     + \lambda \KL_{\rho}(\pi \Vert \tpi_k)
     + \text{const} = V^{\tpi_k}_\lambda(\pi,\mu) + \text{const} \,, \\
  J_k(\theta; \pi)
  &= \cP(\pi \succ \pi_\theta)
     + \lambda \KL_{\rho}(\pi_\theta \Vert \tpi_k)
     + \text{const}
   = J^{\tpi_k}_\lambda(\theta; \pi) + \text{const}\,,
\end{align*}
where $J^{\tpi_k}_\lambda$ is precisely the best-response objective from Section~\ref{app:self_play_pg} with regularization parameter $\lambda$ and anchor $\tpi = \tpi_k$.
\end{lemma}

\begin{proof}
For an arbitrary policy $\pi$, we use the identity for any context $x \in \cX$
\[
  \beta \KL(\pi(x) \Vert \piref(x))
  + \beta_{\target} \KL(\pi(x) \Vert \pi_k(x))
  = \lambda \KL(\pi(x) \Vert \tpi_k(x)) + \text{const}(\piref,\pi_k, x)\,,
\]
which follows by expanding the KL terms and grouping the $\pi$-dependent and constant parts. Taking the expectation with respect to $\rho$ yields the expression for $V_k$. Substituting $\pi=\pi_\theta$ gives the expression for $J_k$.
\end{proof}

By Lemma~\ref{lem:inner_game_equiv}, each inner problem at outer iterate $k$ is exactly of the form studied in Section~\ref{app:self_play_pg}, with regularization parameter $\lambda = \beta + \beta_{\target}$ and anchor $\tpi_k$. The only difference is that $\tpi_k$ now depends on $\pi_k$, but $\lambda$ is fixed across outer iterations.

\subsubsection{Inner accuracy and PP residual}

The approximate PP analysis in Proposition~\ref{prop:approx_pp_convergence_context} requires that the inner solver at outer step $k$ returns a policy $\pi_{k+1}$ such that
\begin{equation}\label{eq:pp_residual_definition}
    \bigl\|
      \nabla_\pi V_k(\pi_{k+1}, \pi_{k+1})
    \bigr\|_{\spann, \rho}^2
    \leq \varepsilon,
\end{equation}
for some accuracy parameter $\varepsilon>0$, uniformly in $k$. We note that in general, it is possible to construct a sequence of $\pi_k$ such that for each step $\subopt^{\tpi_k}_{\lambda}(\pi_{k+1}) \to 0$, but $\norm{\nabla_\pi V_k(\pi_{k+1}, \pi_{k+1})}_{\spann, \rho}^2 \not \to 0$. Thus, we need one additional assumption on the properties of parametrization to ensure that an approximate solution to the inner problem will yield small PP residual.

\begin{assumption}[Gradient compatibility]\label{ass:gradient_compatibility}
    There exists $\varepsilon_{\gradcomp} > 0$ and a constant $C_{\gradcomp} > 0$ such that for any $k \in \N$, for any $\theta \in \Theta_{\cT_k}$, the following inequality holds
    \[
      \norm{\nabla_\pi V_{k}(\pi_\theta, \pi_\theta)}_{\spann,\rho}^2 \leq C_{\gradcomp} \cdot \norm{\nabla J_k(\theta; \pi_\theta)}_2^2 + \varepsilon_{\gradcomp}\,.
    \]
\end{assumption}
This assumption holds under both softmax and Fisher compatible parameterizations (see, e.g., Lemma~\ref{lem:fisher_gradient_comparison}). Next, we state the required regularity assumption which should be satisfied for each inner game to guarantee that the self-play policy gradient convergence is independent of step $k$.

\begin{assumption}[Uniform inner-game regularity]\label{ass:parametrization_regularity}
Fix $\beta>0$ and $\eta>0$, and let
\[
  \lambda \triangleq \beta + \beta_{\target} = \beta\Bigl(1+\frac{1}{\eta}\Bigr),
  \qquad \beta_{\target} = \beta/\eta,
\]
and let $\tpi_k$ be the mixed anchor from Lemma~\ref{lem:inner_game_equiv}.
There exist constants
\[
  G,\; L_{\beta,\eta},\; \pl_{\beta,\eta},\; \epspl,\; \varepsilon_{\grad},\; \sigma^2_{\beta,\eta},\; v^2_{\beta,\eta} \;>\; 0\,,
\]
independent of $k$, and a family of improvement operators $(\cT_k)_{k\ge0}$ such that for every
outer iteration $k$ the inner objective $J_k(\theta;\pi)=J^{\tpi_k}_\lambda(\theta;\pi)$
satisfies Assumptions~\ref{ass:lipschitz_param}, \ref{ass:smoothness},
\ref{ass:pl}, and \ref{ass:improvement} with constants bounded by
$(G,L_{\beta,\eta},\pl_{\beta,\eta},\epspl)$, and the mini-batch gradient estimator $g_{k,t}$
satisfies Assumption~\ref{ass:gradient_noise} with bias $\varepsilon_{\grad}$
and variance proxies $(\sigma^2_{\beta,\eta},v^2_{\beta,\eta})$. Moreover, a PP step-size $\eta$ satisfies PL-compatibility condition uniformly:
\[
  \beta(1+1/\eta) \cdot \pl_{\beta,\eta} \geq G^2\,.
\]
\end{assumption}

We note that this assumption does not hold automatically for our previously chosen examples since many constants, such as $L_{\tpi,\lambda}, \pl_{\tpi,\lambda}$, in Assumption~\ref{ass:smoothness}--\ref{ass:gradient_noise} depend on the reference policy $\tpi$, and our sequence of policies $\tpi_k$ might be ill-behaved. However, in Appendix~\ref{app:pp_spg_verification} we show that this assumption is in fact satisfied thanks to established convergence of the proximal point method in span-seminorm of log-probabilities.

\begin{lemma}[From SPG progress to PP residual]\label{lem:inner_residual_vs_subopt}
Fix $k \geq 0$. Assume Assumption~\ref{ass:gradient_compatibility} with constants $(C_{\gradcomp},\varepsilon_{\gradcomp})$. Assume additionally that $J_k(\cdot;\pi)$ is $L_{\beta,\eta}$-smooth (as in Assumption~\ref{ass:parametrization_regularity}).
Then for any $\theta \in \Theta_{\cT_k}$,
\[
  \bigl\|\nabla_\pi V_k(\pi_\theta,\pi_\theta)\bigr\|_{\spann,\rho}^2
  \;\le\;
  2 C_{\gradcomp}\,L_{\beta,\eta} \cdot \subopt^{\tpi_k}_{\lambda}(\pi_\theta)
  \;+\;
  \varepsilon_{\gradcomp}.
\]
In particular, if the inner solver returns $\pi_{k+1}$ such that
$\subopt^{\tpi_k}_{\lambda}(\pi_{k+1}) \le \varepsilon_{\rin}$, then the PP residual
satisfies
\[
  \bigl\|\nabla_\pi V_k(\pi_{k+1},\pi_{k+1})\bigr\|_{\spann,\rho}^2
  \;\le\;
  \varepsilon_{\pp}
  \quad\text{with}\quad
  \varepsilon_{\pp}
  \triangleq
  2C_{\gradcomp}L_{\beta,\eta}\,\varepsilon_{\rin}+\varepsilon_{\gradcomp}.
\]
\end{lemma}

\begin{proof}
Assumption~\ref{ass:gradient_compatibility} gives
\[
  \|\nabla_\pi V_k(\pi_\theta,\pi_\theta)\|_{\spann,\rho}^2
  \le C_{\gradcomp}\,\|\nabla_\theta J_k(\theta;\pi_\theta)\|_2^2 + \varepsilon_{\gradcomp}.
\]
By $L_{\beta,\eta}$-smoothness of $\theta\mapsto J_k(\theta;\pi_\theta)$ (Assumption~\ref{ass:smoothness}), we have
\[
  \|\nabla_\theta J_k(\theta;\pi_\theta)\|_2^2
  \le 2L_{\beta,\eta}\subopt^{\tpi_k}_{\lambda}(\pi_\theta)\,,
\]
which allows us to conclude the statement.
\end{proof}

Additionally, we provide some regularity conditions on the reference policy $\piref$ that serves as the initial policy for our algorithm.

\begin{lemma}[Regularity of reference policy.]\label{lem:ref_regularity}
  Let $\piref \in \policies$ be a full-support reference policy. Then it holds for any $\beta > 0$
  \[
    \KL_{\rho}(\pistar_\beta \Vert \piref) \leq \frac{1}{2\beta}, \qquad \subopt_\beta(\piref) \leq \frac{1}{2},\qquad \norm{\log\piref - \log\pistar_\beta}^2_{\spann, \rho} \leq \frac{1}{4\beta^2}\,.
  \]
\end{lemma}
\begin{proof}
  By optimality conditions, $\pistar_\beta$ can be characterized as a best response against itself, thus, for any $x \sim \supp(\rho)$ and $y \in \cY$ we have
  \[
       \log \pistar_\beta(y|x) = \frac{1}{\beta}\left(1/2 - \cP(\pistar_\beta(x) \succ y | x) \right) + \log \piref(y|x) + c^\star(x)\,,
  \]
  where $c^\star(x)$ is a log-normalization constant, which is equal to
  \[
    c^\star(x) = -\log \E_{y \sim \piref(\cdot|x)}\left[\exp\left(\frac{1}{\beta}\left(1/2 - \cP(\pistar_\beta(x) \succ y | x) \right)\right)\right]\,.
  \]
  In particular, this expression automatically implies the last bound since
  \[
    \norm{\log\piref(\cdot|x) - \log\pistar_\beta(\cdot|x)}_{\spann} \leq \frac{1}{\beta} \norm{1/2 - \cP(\pistar_\beta(x) \succ \cdot | x)}_{\infty} \leq \frac{1}{2\beta}\,.
  \]
  For the expression for the KL divergence, we use the log-ratio expression to write
  \[
    \KL_{\rho}(\pistar_\beta \Vert \piref) = \E_{x \sim \rho}\left[\frac{1}{\beta} \E_{y \sim \pistar_\beta(\cdot| x)}\left[1/2 - \cP(\pistar_\beta(x) \succ y | x) \right] + c^\star(x) \right] = \E_{x \sim \rho}\left[ c^\star(x)\right]\,.
  \]
  We note that by Jensen's inequality, it holds
  \begin{align*}
    \exp(c^\star(x)) &= \frac{1}{\E_{y \sim \piref(\cdot|x)}\left[\exp\left(\frac{1}{\beta}\left(1/2 - \cP(\pistar_\beta(x) \succ y | x) \right)\right)\right]} \\
    &\leq \frac{1}{\exp\left(\frac{1}{\beta}\left( \frac{1}{2} - \cP(\pistar_\beta(x) \succ \piref(x) | x) \right)\right)} \leq \exp\left(\frac{1}{2\beta}\right)\,,
  \end{align*}
  thus $\KL_\rho(\pistar_\beta \Vert \piref) \leq 1/(2\beta)$. Finally, for the suboptimality, we use the definition:
  \[
    \subopt_\beta(\piref) = \max_{\pi} \left\{ \frac{1}{2} - \cP_\beta(\piref \succ \pi) \right\} \,,
  \]
  where
  \[
    \cP_\beta(\piref \succ \pi) = \cP(\piref \succ \pi) - \beta \KL_{\rho}(\piref \Vert \piref) + \beta \KL_{\rho}(\pi \Vert \piref)\,.
  \]
  In particular, $\cP_\beta(\piref \succ \pi) = \cP(\piref \succ \pi) + \beta \KL_{\rho}(\pi \Vert \piref) \geq 0$. Thus, we have $\subopt_\beta(\piref) \leq 1/2$. 
\end{proof}

\begin{lemma}[Regularity of initial policy]\label{lem:init_regularity}
  Let $\piref \in \policies$ be a full-support reference policy and let $\pi_k$ be the policy at outer iteration $k$ of Algorithm~\ref{alg:pp_spg}. Then, for any $\beta>0$ and $\eta>0$, it holds that the mixed anchor policy $\tpi_k$ from Lemma~\ref{lem:inner_game_equiv} satisfies
  \[
     \subopt^{\tpi_k}_{\lambda}(\pi_k) \leq \subopt_{\beta}(\pi_k)\,,
  \]
  where $\lambda = \beta + \beta_{\target}$ with $\beta_{\target} = \beta/\eta$. 
\end{lemma}
\begin{proof}
  Using a property of $\tpi_k$ being a geometric mixture, we have
  \begin{align*}
    \cP^{\tpi_k}_{\lambda}(\pi' \succ \pi) &= \cP(\pi' \succ \pi) - \lambda \KL_{\rho}(\pi' \Vert \tpi_k) + \lambda \KL_{\rho}(\pi \Vert \tpi_k) \\
    &= \cP(\pi' \succ \pi) - \beta \KL_{\rho}(\pi' \Vert \piref) + \beta \KL_{\rho}(\pi \Vert \piref) \\
    &\qquad\qquad\qquad - \beta_{\target} \KL_{\rho}(\pi' \Vert \pi_k) + \beta_{\target} \KL_{\rho}(\pi \Vert \pi_k)\,.
  \end{align*}
  Taking $\pi' = \pi_k$, the expression above gives
  \begin{align*}
    \cP^{\tpi_k}_{\lambda}(\pi_k \succ \pi) &= \cP(\pi_k \succ \pi) - \beta \KL_{\rho}(\pi_k \Vert \piref) + \beta \KL_{\rho}(\pi \Vert \piref) \\
    &\qquad\qquad\qquad - \beta_{\target} \KL_{\rho}(\pi_k \Vert \pi_k) + \beta_{\target} \KL_{\rho}(\pi \Vert \pi_k) \\
    &= \cP_\beta(\pi_k \succ \pi) + \beta_{\target} \KL_{\rho}(\pi \Vert \pi_k) \geq \cP_{\beta}(\pi_k \succ \pi)\,.
  \end{align*}
  Thus, we have
  \[
    \subopt^{\tpi_k}_{\lambda}(\pi_k) \leq \subopt_{\beta}(\pi_k)\,.
  \]
\end{proof}

Finally, we establish a uniform bound on the suboptimality and log-probability span-norm of the iterates $\pi_k$ assuming that each inner problem is solved up to a fixed accuracy. This result will be useful to select the number of inner iterations in a consistent manner.
\begin{lemma}[Uniform regularity of PP iterates]\label{lem:pp_iterate_uniform_regularity}
Assume $\beta \le 1$ and fix $K\ge 0$. Assume that for all $k<K$ the $k$-th proximal subproblem is solved up to accuracy
$\varepsilon_{\pp} > 0$:
\[
  \bigl\|\nabla_\pi V_k(\pi_{k+1},\pi_{k+1})\bigr\|_{\spann,\rho}^2
  \le \varepsilon_{\pp}\,.
\]
Then the iterate $\pi_K$ satisfies
\[
  \|\log \pi_K-\log \piref\|_{\spann,\rho}
  \le \frac{\frac12+\sqrt{\varepsilon_{\pp}}}{\beta}\,,
  \qquad
  \KL_\rho(\pi_K\Vert \piref)
  \le \frac{1+2\sqrt{\varepsilon_{\pp}}}{\beta}\,,
\]
and consequently
\[
  \subopt_\beta(\pi_K)
  \le \frac32 + 2\sqrt{\varepsilon_{\pp}}\,,
  \qquad
  \big\|\log \pi_K-\log \pistar_\beta\big\|_{\spann,\rho}^2
  \le \frac{(1+\sqrt{\varepsilon_{\pp}})^2}{\beta^2}\,.
\]
\end{lemma}

\begin{proof}
Let $\beta_{\target}\triangleq \beta/\eta$ and recall
\[
  V_k(\pi,\mu)=\cP(\mu\succ \pi)+\beta\KL_\rho(\pi\Vert \piref)+\beta_{\target}\KL_\rho(\pi\Vert \pi_k).
\]
Write $\zeta_{k+1}\triangleq \nabla_\pi V_k(\pi_{k+1},\pi_{k+1})$.
A direct differentiation (same as in Lemma~\ref{lem:value_strong_convexity_context}) gives, for $\rho$-a.e.\ $x$ and all $y$,
\[
  \zeta_{k+1}(x,y)
  =
  \cP(\pi_{k+1}(x)\succ y\mid x)
  +\beta\Bigl(1+\log\frac{\pi_{k+1}(y\mid x)}{\piref(y\mid x)}\Bigr)
  +\beta_{\target}\Bigl(1+\log\frac{\pi_{k+1}(y\mid x)}{\pi_k(y\mid x)}\Bigr).
\]
The additive ``$+1$'' terms are constant across $y$ and therefore irrelevant for the span seminorm.
Thus, for each $x$ we may subtract a baseline $c_{k+1}(x)$ and rewrite as
\[
  \cP(\pi_{k+1}(x)\succ y\mid x)
  +\beta\log\frac{\pi_{k+1}(y\mid x)}{\piref(y\mid x)}
  +\beta_{\target}\log\frac{\pi_{k+1}(y\mid x)}{\pi_k(y\mid x)}
  =
  e_{k+1}(x,y),
\]
where $e_{k+1}(x)\triangleq \zeta_{k+1}(x)-c_{k+1}(x)\bOne$ satisfies
$\|e_{k+1}\|_{\spann,\rho}=\|\zeta_{k+1}\|_{\spann,\rho}\le \sqrt{\varepsilon_{\pp}}$.

Rearranging and using $\beta+\beta_{\target}=\beta(1+1/\eta)=\beta(1+\eta)/\eta$, we obtain the (approximate) log equation
\begin{equation}\label{eq:pp_log_equation_uniform}
  \log \pi_{k+1}
  =
  \frac{\eta}{1+\eta}\log \piref
  +\frac{1}{1+\eta}\log \pi_k
  -\frac{\eta}{\beta(1+\eta)}\,\cP(\pi_{k+1}\succ \cdot)
  +\frac{\eta}{\beta(1+\eta)}\,e_{k+1}
  +\text{(const)}.
\end{equation}
Subtract $\log\piref$ and take $\|\cdot\|_{\spann,\rho}$.
Using Minkowski and that $\|e_{k+1}\|_{\spann,\rho}\le \sqrt{\varepsilon_{\pp}}$, we get
\begin{align*}
  \|\log \pi_{k+1}-\log \piref\|_{\spann,\rho}
  &\le
  \frac{1}{1+\eta}\,\|\log \pi_k-\log \piref\|_{\spann,\rho}
  +\frac{\eta}{\beta(1+\eta)}\,
    \|\cP(\pi_{k+1}\succ \cdot)\|_{\spann,\rho}
  \\
  &\qquad +\frac{\eta}{\beta(1+\eta)}\,\sqrt{\varepsilon_{\pp}}.
\end{align*}
For each $x$, the vector $y\mapsto \cP(\pi_{k+1}(x)\succ y\mid x)$ takes values in $[0,1]$, hence
$\|\cP(\pi_{k+1}(x)\succ \cdot\mid x)\|_{\spann}\le \tfrac12$, and therefore
$\|\cP(\pi_{k+1}\succ \cdot)\|_{\spann,\rho}\le \tfrac12$.
Thus, defining $D_k\triangleq \|\log \pi_k-\log \piref\|_{\spann,\rho}$,
\[
  D_{k+1}
  \le
  \frac{1}{1+\eta}D_k
  +\frac{\eta}{\beta(1+\eta)}\Bigl(\frac12+\sqrt{\varepsilon_{\pp}}\Bigr).
\]
Since $D_0=\|\log\pi_0-\log\piref\|_{\spann,\rho}=0$, unrolling the recursion yields for all $K$:
\[
  D_K
  \le
  \frac{1}{\beta}\Bigl(\frac12+\sqrt{\varepsilon_{\pp}}\Bigr).
\]
This proves the first bound in the lemma.

\textbf{Bounding $\KL_\rho(\pi_K\Vert \piref)$ from span of log-ratios.}
Fix $x$ and abbreviate $p=\pi_K(\cdot\mid x)$, $q=\piref(\cdot\mid x)$, and $r=\log(p/q)$.
Because $\sum_y q_y e^{r_y}=1$, we have
$1\le e^{\max_y r_y}$ and $1\ge e^{\min_y r_y}$, so $\max r\ge 0\ge \min r$ and therefore
\[
  \max_y r_y \le \max r-\min r = 2\|r\|_{\spann}.
\]
Then
\[
  \KL(p\Vert q)=\sum_y p_y r_y \le \max_y r_y \le 2\|r\|_{\spann}.
\]
Taking expectation over $x\sim\rho$ and using $\E\|r(x)\|_{\spann}\le \|r\|_{\spann,\rho}$ gives
\[
  \KL_\rho(\pi_K\Vert \piref)
  \le 2\|\log\pi_K-\log\piref\|_{\spann,\rho}
  \le \frac{1+2\sqrt{\varepsilon_{\pp}}}{\beta}.
\]

\textbf{Suboptimality bound.}
By definition,
\begin{align*}
  \subopt_\beta(\pi_K)
  &=
  \max_\mu \Bigl(\tfrac12-\cP(\pi_K\succ\mu)
  +\beta\KL_\rho(\pi_K\Vert\piref)-\beta\KL_\rho(\mu\Vert\piref)\Bigr)
  \\
  &\le \tfrac12+\beta\KL_\rho(\pi_K\Vert\piref).
\end{align*}
Plugging the KL bound yields
\[
  \subopt_\beta(\pi_K)\le \tfrac12+\bigl(1+2\sqrt{\varepsilon_{\pp}}\bigr)
    = \tfrac32+2\sqrt{\varepsilon_{\pp}}\,.
  \]

\textbf{Span bound to $\pistar_\beta$.}
By the triangle inequality,
\[
  \|\log \pi_K-\log \pistar_\beta\|_{\spann,\rho}
  \le
  \|\log \pi_K-\log \piref\|_{\spann,\rho}
  +\|\log \piref-\log \pistar_\beta\|_{\spann,\rho}.
\]
By Lemma~\ref{lem:ref_regularity}, $\|\log \piref-\log \pistar_\beta\|_{\spann,\rho}\le \frac{1}{2\beta}$.
Therefore
\[
  \|\log \pi_K-\log \pistar_\beta\|_{\spann,\rho}
  \le
  \frac{\frac12+\sqrt{\varepsilon_{\pp}}}{\beta}+\frac{1}{2\beta}
  =\frac{1+\sqrt{\varepsilon_{\pp}}}{\beta},
\]
and squaring gives the last statement of the lemma.
\end{proof}

\subsubsection{Deterministic PP--SPG}\label{app:pp_spg_deterministic}

For a warm-up, assume the inner-loop gradients are exact: $g_{k,t}=\nabla J_k(\theta_{k,t};\pi_{k,t})$.

\begin{corollary}[Deterministic PP--SPG]\label{cor:pp_spg_deterministic}
Assume $\beta \leq 1$ and let $\varepsilon_{\rin} > 0$. Assume that parameterization and PP learning rate $\eta$ satisfy Assumptions~\ref{ass:gradient_compatibility} and \ref{ass:parametrization_regularity}. Then, the deterministic version of Algorithm~\ref{alg:pp_spg} using 
\[
  T_k \geq \frac{4L_{\beta,\eta}}{\pl_{\beta,\eta}} \log\left(\frac{3/2 + 2\sqrt{\varepsilon_{\pp}}}{\varepsilon_{\rin}}\right)
\]
iterations for each inner loop, where $\varepsilon_{\pp} \triangleq 2C_{\gradcomp}L_{\beta,\eta}\,\left(\varepsilon_{\rin} + \frac{3\epspl}{2\pl_{\beta,\eta}}\right) + \varepsilon_{\gradcomp}$, with a learning rate $\gamma_{k,t}\equiv \gamma = 1/(2L_{\beta,\eta})$, for any $k \geq 0$ it holds
\[
        \KL_{\rho}(\pistar_\beta \Vert \pi_{k}) \leq (1+\eta/2)^{-k} \cdot \frac{1}{2\beta} + \frac{2 \varepsilon_{\pp}}{\beta^2}\,,
    \]
    and
    \[
        \subopt_{\beta}(\pi_{k}) \leq (1+\eta/2)^{-k} \cdot \left(\frac{5}{2\beta^2} + \frac{1}{\eta}\right)  + \left(\frac{5}{2\beta^2} + \frac{1}{\eta}\right)\frac{4\varepsilon_{\pp}}{\beta}\,,
    \]
    and, moreover,
    \[
        \big\|\log \pi_{k} - \log \pistar_\beta\big\|_{\spann,\rho}^2 \leq \frac{1}{4\beta^2}(1+\eta)^{-k}  + \frac{1}{\beta^3}  \cdot (1+\eta/2)^{-k}
        + \frac{4 \varepsilon_{\pp}}{\beta^4}\,.
    \]

\end{corollary}
\begin{proof}
We want to show that for all $k \geq 0$, the following bounds hold:
\[
  \bigl\|\nabla_\pi V_k(\pi_{k+1},\pi_{k+1})\bigr\|_{\spann,\rho}^2
  \le
  \varepsilon_{\pp}
  \triangleq
  2C_{\gradcomp}L_{\beta,\eta}\,\left(\varepsilon_{\rin} + \frac{3\epspl}{2\pl_{\beta,\eta}}\right) + \varepsilon_{\gradcomp}\,.
\]
Assume now $k \geq 0$ and that the statements hold for $k' < k$. We will show they hold for $k$. By Lemma~\ref{lem:ref_regularity} for $k=0$ and by the induction hypothesis and Lemma~\ref{lem:pp_iterate_uniform_regularity} for $k > 0$, we have
\[
  \subopt_{\beta}(\pi_k) \leq \frac{3}{2} + 2\sqrt{\varepsilon_{\pp}}\,.
\]
Applying Proposition~\ref{prop:spg_selfplay_deterministic}, Lemma~\ref{lem:init_regularity}, and a bound $1+x \leq \rme^{x}$ to the inner loop at each $k$ gives
\begin{align*}
  \subopt^{\tpi_k}_{\lambda}(\pi_{k+1}) &\leq
  \Bigl(1-\frac{\pl_{\beta,\eta}}{4L_{\beta,\eta}}\Bigr)^{T_k} \subopt^{\tpi_k}_{\lambda}(\pi_{k}) + \frac{3\epspl}{2\pl_{\beta,\eta}} \\
  &\leq \Bigl(1-\frac{\pl_{\beta,\eta}}{4L_{\beta,\eta}}\Bigr)^{T_k} \subopt_{\beta}(\pi_{k}) + \frac{3\epspl}{2\pl_{\beta,\eta}} \\
  &\leq \exp\left(-\frac{\pl_{\beta,\eta}}{4L_{\beta,\eta}} T_k\right) \left( \frac{3}{2} + 2\sqrt{\varepsilon_{\pp}} \right) + \frac{3\epspl}{2\pl_{\beta,\eta}} \leq 
  \varepsilon_{\rin} + \frac{3\epspl}{2\pl_{\beta,\eta}}\,.
\end{align*}
By Lemma~\ref{lem:inner_residual_vs_subopt},
\[
  \bigl\|\nabla_\pi V_k(\pi_{k+1},\pi_{k+1})\bigr\|_{\spann,\rho}^2
  \le
  2C_{\gradcomp}L_{\beta,\eta}\,\left(\varepsilon_{\rin} + \frac{3\epspl}{2\pl_{\beta,\eta}}\right) + \varepsilon_{\gradcomp} = \varepsilon_{\pp}\,.
\]
Thus Proposition~\ref{prop:approx_pp_convergence_context} applies to the outer PP loop with
$\varepsilon=\varepsilon_{\pp}$, yielding the following corollary.
\end{proof}

\subsubsection{Stochastic PP--SPG}\label{app:pp_spg_stochastic}

We now analyze the stochastic version of Algorithm~\ref{alg:pp_spg}, where the
inner-loop gradients are estimated by mini-batches and satisfy
Assumption~\ref{ass:gradient_noise}. The proof mirrors the deterministic case,
but we additionally track the confidence parameter and apply a union bound over
outer iterations.

\begin{corollary}[Stochastic PP--SPG]\label{cor:pp_spg_stochastic}
Assume $\beta\le 1$ and let $\delta\in(0,1)$. Assume
Assumptions~\ref{ass:gradient_compatibility} and \ref{ass:parametrization_regularity}.
Fix a target inner accuracy $\varepsilon_{\rin}>0$ and a target number of outer steps $K \in \N$. Define the per-outer failure probability as
\[
  \delta_k \triangleq \frac{\delta}{K},\qquad k\in\{0,\dots,K-1\}\,.
\]
and a target PP-residual accuracy as
\[
  \varepsilon_{\pp} \triangleq 2C_{\gradcomp}L_{\beta,\eta}\left(\varepsilon_{\rin}+\frac{6\log(K\rme/\delta)}{\pl_{\beta,\eta}} \cdot \left(\epspl +  15/7 \cdot \varepsilon_{\grad}^2\right)\right) +\varepsilon_{\gradcomp}\,.
\]
Define $\kappa_{\beta,\eta} = \frac{L_{\beta,\eta}}{\pl_{\beta,\eta}} \geq 1$ and sequences $(\gamma_t)_{t\geq 0}$ and $(B_t)_{t \geq 0}$ as follows
  \[   
      \gamma_t = \frac{4t + 32\kappa_{\beta,\eta} - 2}{\pl_{\beta,\eta}(t + 8\kappa_{\beta,\eta})^2}  = \Theta\left( \frac{1}{\pl_{\beta,\eta} t} \right)\,, \qquad B_t =  \left\lceil \frac{t + 8 \kappa_{\beta,\eta}}{\pl_{\beta,\eta}} \right\rceil = \Theta\left( \frac{t}{\pl_{\beta,\eta}} \right)\,.
  \]
Then, Algorithm~\ref{alg:pp_spg} with $\gamma_{k,t} \equiv \gamma_t$ and $B_{k,t} \equiv B_t$, and 
\[
  T_k = \tcO\left( \frac{v^2_{\beta,\eta}}{\varepsilon_{\rin}} + \frac{\kappa_{\beta,\eta}}{\sqrt{\varepsilon_{\rin}}} + \sqrt{ \frac{\kappa_{\beta,\eta} (\sigma^2_{\beta,\eta} + v^2_{\beta,\eta})}{\varepsilon_{\rin}}}\right)
\]

outputs $\{\pi_k\}_{k \in [0,\ldots,K]}$ such that with probability at least $1-\delta$ for all $k=0,\dots,K$ the following holds:
\begin{align*}
  \KL_{\rho}(\pistar_\beta \Vert \pi_{k})
  &\le (1+\eta/2)^{-k} \cdot\frac{1}{2\beta}
  + \frac{2 \varepsilon_{\pp}}{\beta^2},\\
  \subopt_{\beta}(\pi_{k})
  &\le (1+\eta/2)^{-k} \cdot \left(\frac{5}{2\beta^2} + \frac{1}{\eta}\right)
  + \left(\frac{5}{2\beta^2} + \frac{1}{\eta}\right)\frac{4\varepsilon_{\pp}}{\beta},\\
  \big\|\log \pi_{k} - \log \pistar_\beta\big\|_{\spann,\rho}^2
  &\le \frac{1}{4\beta^2}(1+\eta)^{-k}
  + \frac{1}{\beta^3}(1+\eta/2)^{-k}
  + \frac{4 \varepsilon_{\pp}}{\beta^4}\,.
\end{align*}
\end{corollary}

\begin{proof}
We prove that the PP residual condition holds uniformly for all outer steps with probability at least $1-\delta$,
and then apply Proposition~\ref{prop:approx_pp_convergence_context} exactly as in the deterministic case.

Fix an outer iteration $k\in\{0,\dots,K-1\}$. Condition on the history up to the beginning of the
$k$-th inner loop (so $\pi_k$ and hence $\tpi_k$ are fixed).
By Assumption~\ref{ass:parametrization_regularity}, the inner objective $J_k$ satisfies the same
regularity constants $(L_{\beta,\eta},\pl_{\beta,\eta},\epspl)$ and the stochastic gradients satisfy
the same noise constants $(\varepsilon_{\grad},\sigma^2_{\beta,\eta},v^2_{\beta,\eta})$.
Therefore Proposition~\ref{prop:spg_selfplay_stochastic} is applicable to the inner loop.

Moreover, as in the deterministic proof, for $k=0$ Lemma~\ref{lem:ref_regularity} controls the initial
outer iterate, and for $k>0$ the induction hypothesis together with Lemma~\ref{lem:pp_iterate_uniform_regularity}
gives a uniform bound on $\subopt_\beta(\pi_k)$ and hence via Lemma~\ref{lem:init_regularity} a uniform bound on
$\subopt^{\tpi_k}_{\lambda}(\pi_k)$ required to instantiate the SPG bound.
Concretely, on the event that the residual bounds hold for all steps $k'<k$, we have
\[
  \subopt^{\tpi_k}_{\lambda}(\pi_k)
  \le \subopt_{\beta}(\pi_k)
  \le \frac{3}{2}+2\sqrt{\varepsilon_{\pp}}.
\]

Now apply Proposition~\ref{prop:spg_selfplay_stochastic} to the inner run at step $k$ with confidence $\delta_k=\delta/K$ after $T_k$ steps and required learning rate and batch size sequences $(\gamma_t)$ and $(B_t)$. It implies that with probability at least $1-\delta_k$,
\begin{align*}
    \subopt^{\tpi_k}_\lambda(\pi_{k+1}) &\leq \frac{64 \cdot \kappa^2_{\beta,\eta} \log(K\rme/\delta)}{(T_k + 8\kappa_{\beta,\eta} - 1)^2} \left( \frac{3}{2} + 2\sqrt{\varepsilon_{\pp}} \right) \\
    &\qquad+ \frac{ 24 \cdot  \kappa_{\beta,\eta} \log(K\rme/\delta) \cdot \log(1+T_{k}/(2\kappa_{\beta,\eta}))}{(T_k + 8 \kappa_{\beta,\eta} - 1)^2}\left( \frac{675 \cdot \sigma^2_{\beta,\eta}}{49} + 2 v^2_{\beta,\eta} \right) \\
    &\qquad + \frac{6G^2\cdot v^2_{\beta,\eta} \log(K\rme/\delta)}{\lambda  \pl_{\beta,\eta} \cdot (T_k + 8 \kappa_{\beta,\eta} - 1)} + \frac{6\log(K\rme/\delta)}{\pl_{\beta,\eta}} \cdot \left(\epspl +  15/7 \cdot \varepsilon_{\grad}^2\right) \,,
\end{align*}
where $\lambda = \beta(1 + 1/\eta)$.
To guarantee that all terms with dependence on $T_k$ are at most $\varepsilon_{\rin}$, it suffices to choose $T_k$ such that
\begin{align}
    \label{eq:stochastic_pp_spg_cond1}
    &\frac{64 \cdot \kappa^2_{\beta,\eta} \log(K\rme/\delta)}{(T_k + 8\kappa_{\beta,\eta} - 1)^2} \left( \frac{3}{2} + 2\sqrt{\varepsilon_{\pp}} \right) \leq \frac{\varepsilon_{\rin}}{3}\,,\\
    \label{eq:stochastic_pp_spg_cond2}
    &\frac{ 24 \cdot  \kappa_{\beta,\eta} \log(K\rme/\delta) \cdot \log(1+T_{k}/(2\kappa_{\beta,\eta}))}{(T_k + 8 \kappa_{\beta,\eta} - 1)^2}\left( \frac{675 \cdot \sigma^2_{\beta,\eta}}{49} + 2 v^2_{\beta,\eta} \right) \leq \frac{\varepsilon_{\rin}}{3}\,,\\
    \label{eq:stochastic_pp_spg_cond3}
    &\frac{6G^2\cdot v^2_{\beta,\eta} \log(K\rme/\delta)}{\lambda  \pl_{\beta,\eta} \cdot (T_k + 8 \kappa_{\beta,\eta} - 1)} \leq \varepsilon_{\rin}/3\,.
\end{align}
Our choice $\beta(1+1/\eta) \cdot \pl_{\beta,\eta} = G^2$ simplifies the last condition \eqref{eq:stochastic_pp_spg_cond3} to
\[
  T_k \geq \frac{18 \cdot v^2_{\beta,\eta} \log(K\rme/\delta)}{\varepsilon_{\rin}} - 8 \kappa_{\beta,\eta} + 1\,,
\]
which asymptotically scales as $\cO\left(v^2_{\beta,\eta} \log(K/\delta)/\varepsilon_{\rin}\right)$. For the first two conditions \eqref{eq:stochastic_pp_spg_cond1} and \eqref{eq:stochastic_pp_spg_cond2}, it suffices to choose
\begin{align*}
  T_k \geq \max&\Bigg\{
    \sqrt{\frac{192 \cdot \kappa^2_{\beta,\eta} \log(K\rme/\delta) \left( \frac{3}{2} + 2\sqrt{\varepsilon_{\pp}} \right)}{\varepsilon_{\rin}}}, \\
    &\quad \sqrt{\frac{72 \cdot  \kappa_{\beta,\eta} \log(K\rme/\delta) \left( \frac{675 \cdot \sigma^2_{\beta,\eta}}{49} + 2 v^2_{\beta,\eta} \right) \log(1+T_{k}/(2\kappa_{\beta,\eta}))}{\varepsilon_{\rin}}}
  \Bigg\}\,.
\end{align*}
In particular, the first term scales as $\cO\left(\kappa_{\beta,\eta} \sqrt{\frac{\log(K/\delta)}{\varepsilon_{\rin}}}\right)$ and the second term scales as $\tilde{\cO}\left(\sqrt{ \frac{\kappa_{\beta,\eta} (\sigma^2_{\beta,\eta} + v^2_{\beta,\eta}) \log(K/\delta)}{\varepsilon_{\rin}}}\right)$.

Under this choice, we have shown that with probability at least $1-\delta_k$,
\[
  \subopt^{\tpi_k}_\lambda(\pi_{k+1}) \leq \varepsilon_{\rin} + \frac{6\log(K\rme/\delta)}{\pl_{\beta,\eta}} \cdot \left(\epspl +  15/7 \cdot \varepsilon_{\grad}^2\right).
\]

On the same event, Lemma~\ref{lem:inner_residual_vs_subopt} gives
\[
  \bigl\|\nabla_\pi V_k(\pi_{k+1},\pi_{k+1})\bigr\|_{\spann,\rho}^2
  \le
  2C_{\gradcomp}L_{\beta,\eta}\left(\varepsilon_{\rin}+ \frac{6\log(K\rme/\delta)}{\pl_{\beta,\eta}} \cdot \left(\epspl +  15/7 \cdot \varepsilon_{\grad}^2\right)\right)
  +\varepsilon_{\gradcomp}
  = \varepsilon_{\pp}.
\]
Define the event
\[
  \cE
  \triangleq
  \bigcap_{k=0}^{K-1}\cE_k\,, \qquad \cE_k \triangleq
  \left\{
    \bigl\|\nabla_\pi V_k(\pi_{k+1},\pi_{k+1})\bigr\|_{\spann,\rho}^2 \le \varepsilon_{\pp}
  \right\} \,.
\]
By the per-step success probability $\P(\cE_k)\ge 1-\delta_k$ and a union bound, $\P(\cE) \geq 1-\delta$. On $\cE$, the residual condition \eqref{eq:pp_residual_definition} holds uniformly for all outer iterations with $\varepsilon=\varepsilon_{\pp}$. Thus Proposition~\ref{prop:approx_pp_convergence_context} applies, giving the three displayed bounds.
Finally, since $\pi_0=\piref$ we use Lemma~\ref{lem:ref_regularity} to bound
initial quantities as in the deterministic case.
\end{proof}

\subsection{Assumption Verification}\label{app:pp_spg_verification}

First, we prove a technical lemma that will be useful to verify all assumptions connected to a non-degeneracy of reference policies.

\begin{lemma}\label{lem:log_min_prob_vs_span}
  Let $\pi,\mu \in \policies$ be two full-support policies and define $\pi_{\min}(x) = \min_{y \in \cY} \pi(y|x)$ and $\mu_{\min}(x) = \min_{y \in \cY} \mu(y|x)$. Then, the following bound holds
  \begin{align*}
    \forall x\in \cX: \max_{y \in \cY}\left|\log\frac{\pi(y|x)}{\mu(y|x)}\right| &\leq 2 \norm{\log \pi(x) - \log \mu(x)}_{\spann}\,, \\
    \E_{x \sim \rho}\left[ \max_{y \in \cY}\log^2\left(\frac{\pi(y|x)}{\mu(y|x)}\right) \right] &\leq 4 \norm{\log \pi - \log \mu}_{\spann,\rho}^2\,.
  \end{align*}
  Additionally, the following bound holds
  \[
    \E_{x \sim \rho}\left[ \log^2\left(\frac{1}{\pi_{\min}(x)}\right) \right] \leq 2\E_{x \sim \rho}\left[ \log^2\left(\frac{1}{\mu_{\min}(x)}\right) \right]+ 8 \norm{\log \pi - \log \mu}_{\spann,\rho}^2\,.
  \]
\end{lemma}
\begin{proof}
  First, we use the relation between the span seminorm and a range:
  \[
    \norm{\log \pi(x) - \log \mu(x)}_{\spann} = \frac{1}{2}\left(\max_{y \in \cY} \log \frac{\pi(y|x)}{\mu(y|x)} - \min_{y \in \cY} \log \frac{\pi(y|x)}{\mu(y|x)}\right)\,.
  \]
  We notice that for any pair of policies it holds $\max_{y \in \cY} \log \frac{\pi(y|x)}{\mu(y|x)} \geq 0$ and $\min_{y \in \cY} \log \frac{\pi(y|x)}{\mu(y|x)} \leq 0$, thus we have
  \[
    \max\left\{ \max_{y \in \cY} \log \frac{\pi(y|x)}{\mu(y|x)}, - \min_{y \in \cY} \log \frac{\pi(y|x)}{\mu(y|x)} \right\} \leq 2\norm{\log \pi(x) - \log \mu(x)}_{\spann}\,.
  \]
  Finally, we note that the left-hand side is equal to $\max_{y \in \cY} |\log \frac{\pi(y|x)}{\mu(y|x)}|$. Taking a square and expectation over $x \sim \rho$ concludes the first part of the proof.

  The second part of the proof follows from a bound 
  \[
    \norm{\log \pi(x)}_\infty \leq \norm{\log \pi(x) - \log \mu(x)}_{\infty} + \norm{\log \mu(x)}_\infty \leq 2 \norm{\log \pi(x) - \log \mu(x)}_{\spann} + \norm{\log \mu(x)}_\infty\,,
  \] 
  which follows from the triangle inequality and the first part of the lemma, squaring both sides, taking expectation over $x \sim \rho$, and using a bound $(a+b)^2 \leq 2a^2 + 2b^2$.
\end{proof}
In particular, this lemma allows us to show that if a policy $\pi_k$ generated by the PP--SPG algorithm satisfies $\norm{\log \pi_k - \log \pistar_\beta}_{\spann,\rho}^2 \leq C$ for some constant $C>0$, then the mixed anchor $\tpi_k$ for the next step satisfies Assumption~\ref{ass:reference_policy}. Next, we prove the result that connects the inexact solution to the inner problem to PP residuals needed to establish the convergence of the outer algorithm.

\subsection{Verification for Softmax Parametrization}

In this section, we show that context-free softmax parametrization satisfies Assumptions~\ref{ass:gradient_compatibility} and \ref{ass:parametrization_regularity}.

\begin{lemma}\label{lem:context_free_minprob_control}
  Consider a context-free setting and assume that $\norm{\nabla V_{k}(\pi_{k+1}, \pi_{k+1})}_{\spann}^2  \leq \varepsilon_{\pp}$ for all $k < K$. Then for the next step $K$ it holds
  \[
    \log(1/\tpi_{K,\min}) \leq \log(1/\piref_{\min}) +  \frac{1+2\sqrt{\varepsilon_{\pp}}}{(1+\eta)\beta}\,,
  \]
  or, equivalently,
  \[
    \frac{1}{\tpi_{K,\min}} \leq \frac{1}{\piref_{\min}} \cdot \exp\left( \frac{1+2\sqrt{\varepsilon_{\pp}}}{(1+\eta)\beta}\right)\,.
  \]
\end{lemma}
\begin{proof}
  We start from the first inequality of Lemma~\ref{lem:log_min_prob_vs_span} with $\pi=\piref$ and $\mu = \tpi_k$
  \[
    \norm{\log \piref - \log \tpi_k}_\infty \leq 2 \norm{\log \piref - \log \tpi_k}_{\spann}\,.
  \]
  For the left-hand side, we apply a reverse triangle inequality and a fact $\norm{\log \tpi_k}_\infty = \log (1/\tpi_{k,\min})$. For the right-hand side, we use a definition of a geometric mixture. Overall, we have
  \[
    \log(1/\tpi_{k,\min}) \leq \log(1/\piref_{\min}) + \frac{2}{1+\eta} \norm{\log \pi_k - \log \piref}_{\spann}\,.
  \]
  Next, we use Lemma~\ref{lem:pp_iterate_uniform_regularity} to bound the second term. Thus, we have
  \[
    \log(1/\tpi_{k,\min}) \leq \log(1/\piref_{\min}) + \frac{1+2\sqrt{\varepsilon_{\pp}}}{(1+\eta)\beta}\,.
  \]
  Exponentiating both sides concludes the proof.
\end{proof}

Also, we will need the following technical lemma.
\begin{lemma}\label{lem:Hx_lower_bound}
    Let $x \in \R^d$, then for any $\pi \in \simplex_d$
    \[
        \norm{ H(\pi) x }_{\spann} \geq \pi_{\min} \cdot \norm{x}_{\spann}\,,
    \]
    where $\pi_{\min} = \min_{i \in [d]} \pi(i)$ and $H(\pi) = \diag(\pi) - \pi \pi^\top$.
\end{lemma}
\begin{proof}
    Without loss of generality, we assume that $\pi(i) > 0$ for any $i\in [d]$, otherwise the statement trivially holds.

    First, we notice that $\norm{x}_{\spann} = 1/2 \cdot \left(\max_{i} x_i - \min_i x_i\right)$ and define $i_{\max} = \argmax_{i} x_i$ and $i_{\min} = \argmin_{i} x_i$.
    
    Then, we notice that $H(\pi) x = \pi \odot (x - \mu \bOne)$, where $\mu = \pi^\top x$ is a mean of $x$ under $\pi$. Since $\pi$ is a positive measure, we have $x_{i_{\max}} \geq \mu$ and $x_{i_{\min}} \leq \mu$, therefore
    \begin{align*}
        \max_{i} \pi(i) (x_i - \mu) &\geq \pi(i_{\max}) \cdot \underbrace{(x_{i_{\max}} - \mu)}_{\geq 0} \geq \pi_{\min} \cdot (x_{i_{\max}} - \mu)\,\\
        \min_{i} \pi(i) (x_i - \mu) &\leq \pi(i_{\min}) \underbrace{(x_{i_{\min}} - \mu)}_{\leq 0} \leq \pi_{\min} \cdot (x_{i_{\min}} - \mu)\,,
    \end{align*}
    thus
    \begin{align*}
        \norm{ H(\pi) x }_{\spann} &= (1/2) \cdot \left( \max_{i} \pi(i) (x_i - \mu) - \min_{i} \pi(i) (x_i - \mu)\right) \\
        &\geq (1/2) \pi_{\min} (x_{i_{\max}} - \mu - x_{i_{\min}} + \mu) = \pi_{\min} \norm{x}_{\spann}.
    \end{align*}
\end{proof}

\begin{corollary}[Convergence guarantees for Softmax parametrization]\label{cor:final_rates_softmax}
  Let $\theta \mapsto \pi_\theta$ be a context-free softmax parametrization and $(\cT_k)_{k \geq 0}$ a sequence of improvement operators $\cT_k \equiv \cT^{\piref}_{\tau_{k,0}}$ for $\tau_{k,0}$ equal to $\tau_0$ with $\tpi=\tpi_k$, as defined in Lemma~\ref{lem:improvement_projection_parameters} with $\nu = \piref$ and $\tpi = \tpi_k$. Assume $\beta \leq 1$. Fix $\varepsilon_{\rin} \in (0, (\tau_0 \piref_{\min})^2/(2 L_{\beta,\eta}))$ as a desired accuracy for approximate solving of the inner problem, where $\tau_0$ and $L_{\beta,\eta}$ are defined below.
  
  Then the following statements hold.
  \begin{itemize}[leftmargin=1em]
    \item[(i)] Assumption~\ref{ass:gradient_compatibility} holds with $C_{\gradcomp} = 1/(\tau_0 \cdot \piref_{\min})^2$ and $\varepsilon_{\gradcomp} = 0$, where 
    \[
      \tau_0 \triangleq \min\left\{ \exp\left( - \frac{3+\eta}{\beta(1+\eta)}\right) , (1 + 1/\piref_{\min})^{-1} \right\}\,;
    \]
    \item[(ii)] There exists a truncation threshold $M_k > 0$ such that Assumption~\ref{ass:parametrization_regularity} holds with 
    \begin{align*}
      G &= 1,\quad L_{\beta,\eta} = \frac{5}{2}\left(1 + \beta (1+1/\eta) \log(1/\piref_{\min}) +  3/\eta\right) + \beta(1+1/\eta)(4 + \log |\cY|)\,, \\
      \pl_{\beta,\eta} &= \beta (1 + 1/\eta) \exp\left(- \frac{2}{\beta (1+1/\eta)}\right) \times \\
      &\qquad \times \exp\left( \min\left\{ - \frac{6}{\beta(1+\eta)}, 2\log\left( \frac{\piref_{\min}}{1+\piref_{\min}}\right)\right\}\right) \cdot (\piref_{\min})^2\,,\\ 
      \epspl &= 0\,,\ \qquad\qquad \varepsilon_{\grad} = \sqrt{\varepsilon_{\rin}}, \qquad\qquad \sigma^2_{\beta,\eta} = D_{\beta,\eta}^2,\qquad\qquad v^2_{\beta,\eta} = 6D_{\beta,\eta}^2\,,
    \end{align*}
    where 
    \[
      D_{\beta,\eta}^2 = 2\left(1 + 4\beta (1+1/\eta) \log\left( \frac{4\sqrt{2} \cdot \beta (1+1/\eta)}{\piref_{\min}} \right) + 12/\eta + 2\beta(1+1/\eta) \log(1/\varepsilon_{\rin})\right)^2\,,
    \]
    and $\eta$ is chosen as solution to $\beta(1+1/\eta) \cdot \pl_{\beta,\eta} = 1$.
    \item[(iii)] Algorithm~\ref{alg:pp_spg} outputs $\cO(\varepsilon_{\rin})$-optimal regularized policy after $K = \tcO(1)$ outer iterations, $T_k = \tcO(1/\varepsilon_{\rin})$ inner iterations, using in total $\tcO(1/\varepsilon_{\rin}^2)$ samples, ignoring constants depending on $\piref_{\min}, \pistar_{\beta,\min}, \beta,\eta$ and logarithmic terms in $\varepsilon_{\rin}$ and the confidence level $\delta$.
  \end{itemize}
  
\end{corollary}
\begin{proof}
  To prove the statements, we also need to prove that for all $k$ it holds $\norm{ \nabla V_{k}(\pi_{k+1}, \pi_{k+1})}_{\spann,\rho}^2 \leq 1$.
  We prove this statement by induction over steps $k$. We assume that for all $k' < k$ it holds $\norm{ \nabla V_{k'}(\pi_{k'+1}, \pi_{k'+1})}_{\spann,\rho}^2 \leq 1$ for all $k' < k$. This assumption allows us to use Lemma~\ref{lem:context_free_minprob_control} to control the minimum probability of the mixed anchor policy $\tpi_k$.

  Next, we notice that under our improvement operator choice, for any $k$ and any $\theta \in \Theta_{\cT_k}$ it holds $\pi_{\theta,\min} \geq \tau_{k,0} \cdot \piref_{\min}$, where $\tau_{k,0}$ is defined in Lemma~\ref{lem:improvement_projection_parameters} with $\nu = \piref$, $\tpi = \tpi_k$, and $\lambda = \beta(1+1/\eta)$. In particular, we have the following expression for $\tau_{k,0}$:
  \[
    \tau_{k,0} = \min\Bigl\{
        \exp\bigl(-\tfrac{1}{\lambda} - 2 \norm{\log \tpi_k - \log \piref}_{\spann}\bigr),\;
        (1 + 1/\piref_{\min})^{-1}
      \Bigr\}\,.
  \]
  To bound this constant away from zero uniformly over $k$, we need to show that $\norm{\log \tpi_k - \log \piref}_{\spann}$ is uniformly bounded over $k$. For that, we start from a definition of geometric mixture and then apply Lemma~\ref{lem:pp_iterate_uniform_regularity}:
  \[
    \norm{\log \tpi_k - \log \piref}_{\spann} = \frac{1}{1+\eta} \norm{\log \pi_k - \log \piref}_{\spann} \leq \frac{1/2 + \sqrt{\varepsilon_{\pp}}}{\beta(1+\eta)} \leq \frac{3}{2\beta(1+\eta)}\,.
  \]
  Thus, we have 
  \begin{align*}
    \tau_{k,0} &\geq \min\left\{ \exp\left( - \frac{3 + \eta}{\beta(1+\eta)} \right), (1 + 1/\piref_{\min})^{-1} \right\}  \triangleq \tau_0 > 0\,,
  \end{align*}
  uniformly over $k$.
  
  Next, we can prove the first statement of the lemma. Notice that $\nabla J^{\tpi_k}(\theta; \pi) = H(\pi_\theta) \cdot\nabla V_k(\pi_\theta;\pi)$, where $H(\pi_\theta) = \diag(\pi_\theta) - \pi_\theta \pi_\theta^\top$ is the parametrization Jacobian matrix. Thus, we can apply Lemma~\ref{lem:Hx_lower_bound} to obtain for any $\theta \in \Theta_{\cT_k}$
  \[
    \norm{\nabla J^{\tpi_k}(\theta; \pi)}_{2} \geq \norm{\nabla J^{\tpi_k}(\theta; \pi)}_{\spann} \geq \pi_{\theta,\min} \cdot \norm{\nabla V_k(\pi_\theta;\pi)}_{\spann} \geq \tau_0 \cdot \piref_{\min} \cdot \norm{\nabla V_k(\pi_\theta;\pi)}_{\spann}\,,
  \]
  thus concluding the proof of the first statement with $C_{\gradcomp} = 1/(\tau_0 \cdot \piref_{\min})^2$ and $\varepsilon_{\gradcomp} = 0$.

  The second statement follows from Proposition~\ref{prop:softmax_verification}, a bound on $\log(1/\tpi_{k,\min})$ which follows from Lemma~\ref{lem:context_free_minprob_control}, and taking $\varepsilon_{\grad} = \sqrt{\varepsilon_{\rin}}$. In fact, a constant $G=1$ does not depend on $k$; the bound on a smoothness coefficient follows using $\lambda = \beta(1+1/\eta)$:
  \begin{align*}
    L_{\tpi,\lambda} &= \frac{5}{2}\bigl(1 + \lambda \log(1/\tpi_{\min})\bigr) + \lambda(4 + \log |\cY|) \\
    &\leq \frac{5}{2}\left(1 + \beta (1+1/\eta) \log(1/\piref_{\min}) +  \frac{1+2\sqrt{\varepsilon_{\pp}}}{\eta}\right) + \beta(1+1/\eta)(4 + \log |\cY|) \\
    &\leq  \frac{5}{2}\left(1 + \beta (1+1/\eta) \log(1/\piref_{\min}) +  3/\eta \right) + \beta(1+1/\eta)(4 + \log |\cY|) \triangleq L_{\beta,\eta}\,.
  \end{align*}
  After that bound, we can prove the induction step to show that $\norm{ \nabla V_{k}(\pi_{k+1}, \pi_{k+1})}_{\spann,\rho}^2 \leq 1$ for all $k$. In fact, we have
  \begin{align*}
    \norm{ \nabla V_{k}(\pi_{k+1}, \pi_{k+1})}_{\spann,\rho}^2 &\leq 2 C_{\gradcomp} L_{\beta,\eta} \cdot \varepsilon_{\rin} \leq 2 \cdot \frac{1}{(\tau_0 \cdot \piref_{\min})^2} \cdot L_{\beta,\eta} \cdot \varepsilon_{\rin} = \varepsilon_{\pp} \leq 1\,,
  \end{align*}  
  where the last inequality holds for our choice of $\varepsilon_{\rin}$ small enough.

  Next, for bound on $\pl_{\tpi,\lambda}$ we recall that under our choice of $\nu =\piref$
  \[
    \pl_{\tpi,\lambda} = \lambda\rme^{-2/\lambda} \cdot c_{\piref}^2\,, 
\] 
    for $c_{\piref} \triangleq \exp\left( \min\left\{ - 2\norm{\log \tpi_k - \log \piref}_{\spann}, \log( \piref_{\min}/(1+\piref_{\min}))\right\}\right) \cdot \piref_{\min}$, where we repeat our calculations for $\norm{\log \tpi_k - \log \piref}_{\spann}$ to obtain
  \[
    c_{\piref} \geq \exp\left( \min\left\{ - \frac{3}{\beta(1+\eta)}, \log( \piref_{\min}/(1+\piref_{\min}))\right\}\right) \cdot \piref_{\min} > 0\,,
  \]
  and $\pl_{\tpi,\lambda} \geq \pl_{\beta,\eta} > 0$ uniformly over $k$. Next, the bound on $\sigma^2_{\tpi,\lambda}$ and $v^2_{\tpi,\lambda}$ again follows automatically from a bound on $\tpi_{k,\min}$, which depends on the following quantity
  \begin{align*}
    D_{\tpi,\lambda}^2(\sqrt{\varepsilon_{\rin}}) &=  2\left(1 + 4\lambda \log\left( \frac{2\sqrt{2} \cdot \lambda (1 + \tpi_{\min})}{\tpi_{\min}} \right) + 2\lambda \log(1/\varepsilon_{\rin})\right)^2  \\
    &\leq 2\bigg(1 + 4\beta (1+\eta^{-1}) \log\bigg( \frac{4\sqrt{2} \cdot \beta (1+1/\eta)}{\piref_{\min}} \bigg) + \frac{12}{\eta} + 2\beta(1+\eta^{-1}) \log\bigg(\frac{1}{\varepsilon_{\rin}}\bigg)\bigg)^2, 
  \end{align*}
  implying that the indicated bound on $\sigma^2_{\tpi,\lambda}$ and $v^2_{\tpi,\lambda}$ hold uniformly over $k$. To conclude the second statement of the lemma, we need to choose $\eta > 0$ such that $\beta(1+1/\eta) \cdot \pl_{\beta,\eta} = 1$. In fact, such $\eta$ exists since the function $\eta \mapsto \beta(1+1/\eta) \cdot \pl_{\beta,\eta}$ is continuous and decreasing on $\R_{>0}$, with limits $\lim_{\eta \to 0} \beta(1+1/\eta) \cdot \pl_{\beta,\eta} = +\infty$ and 
  \[
    \lim_{\eta \to +\infty} \beta(1+1/\eta) \cdot \pl_{\beta,\eta} = \beta \exp(-2/\beta) \cdot \exp\left( \min\left\{ -6/\beta, 2\log\left( \frac{\piref_{\min}}{1+\piref_{\min}}\right)\right\}\right) \cdot (\piref_{\min})^2 < 1
  \]
  under the choice $\beta \leq 1$.

  Finally, the third statement follows from Corollary~\ref{cor:pp_spg_stochastic} noticing that our constants depend on $\varepsilon_{\rin}$ only logarithmically and polynomially on $\piref_{\min}, \beta,\eta$.
\end{proof}

\subsection{Verification for Compatible Fisher Non-Degenerate Parametrization}

In this section, we show that along the iterations of Algorithm~\ref{alg:pp_spg}, Assumptions~\ref{ass:reference_policy}--\ref{ass:compatible} hold uniformly over $k$. In particular, Assumptions~\ref{ass:bounded_score}, \ref{ass:fisher_matrix}, and~\ref{ass:compatible} do not depend on the target policy, thus they cannot be disturbed by iterates. 

\begin{lemma}\label{lem:v_pi_k_control}
  Assume that $\norm{\nabla V_{k}(\pi_{k+1}, \pi_{k+1})}_{\spann,\rho}^2  \leq \varepsilon_{\pp}$ for all $k < K$. Then Assumption~\ref{ass:reference_policy}-(i) holds for the next step $K$ with 
  \[
    V_{\tpi_K} \leq 2 \E_{x \sim \rho}\left[ \log^2(1/\piref_{\min}(x)) \right] + \frac{2(1+2\sqrt{\varepsilon_{\pp}})^2}{\beta^2 (1 + \eta)^2}\,.
  \]
\end{lemma}
\begin{proof}
  First, we apply the second inequality of Lemma~\ref{lem:log_min_prob_vs_span} with $\pi = \tpi_k$ and $\mu = \piref$:
  \[
    \E_{x\sim \rho}\left[ \log^2 \frac{1}{\tpi_{k,\min}(x)} \right] \leq 2 \E_{x\sim \rho}\left[ \log^2 \frac{1}{\piref_{\min}(x)} \right] + 8 \norm{\log \tpi_k - \log \piref}_{\spann,\rho}^2\,.
  \]
  For the second term we use the property of geometric mixture policies
  \begin{align*}
    \norm{\log \tpi_k - \log \piref}_{\spann, \rho} &= \norm{1/(1+\eta)\log \pi_k + \eta/(1+\eta) \log \piref - \log \piref}_{\spann, \rho} \\
    &\leq \frac{1}{1+\eta} \norm{\log \pi_k - \log \piref}_{\spann,\rho}\,.
  \end{align*}
  Next, we apply Lemma~\ref{lem:pp_iterate_uniform_regularity} and achieve
  \[
    \norm{\log \tpi_k - \log \piref}_{\spann, \rho} \leq \frac{1}{1+\eta} \cdot \frac{1/2+\sqrt{\varepsilon_{\pp}}}{\beta}  \le  \frac{1+2\sqrt{\varepsilon_{\pp}}}{2\beta(1+\eta)}\,.
  \]
  Taking a square and plugging back into the previous inequality concludes the proof.
\end{proof}
In particular, a simple induction shows that Assumption~\ref{ass:reference_policy}-(i) is satisfied for all $k \in \N$ simultaneously. Thus, we have the following final corollary.

\begin{corollary}[Convergence guarantees for Fisher compatible parameterization]\label{cor:final_rates_fisher}
  Let $\theta \mapsto \pi_\theta$ be a parametrization and $(\cT_k)_{k \geq 0}$ a sequence of improvement operators that satisfies Assumption~\ref{ass:bounded_score}--\ref{ass:compatible} for all $k \in \N$, the reference policy satisfies $V_{\piref} \triangleq \E_{x\sim \rho}[\log^2(1/\piref_{\min}(x))] < \infty$ and let $\beta \leq 1$. 

  Then, fix $\varepsilon_{\rin} > 0$ as a desired accuracy for approximate solving of the inner problem. Then the following statements hold.
  \begin{itemize}[leftmargin=1em]
    \item[(i)] Assumption~\ref{ass:gradient_compatibility} holds with $C_{\gradcomp} = 2 M_g^2/ \pl_F^2$ and $\varepsilon_{\gradcomp} = 2 \varepsilon_{\bias}$;
    \item[(ii)] There exists a truncation threshold $M_k > 0$ such that Assumption~\ref{ass:parametrization_regularity} holds with 
    \begin{align*}
      G &= M_g,\qquad \pl_{\beta,\eta} = \frac{\beta (1 + 1/\eta)\pl_F^2}{2M_g^2},\qquad \epspl = \frac{\varepsilon_{\bias} \cdot \pl_F^2}{M_g^2}\,,\\
      L_{\beta, \eta} &= \left( \tfrac{1}{2} + \tfrac{1}{\eta} \cdot \sqrt{2\beta^2(1+\eta)^2 + 2\beta^2 (1+\eta)^2 V_{\piref} + 2(1 + 2\sqrt{\varepsilon_{\pp}})^2} \right) (M_g^2 + M_h) \\
      &\qquad\qquad\qquad\qquad\qquad\qquad\qquad\qquad+ \beta(1+1/\eta) M_g^2\,, \\
      \varepsilon_{\grad} &= \sqrt{\varepsilon_{\rin}}\,,\qquad  \sigma^2_{\beta,\eta} = D^2_{\beta,\eta}, \qquad v^2_{\beta,\eta} = 6 D^2_{\beta,\eta}\,,\\
      D^2_{\beta,\eta} &\triangleq 2M_g^2\left(1 + \frac{128 M_g^2 [\beta^2(1+\eta)^2 (1 + 4V^2_{\piref}) + 4(1+2\sqrt{\varepsilon_{\pp}})^4]}{\eta^4 \cdot \varepsilon_{\rin}}\right)\,,
    \end{align*}
    where $\varepsilon_{\pp} = \cO(M_g^2 / \pl_F^2 \cdot \varepsilon_{\rin} + \varepsilon_{\bias})$  and a PP learning rate $\eta = \beta \cdot \frac{\pl_F}{\sqrt{2} \cdot M_g^2 - \beta \pl_F}$.
    \item[(iii)] Algorithm~\ref{alg:pp_spg} outputs $\cO(\varepsilon_{\rin} + \varepsilon_{\bias})$-optimal regularized policy after $K = \tcO(1)$ outer iterations, $T_k = \tcO(1/\varepsilon_{\rin}^2)$ inner iterations, using in total $\tcO(1/\varepsilon_{\rin}^4)$ samples, ignoring constants depending on $M_g, M_h, \pl_F, \beta, V_{\piref}$ and logarithmic terms in $\varepsilon_{\rin}, \varepsilon_{\bias}$ and the confidence level $\delta$.
  \end{itemize}
\end{corollary}
\begin{proof}
  The statement (i) automatically follows from Lemma~\ref{lem:fisher_gradient_comparison}. The statement (ii) follows from Lemma~\ref{lem:v_pi_k_control} combined with Proposition~\ref{prop:fisher_verification} applied for each $k \in \N$ as well as a choice of $\varepsilon_{\grad} = \sqrt{\varepsilon_{\rin}}$ and manipulations to simplify the expression for $D^2_{\beta,\eta}$:
  \begin{align*}
    D^2_{\tpi_k, \lambda}(\sqrt{\varepsilon_{\rin}}) &= M_g^2 \left( 1 + \frac{8 \lambda^2 M_g (V_{\tpi_k} + 1)}{\sqrt{\varepsilon_{\rin}}} \right)^2 \leq 
    2M_g^2\left( 1 + \frac{(8 \lambda^2 M_g (V_{\tpi_k} + 1))^2}{\varepsilon_{\rin}} \right) \\
    &\leq 2M_g^2\left(1 + 64 \beta^4 (1+\eta)^4\frac{M_g^2 (1 + 2 V_{\piref} + 2(1+2\sqrt{\varepsilon_{\pp}})^2/(\beta^2 (1+\eta)^2) )^2}{\varepsilon_{\rin} \cdot \eta^4}\right) \\
    &\leq 2M_g^2\left(1 + \frac{128 M_g^2 (\beta^2(1+\eta)^2 (1 + 4V^2_{\piref}) + 4(1+2\sqrt{\varepsilon_{\pp}})^4)}{\eta^4 \cdot \varepsilon_{\rin}}\right) \triangleq D^2_{\beta,\eta}\,,
  \end{align*}
  where we used Lemma~\ref{lem:v_pi_k_control} and an inequality $(a+b)^2 \leq 2a^2 + 2b^2$. Finally, we choose $\eta = \beta \cdot \frac{\pl_F}{\sqrt{2} \cdot M_g^2 - \beta \pl_F}$ to guarantee that $\beta(1+1/\eta) \cdot \pl_{\beta,\eta} = 1$.

  For the statement (iii), we apply Corollary~\ref{cor:pp_spg_stochastic} with the constants from (ii). In particular, we choose $\varepsilon_{\pp} = \cO(M_g^2 / \pl_F^2 (\varepsilon_{\rin} + \varepsilon_{\bias}))$ to guarantee that the final policy is $\cO(\varepsilon_{\rin} + \varepsilon_{\bias})$-VNW. The number of outer iterations $K$ follows from Corollary~\ref{cor:pp_spg_stochastic} and scales as $\tcO(1)$. The number of inner iterations $T_k$ follows from Corollary~\ref{cor:pp_spg_stochastic} and scales as $\tcO(1/\varepsilon_{\rin}^2)$, where additional $1/\varepsilon_{\rin}$ factor comes from a bound on $v^2_{\beta,\eta} \asymp 1/\varepsilon_{\rin}$. Finally, the total number of samples follows from  $\sum_{k=1}^K \sum_{t=1}^{T_k} B_{k,t} = \tcO(1/\varepsilon_{\rin}^4)$.
\end{proof}

\section{Detailed Experiment Description}\label{app:experiments}

\subsection{Rock-Paper-Scissors}\label{app:rps_experiment}

\begin{figure}
    \centering

    \includegraphics[width=0.95\linewidth]{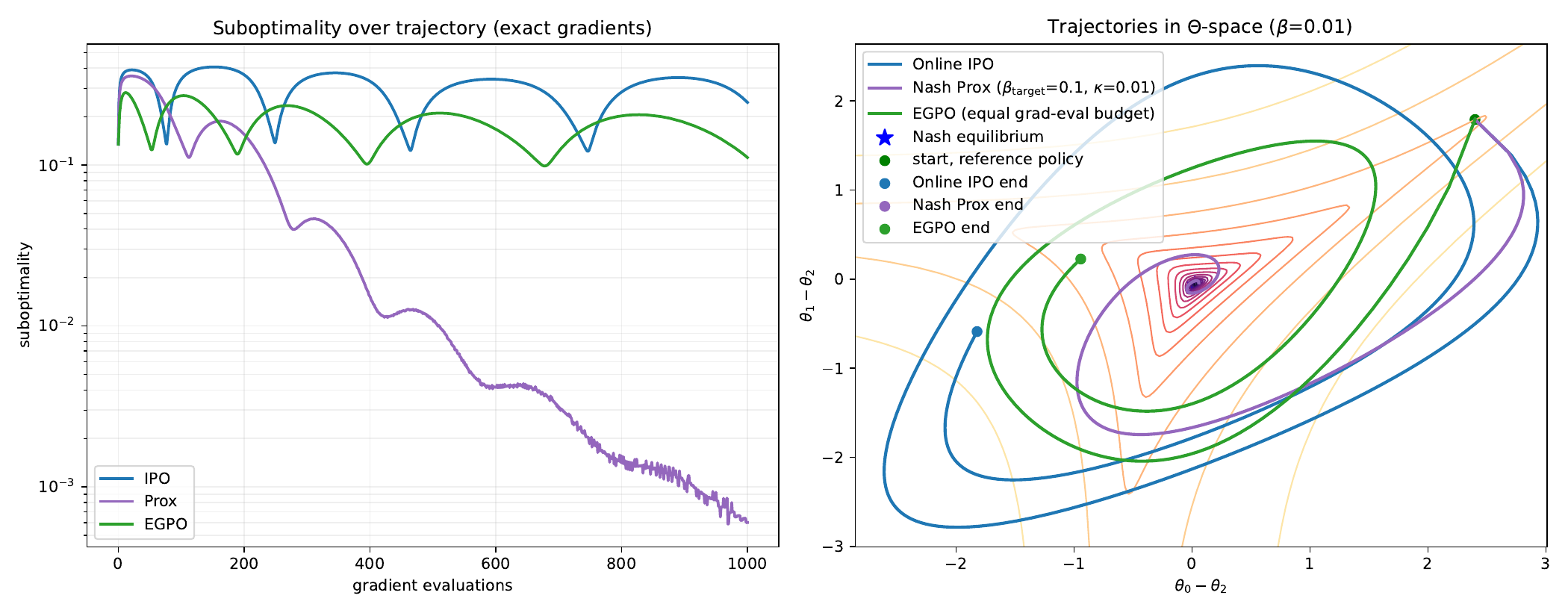}
    \includegraphics[width=0.95\linewidth]{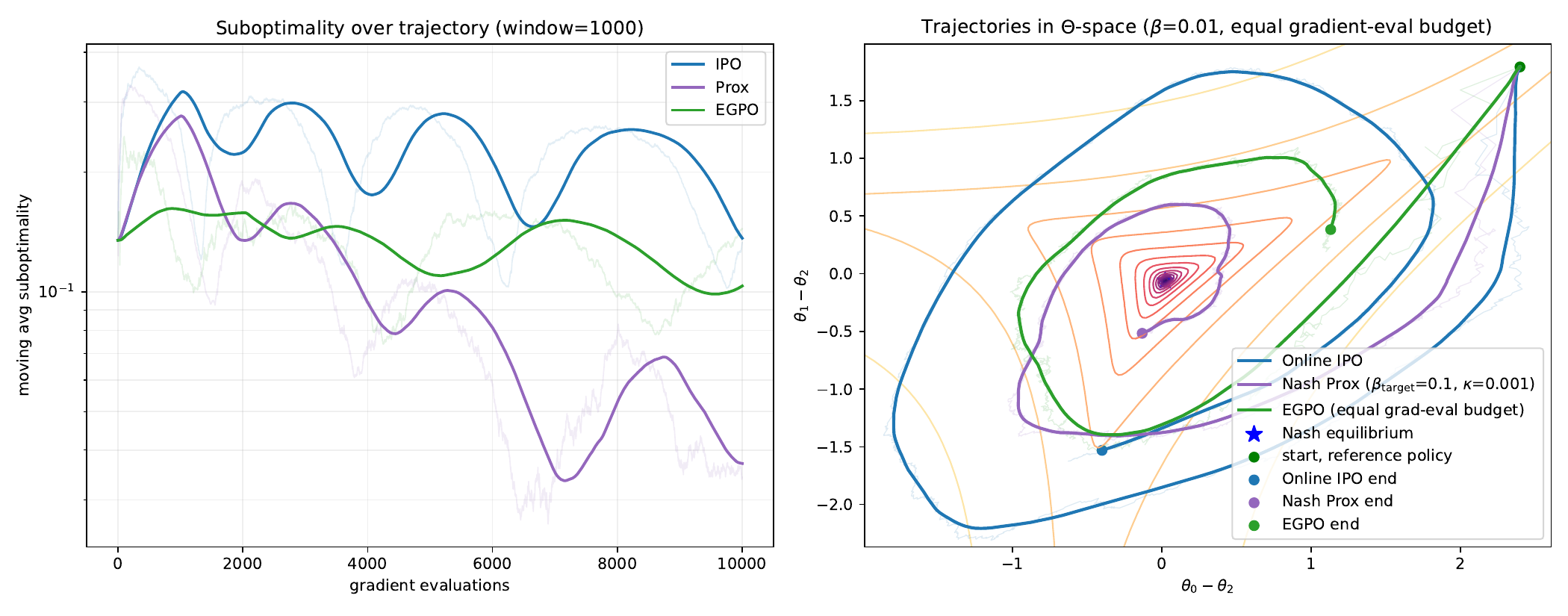}
    \caption{\textit{Optimization trajectories for the Rock--Paper--Scissors example.} \textbf{Top row}: exact-gradient updates. \textbf{Bottom row}: stochastic-gradient (SGD) updates. \textbf{Left}: suboptimality versus gradient evaluations; for SGD (bottom) we plot an exponential moving average (EMA) to smooth single-trajectory noise. \textbf{Right}: trajectories in a 2D projection of the 3D parameter space ($\Theta$-space).}
    \label{fig:rps}
\end{figure}

In this section, we provide an additional experiment on a very simple 3-dimensional game of Rock-Paper-Scissors. It is a context-free game defined by the following matrix
\[
    \bP = \begin{bmatrix}
        0.5 & 1 & 0 \\
        0 & 0.5 & 1 \\
        1 & 0 & 0.5
    \end{bmatrix}\,,
\]
a reference policy $\piref = (11/18, 1/3, 1/18)$, and $\beta =0.01$. With this environment, we implemented in JAX~\citep{jax2018github} an exact and stochastic versions of Online IPO~\citep{calandriello2024human}, EGPO~\citep{zhou2025extragradient}, and $\algo$, using a learning rate $\alpha_t = t^{-1/2}$ for an exact version and $\alpha_t = 0.2 \times t^{-1/2}$ for a stochastic one, where $t$ is an iteration number, additionally multiplied the IPO-style losses by an effective regularization to guarantee the same scaling of updates. The value of $\kappa$ is normalized to make $10$ total soft updates of the target policy at the training. The results are presented on Figure~\ref{fig:rps}.

Overall, we observe that the trajectories of $\algo$ and EGPO are both present more stabilized behavior compared to Online IPO, but they are stabilized differently: two-step stabilization of EGPO turns out to be an effective measure, especially in the beginning of trajectory, but later the soft-anchoring of $\algo$ starts behaving better empirically. Additionally, we observe two facts about the geometry of the problem: (1) suboptimality landscape exhibits non-convex behavior, although it admits a gradient dominance geometry, and (2) all the methods exhibits non-monotonic behavior in suboptimality. In fact, (2) exactly prevents us to provide a \emph{last-iterate} convergence theory for Online IPO for small values of $\beta$.

\subsection{Matrix Games}

\paragraph{Experiment setup.} In our experiments, we fixed $r=2, Y=100, \beta=0.01$ and a reference policy to be a uniform distribution $\piref(y|x) = 1/Y$. The matrices $U$ and $V$ are generated as random Gaussian matrices. To parameterize the space of policies, we use a 3-layer MLP with a ReLU action function and 128 hidden units, taking a flattened matrix $\Theta_x \in \R^{2 \times 2}$ as input and outputting logits over possible actions. We use the Adam optimizer \citep{kingma2015adam}, and for all the baselines we perform a grid search over the learning rate using the grid $\{ 3\times 10^{-3}, 10^{-3}, 3 \times 10^{-4}, 10^{-4}, 3 \times 10^{-5} \}$. For each batch, we sample 128 random games. For Nash Prox, we utilize $\beta_{\target} = 10 \times \beta$ for all experiments. All experiments implemented using JAX~\citep{jax2018github}, runtime for one configuration running in parallel for all 25 seeds is less than 5 minutes for all the methods.

\paragraph{Soft vs. hard updates.} As a first additional experiment, we compare the influence of soft and hard updates for the target network. We fixed the learning rate to $3 \times 10^{-4}$ and varied the update period before the update in hard and soft senses; see Figures~\ref{fig:sched_a},\ref{fig:sched_b},\ref{fig:sched_c} for the results. Overall, these results show the benefits of soft updates over hard ones in terms of suboptimality.

\paragraph{Adaptive choice of $\kappa$.} Next, we verify the adaptive schedule $\kappa_t = 1/(0.3\cdot t + 1)$, which changes over the optimization procedure; the results are presented in Figure~\ref{fig:sched_d}. In particular, we observe that the adaptive schedule performs on par with the best choice of standard soft updates; however, it is less sensitive to the choice of optimal value of the multiplier than the choice of optimal value of $\kappa$ for usual soft updates.

\begin{figure}[t]
    \centering

    \begin{subfigure}{0.49\linewidth}
        \centering
        \includegraphics[width=\linewidth]{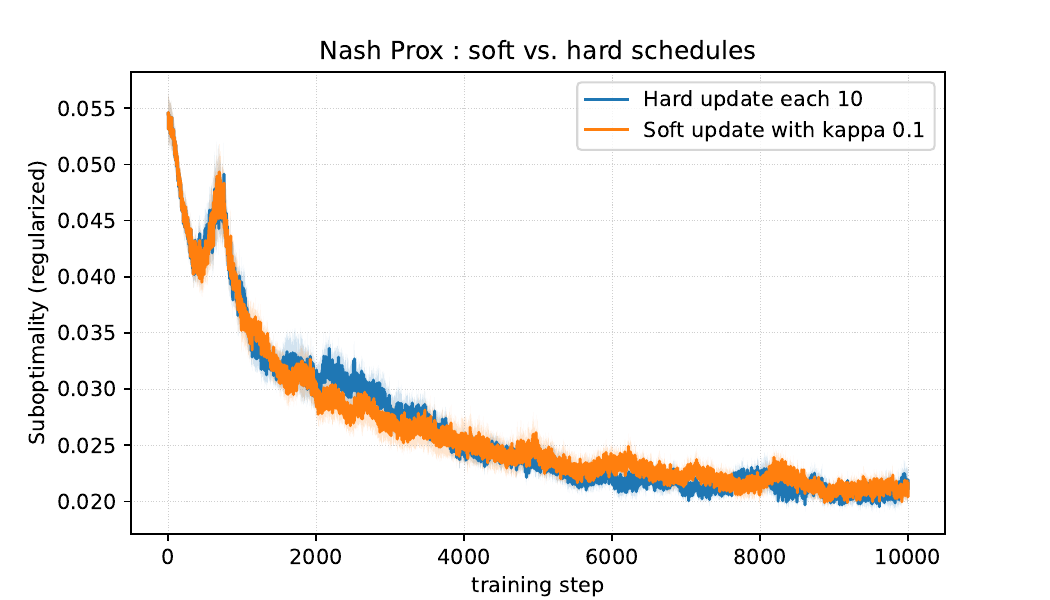}
        \caption{Hard updates vs. soft updates, period 10;}
        \label{fig:sched_a}
    \end{subfigure}
    \hfill
    \begin{subfigure}{0.49\linewidth}
        \centering
        \includegraphics[width=\linewidth]{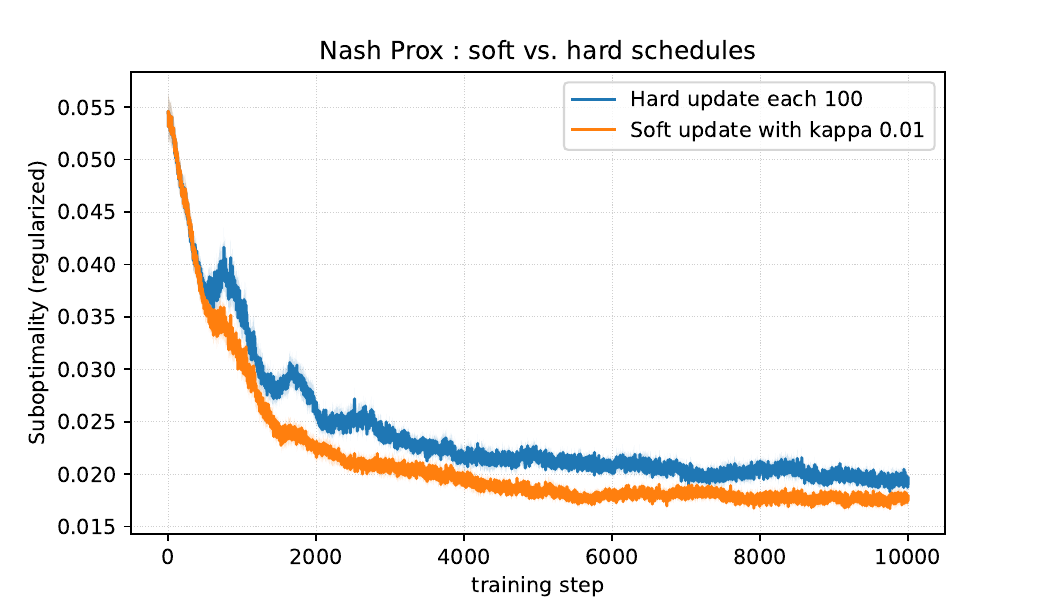}
        \caption{Hard updates vs. soft updates, period 100;}
        \label{fig:sched_b}
    \end{subfigure}

    \vspace{0.6em}

    \begin{subfigure}{0.49\linewidth}
        \centering
        \includegraphics[width=\linewidth]{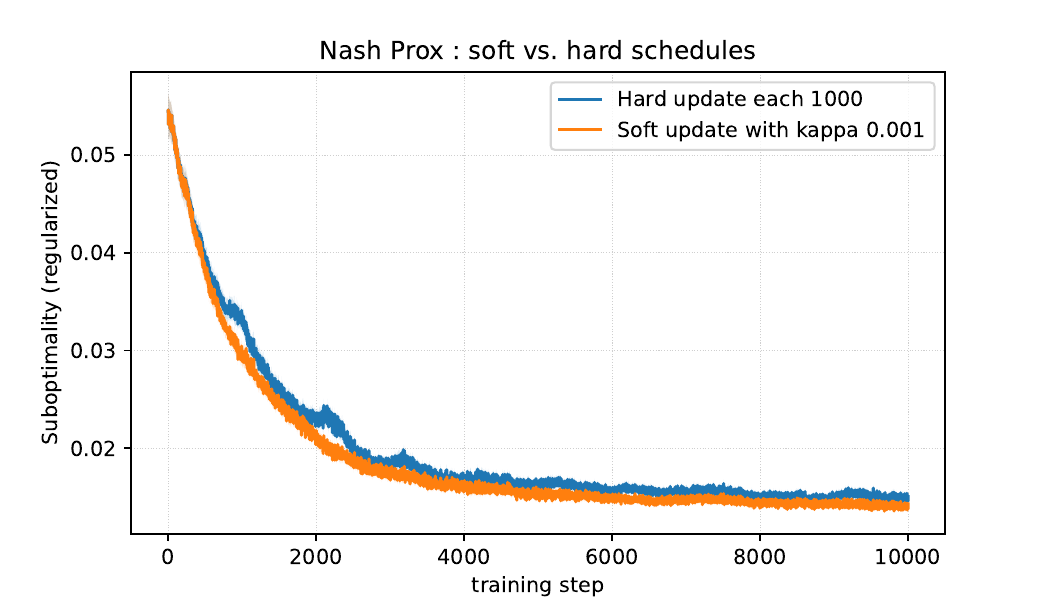}
        \caption{Hard updates vs. soft updates, period 1000;}
        \label{fig:sched_c}
    \end{subfigure}
    \hfill
    \begin{subfigure}{0.49\linewidth}
        \centering
        \includegraphics[width=\linewidth]{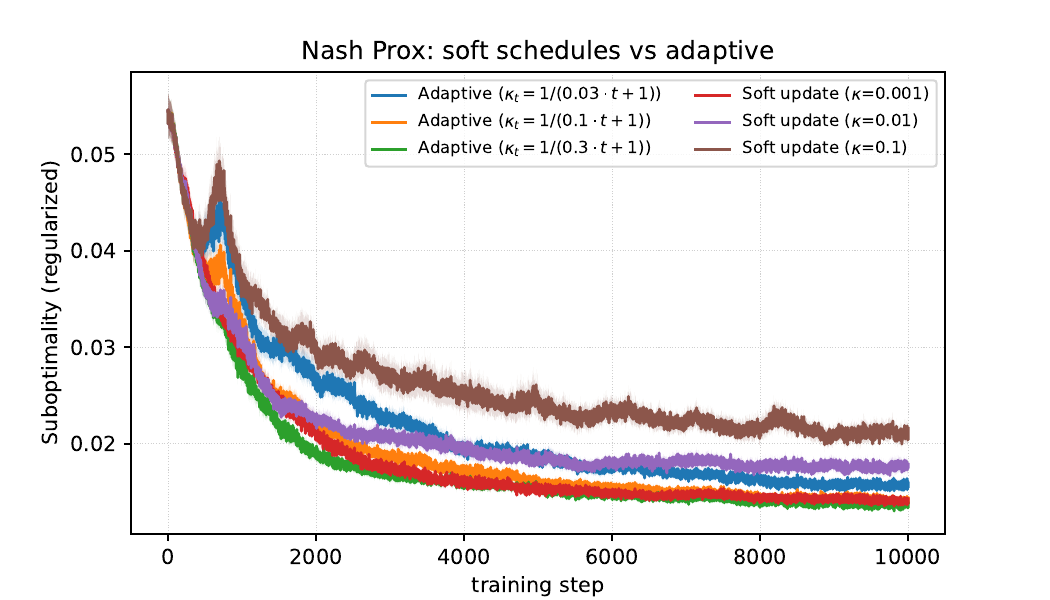}
        \caption{Soft updates vs. adaptive schedule.}
        \label{fig:sched_d}
    \end{subfigure}

    \caption{Comparison between different target policy update schedules for $\algo$. 
    Suboptimality is averaged over 25 random seeds; shaded regions indicate one standard deviation.}
    \label{fig:schedules}
\end{figure}

\subsection{LLM Alignment}

In this section, we provide implementation and details as well as details on hyperparameter selection.

\subsubsection{Loss implementation}
We use a library TRL \citep{vonwerra2022trl} as a base for our implementation, and we use vLLM~\citep{kwon2023efficient} for efficient generation. Following the discussion in Section~\ref{sec:implementation}, we use the following loss function to achieve the correct gradients using automatic differentiation in PyTorch \citep{pytorch}
\begin{align*}
    \widetilde{\cL}_{\algo}(\theta;  \theta^{\mathrm{target}}) &\triangleq  \frac{1}{B} \sum_{i=1}^B  \left(  \ell_{\theta}(x_i,y_i) -\ell_{\theta}(x_i,y'_i) - \frac{\cP(y_i \succ y'_i | x_i)-1/2}{\beta + \beta_{\target}} \right)^2
    \\
    \ell_\theta(x,y) &\triangleq \frac{\beta}{\beta + \beta_{\target}} \log \left(\frac{\pi_\theta(y|x)}{\piref(y|x)}\right) + \frac{\beta_{\target}}{\beta + \beta_{\target}} \log\left( \frac{\pi_\theta(y|x)}{\pi_{\theta^{\target}}(y|x)} \right)\,,
\end{align*}
where $\{x_i\}_{i \in [B]}$ are samples from the prompt dataset, $\{ (y_i, y_i') \}_{i\in[B]}\}$ are generated from a policy $\pi_{\theta_t}$ using a temperature sampling with a temperature 1, and $\mathtt{SG}$ is a stop-gradient operations.
$\log \pi_\theta(y|x)$ is computed in an auto-regressive manner as $\log \pi_\theta(y|x) = \sum_{i=1}^{|y|} \log \pi_\theta(y_i | x, y_{<i}) = \sum_{i=1}^{|y|} \mathtt{logit}_\theta(y_i | x, y_{<i}) - \mathtt{LogSumExp}(\mathtt{logit}_\theta(\cdot | x, y_{<i}))$, where $ \mathtt{logit}_\theta$ are raw logits outputted by the model with parameters $\theta$. 

\subsubsection{Experiment description}

We start our experiments from the Google Gemma-3-4B\footnote{Published under Gemma license.} \citep{team2025gemma} pretrained checkpoint.

\paragraph{Supervised fine-tuning (SFT).}
For SFT, we use the \texttt{RLHFlow/RLHFlow-SFT-Dataset-ver2} dataset \citep{rlhflow_sft}. This dataset, structured as conversations, is processed using a chat template following the Gemma3 format (\texttt{<bos><start\_of\_turn>role\textbackslash ncontent<end\_of\_turn>\textbackslash n...}), where the template maps the \texttt{assistant} role to \texttt{model}. System messages are dropped from the input, and training is performed on the train split. The dataset samples are tokenized, with a maximum sequence length of $8,\!192$ tokens. We use sample packing to efficiently train on long sequences and pad sequences to the maximum length. Following standard SFT practice, the loss is computed only on the model's output tokens (the assistant's turns), not on the input prompts.

The model was fully fine-tuned (no LoRA PEFT adapter was used). Training was conducted for $2$ epochs. Optimization was performed using a fused version of AdamW \citep{loshchilov2018decoupled} 32-bit optimizer with a learning rate of $1.5 \times 10^{-5}$. A cosine learning rate schedule was applied with a warmup ratio of $0.05$  of the total training steps. We used a micro batch size of $1$ sequence per device and accumulated gradients over $16$ steps, resulting in an effective batch size of $16$ sequences per device. On the 8 A100 GPUs, we thus had an effective batch size of $128$. Gradient clipping was applied with a maximum norm of $1.0$, and no weight decay was used. 

For improved memory efficiency and speed, we enabled gradient checkpointing and leveraged Flash Attention \citep{dao2023flashattention2}. Training utilized BFloat16 (BF16) and TF32 precision where supported.

\paragraph{Nash Learning from Human Feedback}
All subsequent NLHF experiments started from the SFT checkpoint described above. This SFT model also served as the initial policy and the reference policy ($\piref$). During this phase, we used LoRA adapters~\cite{hu2022lora} with rank $r=16$ and $\alpha=32$ for all methods.

\textit{Datasets.} For generating responses during NLHF training and for final evaluation, we used a subset of prompts from the \texttt{RLHFlow/prompt-collection-v0.1} dataset \citep{rlhflow_prompt}. We first filtered this dataset to have only prompt conversations of length less than $512$ tokens, and used 5\% of prompts as a separate validation set. Overall, we have $\approx 80,\!200$ train prompts and $\approx 4,\!220$ validation problems. For further details on the original data mixtures within \texttt{RLHFlow/prompt-collection-v0.1} and their licenses, we refer to \citep{dong2024rlhf}.

\textit{Preference model.}
The pairwise preference model, used to provide comparison signals, was a Gemma2-2B model~\citep{riviere2024gemma}. This model was trained on the \texttt{RLHFlow/pair\_preference\_model\_dataset} dataset \citep{rlhflow_preference}, with its training methodology detailed in \citep{liu2025rrm}.
We employed a separate, more capable judge for the final evaluation of model performance: a Gemma3-27-IT model prompted to decide which completion better follows the instructions. On this stage, we perform two episodes of judgment per prompt in two different orders: (\texttt{prompt}, \texttt{completion A}, \texttt{completion B}) and (\texttt{prompt}, \texttt{completion B}, \texttt{completion A}). If results were inconsistent (e.g., in both cases the judge selected different completions depending on their order), we consider such judgments as inconsistent and do not include them in the win-rate computation. On average, around $40\%$ of the comparisons were inconsistent.

The primary evaluation metric was side-by-side pairwise win rate, and for hyperparameter selection in each algorithm, we used win rate against the SFT reference policy as a proxy metric. Confidence intervals were computed as $99.9\%$-confidence intervals using the first term of the empirical Bernstein inequality for Bernoulli random variables: for an estimate $\hat{p}$, we compute the intervals as $\sqrt{2 \hat{p}(1-\hat{p}) \cdot \log(2/10^{-3}) \cdot 1/N}$, where $N$ is a number of \emph{consistent} judgments.

\begin{table}[tbp]
\centering
\caption{Hyperparameter settings for the evaluated algorithms. The hyperparameters with the best performance in comparisons against the SFT checkpoint are presented.}
\label{tab:hyperparameters}
\begin{tabular}{@{}l S[table-format=1e-1] S[table-format=1e-1] l@{}} 
\toprule
Algorithm        & {Learning Rate} & {$\beta$} & Algorithm-Specific Parameters \\
\midrule
Online DPO       & 3e-5  & 1e-2  & N/A \\
Online IPO       & 1e-4  & 1e-2 & N/A \\
$\algo$          & 3e-5  & 1e-3 &  $\kappa_t=1/(0.1\times t + 1)$, $\beta_{\target}=10\times \beta$ \\
\midrule[\heavyrulewidth] 
\bottomrule
\end{tabular}
\end{table}

\textit{Training Configuration and Hyperparameter Tuning.}
For all NLHF experiments, we used the AdamW optimizer \citep{loshchilov2018decoupled}. The learning rate schedule featured a $0.1$ warmup period (as a fraction of total training steps) followed by a linear decay.
The effective global batch size was $32$ prompts, following the practice of small batch sizes as indicated by \cite{schulman2025lora}. We used per-device micro-batch size of $8$ prompts, 2 GPUs A100, and $2$ gradient accumulation steps ($8 \text{ prompts/GPU} \times 2 \text{ GPUs} \times 2 \text{ grad\_accum\_steps}$). Training was conducted for $1$ epoch over the $N_{\mathrm{train}}=80,200$ prompts, corresponding to approximately $2505$ update steps. We employed gradient clipping with a 
 maximum norm of $1.0$.

For each algorithm, we perform a grid search over its key hyperparameters: learning rate over $\mathrm{lr} \in \{ 3 \times 10^{-4}, 10^{4}, 3 \times 10^{-5} \}$, regularization parameter $\beta \in \{ 10^{-3}, 10^{-2}, 10^{-1} \}$, soft update schedule of form $\kappa_t = 1/(ct + 1)$ for $c \in \{0.3, 0.1, 0.03\}$. For $\algo$, we use always $\beta_{\target}=10 \times \beta$.

For each algorithm, we selected the best-performing hyperparameter configurations based on the win rate against the SFT reference policy, evaluated using the Gemma3-27B-IT judge. These selected checkpoints were then compared side-by-side in the final evaluation using the same judge model. The final reported results (see Table~\ref{tab:hyperparameters}) represent the performance of the best configuration found through this process.

To manage memory and improve throughput during all NLHF training phases, we utilized BFloat16 (BF16) mixed-precision, gradient checkpointing, and PyTorch DDP~\citep{li2020pytorch}.

 \textit{Computational Resources and Runtimes.} All NLHF experiments were conducted on 2 NVIDIA A100 (80GB) GPUs, with one additional GPU used to host a pairwise reward model. A full training cycle for Nash Prox and baselines requires approximately 11 hours per algorithm.

\textit{Generation Parameters.}
During response generation, both for collecting experiences within NLHF algorithms (e.g., generating samples per prompt) and for final evaluation on the test set, we used temperature sampling with a temperature of $\tau=1.0$. The maximum generation length was capped at 256 tokens.

\paragraph{Limitations.} We notice that using the target network increases the memory footprint of the model, which is a limitation of our method. However, this increase is marginal since we do not need to backpropagate through the target model, and in the LoRA setting this memory footprint is negligible. Additionally, it does not significantly increase the running time of the algorithm and is straightforward to implement.

\end{document}